\documentclass[twoside,11pt]{article}

\usepackage{amsmath,amsthm,amsopn,amsfonts,pdfpages,url,bbm,amssymb,pdfpages}
\usepackage{fullpage}
\usepackage{graphics,relsize,dsfont}
\usepackage{subfig}
\usepackage{natbib} \bibpunct{(}{)}{,}{a}{}{,}
\usepackage[pdftex,pdfpagelabels,bookmarks,hyperindex,hyperfigures]{hyperref}

\makeatletter
\let\@fnsymbol\@arabic
\makeatother

\newtheorem{remark}{Remark}
\newtheorem{definition}{Definition}
\newtheorem{lemma}{Lemma}
\newtheorem{theorem}{Theorem}
\newtheorem{corollary}{Corollary}

\newcommand{\bbP}{\mathbb{P}}
\newcommand{\bbR}{\mathbb{R}}
\newcommand{\bbRnonneg}{\mathbb{R}_{\ge 0}}

\newcommand{\E}{\mathbb{E}}

\newcommand{\bbar}{\tilde{b}}
\newcommand{\Kbar}{\bar{K}}

\newcommand{\calD}{\mathcal{D}}

\newcommand{\calL}{\mathcal{L}}
\newcommand{\calN}{\mathcal{N}}
\newcommand{\calP}{\mathcal{P}}

\newcommand{\calT}{\mathcal{T}}
\newcommand{\calThat}{\widehat{\calT}}

\newcommand{\calX}{\mathcal{X}}

\newcommand{\ellhat}{\hat{\ell}}

\newcommand{\phat}{\hat{p}}
\newcommand{\rhat}{\hat{r}}
\newcommand{\Shat}{\hat{S}}
\newcommand{\Uhat}{\hat{U}}
\newcommand{\what}{\hat{w}}
\newcommand{\wtrue}{\bar{w}}
\newcommand{\Xhat}{\hat{X}}

\newcommand{\Scheck}{\check{S}}

\newcommand{\Ucheck}{\check{U}}
\newcommand{\wcheck}{\check{w}}
\newcommand{\Xcheck}{\check{X}}

\newcommand{\Atilde}{\tilde{A}}

\newcommand{\calDtilde}{\tilde{\calD}}
\newcommand{\Gtilde}{\tilde{G}}
\newcommand{\ntilde}{\tilde{n}}
\newcommand{\Otilde}{\tilde{O}}
\newcommand{\rtilde}{\tilde{r}}
\newcommand{\Qtilde}{\tilde{Q}}
\newcommand{\Stilde}{\tilde{S}}
\newcommand{\Utilde}{\tilde{U}}
\newcommand{\Vtilde}{\tilde{V}}

\newcommand{\wtilde}{\tilde{w}}
\newcommand{\Xtilde}{\tilde{X}}
\newcommand{\xtilde}{\tilde{x}}
\newcommand{\Deltatilde}{\tilde{\Delta}}
\newcommand{\Lambdatilde}{\tilde{\Lambda}}
\newcommand{\Sigmatilde}{\tilde{\Sigma}}

\newcommand{\avec}{\vec{a}}
\newcommand{\bvec}{\vec{b}}
\newcommand{\dvec}{\vec{d}}
\newcommand{\gvec}{\vec{g}}
\newcommand{\hvec}{\vec{h}}

\newcommand{\rvec}{\vec{r}}
\newcommand{\zvec}{\vec{z}}

\newcommand{\LS}{\operatorname{LS}}
\newcommand{\ML}{\operatorname{ML}}
\newcommand{\RDPG}{\operatorname{RDPG}}
\newcommand{\ASE}{\operatorname{ASE}}
\newcommand{\LSE}{\operatorname{LSE}}

\newcommand{\Bernoulli}{\operatorname{Bernoulli}}
\newcommand{\supp}{\operatorname{supp}}
\newcommand{\diag}{\operatorname{diag}}
\newcommand{\iid}{\stackrel{\text{i.i.d.}}{\sim}}

\newcommand{\eventually}{\text{ eventually}}
\newcommand{\tti}{2,\infty}

\newcommand{\onevec}{e}

\newcommand{\UP}{U}
\newcommand{\UA}{\Uhat}
\newcommand{\UL}{\Utilde}
\newcommand{\SP}{S}
\newcommand{\SA}{\Shat}

\newcommand{\rls}{r_{\LS}}
\newcommand{\rhatls}{\rhat_{\LS}}
\newcommand{\rtildels}{\rtilde_{\LS}}

\newcommand{\wls}{w_{\LS}}
\newcommand{\whatls}{\what_{\LS}}
\newcommand{\wcheckls}{\wcheck_{\LS}}
\newcommand{\wtildels}{\wtilde_{\LS}}

\newcommand{\whatml}{\what_{\ML}}
\newcommand{\thetals}{\theta_{\LS}}
\newcommand{\nuls}{\nu_{\LS}}

\newcommand{\lambdamin}{\lambda_{\min}}
\newcommand{\projcompXtilde}{\calP_{\Xtilde}^{\perp}}

\newcommand{\inlaw}{\xrightarrow{\calL}}

\newcommand{\inprob}{\xrightarrow{P}}

\title{Limit theorems for out-of-sample extensions of the adjacency and Laplacian spectral embeddings}

\author{
Keith Levin\footnote{Department of Statistics, University of Michigan},
Fred Roosta\footnote{School of Mathematics and Physics, University of Queensland}~\textsuperscript{,}\footnote{International Computer Science Institute},
Minh Tang\footnote{Department of Statistics, North Carolina State University},
Michael W. Mahoney\footnotemark[3]~\textsuperscript{,}\footnote{Department of Statistics, University of California at Berkeley}~
and Carey E. Priebe\footnote{Department of Applied Mathematics and Statistics, Johns Hopkins University}
}

\begin{document}

\maketitle

\abstract{
Graph embeddings, a class of dimensionality reduction techniques designed for relational data, have proven useful in exploring and modeling network structure. Most dimensionality reduction methods allow out-of-sample extensions, by which an embedding can be applied to observations not present in the training set. Applied to graphs, the out-of-sample extension problem concerns how to compute the embedding of a vertex that is added to the graph after an embedding has already been computed. In this paper, we consider the out-of-sample extension problem for two graph embedding procedures: the adjacency spectral embedding and the Laplacian spectral embedding. In both cases, we prove that when the underlying graph is generated according to a latent space model called the random dot product graph, which includes the popular stochastic block model as a special case, an out-of-sample extension based on a least-squares objective obeys a central limit theorem about the true latent position of the out-of-sample vertex. In addition, we prove a concentration inequality for the out-of-sample extension of the adjacency spectral embedding based on a maximum-likelihood objective. Our results also yield a convenient framework in which to analyze trade-offs between estimation accuracy and computational expense, which we explore briefly.
}

\section{Introduction}
\label{sec:intro}

%
%
%
%

Graph embeddings are a class of dimensionality reduction techniques
designed for network data, which have emerged as a popular tool for
exploring and modeling network structure.
Given a graph $G = (V,E)$ on vertex set $V = \{1,2,\dots,n\}$
with adjacency matrix $A \in \{0,1\}^{n \times n}$,
the graph embedding problem concerns how best to map $V$ to
a $d$-dimensional vector space so that geometry in
that vector space captures the topology of $G$.
For example, we may ask that vertices that play similar structural roles in
$G$ be mapped to nearby points.
Two common approaches to graph embedding are
the graph Laplacian embedding \citep{BelNiy2003,CoiLaf2006} and
the adjacency spectral embedding \citep[ASE;][]{SusTanFisPri2012},
both which are based on spectral decompositions of the adjacency matrix
or a transformation thereof.
In many settings, data collection or computational constraints may
dictate that having computed an embedding of the graph $G$,
a practitioner may wish to add vertices to $G$, and compute the corresponding
embeddings of these new vertices.
We call these new vertices {\em out-of-sample} vertices,
in contrast to the {\em in-sample} vertices in $V$.
Since constructing the in-sample embedding typically requires a comparatively
expensive eigenvalue computation, it is preferable to compute this
out-of-sample embedding without computing a new graph embedding from scratch.
This problem is well-studied in the dimensionality reduction literature,
where it is known as the out-of-sample extension problem.
The focus of the present paper is to derive out-of-sample extensions
for the ASE and a slight variant of Laplacian eigenmaps,
and to establish their statistical properties under a particular
natural choice of network model.

Latent space network models are a class of statistical models for graphs
in which unobserved geometry drives network formation.
Each vertex is assigned a {\em latent position}, and pairs of vertices
form edges according to how near their latent positions are to one another.
Under certain latent space models, graph embeddings may be thought of as
estimating these latent positions.
The focus of the present work is the random dot product graph,
a latent position model that subsumes the popular stochastic block model
(see Section~\ref{subsec:background} below).
Under this model, both the ASE
and a slight variant of Laplacian eigenmaps called
the Laplacian spectral embedding \citep[LSE;][]{TanPri2018},
recover all the latent positions of the in-sample vertices
uniformly \citep{LyzSusTanAthPri2014,TanPri2018}.
Specifically, one obtains a bound on the estimation error of order
$n^{-1/2}$ (ignoring logarithmic factors)
that holds uniformly over all $n$ vertices in the graph.
Further, any constant number of vertices jointly obey a CLT,
in that their embeddings are jointly asymptotically normally distributed about
the true latent positions
\citep{AthLyzMarPriSusTan2016,LevAthTanLyzPri2017,TanPri2018}.
In this paper, we show that analogous results hold for the out-of-sample
extensions of both the ASE and LSE.
That is, the out-of-sample extensions of these two
methods recover the latent positions of the out-of-sample vertices
at the same rate as would be obtained
by the computationally more expensive in-sample embedding.

\subsection{Background and Notation} \label{subsec:background}

Most dimensionality reduction and embedding techniques begin with
a collection of training data observations
$\calD = \{z_1,z_2,\dots,z_n\} \subseteq \calX$,
where $\calX$ is the set of all possible observations
(e.g., the set of all possible images, audio signals, etc.).
$\calX$ is endowed with a similarity measure
$K: \calX \times \calX \rightarrow \bbRnonneg$,
and most embedding procedures leverage the eigenstructure of the symmetric
similarity matrix $M = [ K(z_i,z_j) ] \in \bbR^{n \times n}$.
An embedding of the data $\calD$ assigns to each $z_i \in \calD$ a
vector $x_i \in \bbR^d$, where $d$ is the embedding dimension,
with the embeddings $\{x_1,x_2,\dots,x_n\}$ chosen so as to
preserve the structure of the sample $\calD$ as captured by the matrix $M$.
This typically manifests as attempting to ensure that elements
$z_i,z_j \in \calD$ for which $K(z_i,z_j)$ is large
are mapped so that $\| x_i - x_j \|$ is small.
Suppose that, having computed $x_1,x_2,\dots,x_n$, we obtain a new
{\em out-of-sample} observation $z \in \calX$
(which may or may not appear in the training sample $\calD$),
which we would like to embed along with the {\em in-sample}
observations $\calD$.
Letting $\calDtilde = \calD \cup \{z\}$, a na\"ive approach
would simply construct a new embedding
$\{\xtilde_1,\xtilde_2,\dots,\xtilde_n,\xtilde_{n+1} \}$
based on the sample $\calDtilde$.
This would involve computational complexity of the same order as that
required to compute the initial embedding
$\{x_1,x_2,\dots,x_n\}$.
Since computing the embedding $\{x_1,x_2,\dots,x_n\}$ tends to involve
expensive computations, most commonly eigendecompositions,
it would be preferable to avoid paying this computational cost repeatedly,
particularly if there exists a scheme whereby
the embedding $\xtilde_{n+1}$ of out-of-sample observation $z$ can be
well approximated by a less costly computation.
This is the motivation for the out-of-sample (OOS) extension problem,
which concerns how to embed $z$ into the
same embedding space $\bbR^d$ based only on the existing {\em in-sample
embedding} $\{x_1,x_2,\dots,x_n\}$ and the similarity measurements
$\{ K(z,x_i) : i=1,2,\dots,n \}$.
That is, we wish to compute an embedding of $z$ {\em without} making recourse
to the full similarity matrix $M \in \bbR^{n \times n}$.

As an illustrative example, consider the Laplacian eigenmaps embedding
\citep{BelNiy2003,BelNiySin2006}.
Recall that the {\em normalized Laplacian} of graph $G = (V,E)$
with adjacency matrix $A \in \bbR^{n \times n}$ is given by
the matrix $L = D^{-1/2} A D^{-1/2}$,
where $D \in \bbR^{n \times n}$ is the diagonal matrix of degrees,
with $D_{ii} = \sum_{j=1}^n A_{ij}$, and $0^{-1/2} = 0$ by convention
\citep{Chung1997,von2007tutorial,vishnoi2013lx}.
The $d$-dimensional normalized Laplacian eigenmaps embedding of $G$ is
then given by the rows of the matrix $\UL \in \bbR^{n \times d}$, where
the columns of $\UL$ are the orthonormal eigenvectors corresponding to the
top $d$ eigenvectors of $L$, excluding the trivial eigenvalue $1$.
Suppose now that we wish to add a vertex $v$ to the graph, to form
graph $\Gtilde$ with adjacency matrix
\begin{align} 
\label{eq:def:Atilde}
\Atilde = \begin{bmatrix} A & \avec \\
\avec^T & 0 \end{bmatrix},
\end{align}
where $\avec \in \{0,1\}^n$ and has $a_i = 1$ if and only if $v$ forms
and edge with in-sample vertex $i \in [n]$.
Na\"ively, one could simply apply the Laplacian eigenmaps embedding again to
$\Atilde$, at the cost of another eigendecomposition.
Cheaper, however, would be an OOS extension,
such as that given by \cite{BengioETAL2003} or \cite{BelNiySin2006},
that only makes use of the embedding $\UL$ and the vector of edges $\avec$.

Out-of-sample extensions for
multidimensional scaling \citep[MDS;][]{Torgerson1952,BorGro2005},
spectral clustering \citep{Weiss1999,NgJorWei2002},
Laplacian eigenmaps \citep{BelNiy2003}
and ISOMAP \citep{TendeSLan2000} appear in \citet{BengioETAL2003}.
These extensions were obtained by formulating each of the dimensionality
reduction techniques as a least-squares problem, which is possible owing to the
fact that the in-sample embeddings are functions of the eigenvalues and
eigenvectors of a similarity or distance matrix.
Let matrix $M = [ K(x_i,x_j) ]_{i,j=1}^n$ be the similarity matrix
for some similarity function $K$, and let $\{ (\lambda_i, u_i)\}_{i=1}^{n}$
be the eigenvalue-eigenvector pairs of $M$.
\cite{BengioETAL2003} derive the OOS extensions for a number of embeddings
as solutions to the least-squares problem
\begin{equation*}
\min_{ f(x) \in \bbR^d }
\sum_{i=1}^n \left( K(x,x_i) - \frac{1}{n} \sum_{j=1}^d \lambda_{j} f_j(x_i) f_j(x) \right)^2,
\end{equation*}
where $\calD = \{x_1,x_2,\dots,x_n\}$ are the in-sample observations
and $f_j(x_i)$ is the $i$-th component of $u_j$.
A different OOS extension for MDS was considered in \cite{TroPri2008}.
Instead of the least-squares framework of \cite{BengioETAL2003},
\citet{TroPri2008} frame the MDS OOS extension problem as a modification of
the optimization problem solved by the in-sample MDS embedding.

An approach to the Laplacian eigenmaps OOS extension,
different from the one presented here,
was pursued in \cite{BelNiySin2006},
incorporating regularization in both the geometry of the
training data and the geometry of the similarity function $K$.
Their approach can also be extended to regularized least squares, SVM
and a variant of SVM in which a Laplacian penalty
term is added to the SVM objective.
The authors showed that all of these OOS extensions are the solutions to
generalized eigenvalue problems.
\cite{LevJanVan2015} provides an illustrative example of the practical
application of these OOS extensions, using the OOS extension of
\cite{BelNiySin2006} to build an audio search system.
More recent OOS extension techniques have attempted to avoid altogether the
need to solve least squares or eigenvalue problems, instead training a
neural net to learn the embedding, so that at out-of-sample embedding time
one need only feed the out-of-sample observation as input to the neural net
\citep[see, for example,][]{QuiPetHeu2016,JanSelLyz2017}.

As far as we are aware,
the only work to date on the OOS extension for ASE
appears in \citet{TanParPri2013},
in which the authors considered the OOS extension problem for certain
\emph{latent space} models of graphs
\citep[see, for example,][]{HofRafHan2002}.
These are models in which each vertex has an associated latent
vector in a Hilbert space,
with edge probabilities determined by inner products
between the latent vectors in this Hilbert space.
The authors presented an OOS extension based on a least-squares objective
and proved a result, analogous to our Theorem~\ref{thm:ase:lsrate},
given the rate of growth of the error between this out-of-sample embedding
and the true out-of-sample latent position.
Theorem~\ref{thm:ase:lsrate} yields a simplification of the proof of the
result originally appearing in \citet{TanParPri2013}, specialized to the
random dot product graph model
(see Definition~\ref{def:rdpg} below).
We note, however, that our results can be extended to more general latent
space network models under suitable conditions on the inner product.

Largely missing from the literature, but of particular importance to the
assessment of OOS extensions, is the comparison of the OOS estimate's
performance compared to its in-sample counter-part.
That is, for training sample $\calD$ and out-of-sample observation
$z \in \calX$ (both drawn, perhaps, from a probability distribution on $\calX$),
how closely does the out-of-sample embedding approximate
its in-sample counterpart computed based on $\calDtilde = \calD \cup \{z\}$?
In this work, we address this question as it pertains to the adjacency spectral
embedding (ASE) and the Laplacian spectral embedding (LSE; an embedding
closely related to the Laplacian eigenmaps embedding but more amenable to
analysis; see Section~\ref{sec:oos}).
In particular, we show the following:
\begin{itemize}
\item Two different approaches to the ASE OOS extension problem
        yield OOS extensions that recover the true out-of-sample latent
        position at a rate that matches the in-sample estimation error rate.
        The first (Theorem~\ref{thm:ase:lsrate}),
        based on a linear least squares objective,
        holds under essentially no conditions on the model.
        The second (Theorem~\ref{thm:ase:mlrate}),
        based on a maximum-likelihood objective,
        requires mild regularity conditions.
\item An LSE OOS extension based on a linear least-squares objective that,
        similarly to the ASE OOS extensions,
        recovers the true out-of-sample latent position at the same rate
        as the in-sample embedding
        (Theorem~\ref{thm:lse:lsrate}).
\item Both of the LLS-based OOS extensions obey central limit theorems
        (Theorems~\ref{thm:ase:lsclt} and \ref{thm:lse:lsclt}),
	with each OOS extension 
        asymptotically normally distributed about the true latent position
        (in the case of ASE) or a transformation thereof (in the case of LSE).
\end{itemize}

We believe that analogous central limit theorems can be obtained for other
OOS extensions such as those presented in~\citet{BengioETAL2003}
and for the maximum-likelihood ASE OOS extension,
but do not pursue this generalization here.

\subsection{Notation} \label{subsec:notation}

Before continuing, we pause to establish notation.
For a matrix $M \in \bbR^{n_1 \times n_2}$, we denote by $\sigma_i(M)$
the $i$-th singular value of $M$, so that
$\sigma_1(M) \ge \sigma_2(M) \ge \dots \ge \sigma_k(M) \ge 0$,
where $k = \min\{n_1,n_2\}$.
For integer $k > 0$, we let $[k] = \{1,2,\dots,n\}$.
Throughout the paper, $n$ will denote the number of vertices
in the observed graph $G$. For a vector $x$, the unadorned norm
$\| x \|$ will denote the Euclidean norm of $x$,
while for all $p > 0$, $\| x \|_p$ will denote the $p$-norm of $x$,
where $\| x \|_\infty = \max_i | x_i |$.
For a matrix $M$, $\| M \|_F$ will denote the Frobenius norm,
$\| M \|$ will denote the spectral norm
\begin{equation*}
\| M \| = \sup_{x : \|x\|=1 } \| M x \|,
\end{equation*}
and $\| M \|_{\tti}$ will denote the $2$-to-$\infty$ norm,
\begin{equation*}
\| M \|_{\tti} = \sup_{x : \|x\|=1 } \| M x \|_\infty.
\end{equation*}

Most of our results will concern the behavior of certain
quantities as the number of vertices $n$ increases to $\infty$.
We will often, for ease of notation, suppress this dependence on $n$, but
it should be assumed throughout that all quantities are dependent on $n$,
with the exception of the distribution $F$ and the latent space dimension $d$.
Thus, for example, we will in several places refer to a
``sequence of matrices'' $Q \in \bbR^{d \times d}$, where we suppress what
ought to be, say, a subscript $n$.
Throughout, $C > 0$ denotes a positive constant, not depending on $n$,
whose value may change from line to line or even, occasionally, within the
same line.
Given an event $E$, we let $E^c$ denote its complement, and let
$\Pr[ E ]$ denote the probability of event $E$ (the probability measure in
question will always be clear from context).
Given a collection of events $\{ E_n \}$ indexed by $n$,
suppose that with probability $1$ there exists $n_0$ such that
$E_n$ occurs whenever $n \ge n_0$.
If this is the case, we say that $E_n$ occurs eventually
or, by a slight abuse of terminology, say simply that $E_n$ occurs.

We make standard use of the big-$O$, big-$\Omega$ and big-$\Theta$
notation. Thus, for example, we write $f(n) = O(g(n))$ to denote the existence
of a constant $C > 0$ such that for all suitably large $n$, $f(n) \le Cg(n)$.
We write $f(n) = \Otilde(g(n))$ to mean that $f(n) = O(g(n))$ ignoring
logarithmic factors.
That is, if there exists a $c > 0$ such that $f(n) = O(g(n) \log^c n)$
(throughout the paper, $c$ is never larger than 2 or 3 and is typically $1/2$).
Our one slight abuse of this notation is in the case where,
letting $\{ Z_n \}$ be a sequence of random variables,
we write $Z_n = O(g(n))$ to mean that there exists a constant $C > 0$ such
that almost surely there exists $n_0$ such that
$|Z_n| \le C g(n)$ for all $n \ge n_0$, replacing the modulus with an
appropriate norm when $Z_n$ is a vector or matrix.
Most results in this paper are of this form.
We note that throughout, we prove these results by showing first
that $\Pr[ |Z_n| \ge Cg(n) ] \le Cn^{-(1+\epsilon)}$
is summable for all suitably small $\epsilon > 0$.
We then use the independence of $\{ Z_n : n=1,2,\dots \}$
to invoke the Borel-Cantelli
lemma \citep{Billingsley1995} to conclude that $Z_n = O(g(n))$.
Thus, though many of our results are stated as holding asymptotically,
they all have finite-sample analogues obtained in the course of their proofs.

\subsection{Roadmap}
The remainder of this paper is structured as follows.
In Section~\ref{sec:oos}, we formalize the graph out-of-sample extension
problem, and introduce a few methods for constructing such extensions.
In Section~\ref{sec:theory},
we present our main theoretical results, proving concentration and asymptotic
distributions for these extensions.
Section~\ref{sec:expts} gives an experimental investigation of the properties
of these embeddings.
We conclude in Section~\ref{sec:discussion} with a brief discussion of
directions for future work.

\section{Out-of-sample Extension for ASE and LSE}
\label{sec:oos}
Given a graph $G = ([n],E)$ with adjacency matrix $A \in \{0,1\}^{n \times n}$,
the adjacency spectral embedding \citep[ASE;][]{SusTanFisPri2012}
and the Laplacian spectral embedding \citep[LSE;][]{TanPri2018}
each provide a mapping of the $n$ vertices of $G$ into $\bbR^d$.
The ASE maps the vertices of $G$ to $d$-dimensional representations
$\Xhat_1,\Xhat_2,\dots,\Xhat_n \in \bbR^d$ given by the rows of the matrix
\begin{equation} \label{eq:def:ase}
        \Xhat = \ASE(A,d) = \Uhat \Shat^{1/2} \in \bbR^{n \times d},
\end{equation}
where $\Shat \in \bbR^{d \times d}$ is the diagonal matrix with entries given
by the top $d$ eigenvalues of $A$ and
the columns of $\Uhat \in \bbR^{n \times d}$ are the corresponding
orthonormal eigenvectors.
The Laplacian spectral embedding \cite[LSE;][]{TanPri2018}
proceeds according to a similar eigenvalue truncation,
applied to the normalized graph Laplacian,
\begin{equation*} \label{eq:def:laplacian}
  L = \calL(A) := D^{-1/2} A D^{-1/2},
\end{equation*}
where $D \in \bbR^{n \times n}$ is the diagonal degree matrix,
with $D_{i,i} = \sum_{j=1}^n A_{i,j}$, with $0^{-1/2} = 0$ by convention.
The LSE embeds the vertices of $G$ as
$\Xcheck_1,\Xcheck_2,\dots\,\Xcheck_n \in \bbR^d$
given by the rows of the matrix
\begin{equation} \label{eq:def:lse}
  \Xcheck = \LSE(A,d) = \Ucheck \Scheck^{1/2} \in \bbR^{n \times d},
\end{equation}
where $\Scheck \in \bbR^{d \times d}$ is the diagonal matrix formed of
the $d$ largest-magnitude eigenvalues of the graph Laplacian $L$ and
$\Ucheck \in \bbR^{n \times d}$ is the matrix
formed of the $d$ corresponding orthonormal eigenvectors.
The well-known Laplacian eigenmaps embedding \citep{BelNiy2003}
corresponds to a rescaling of the LSE,
in that the Laplacian eigenmaps embedding is given by the rows of
$\Uhat \in \bbR^{n \times d}$. As such, results similar to those presented
here for the LSE can be obtained for the Laplacian eigenmaps embedding as well.

We note that in both of the embeddings just described, there may be a concern
that the $d$ largest-magnitude eigenvalues need not all be positive,
and hence square roots $\Shat^{1/2}$ and $\Scheck^{1/2}$ will be ill-defined.
As a result, it may be preferable, in general, to consider
instead the top-$d$ singular values of $A$ and $L$.
We will not consider this issue in the present work, since
under the model considered in this paper (see Definition~\ref{def:rdpg} below),
with probability $1$
the $d$ largest-magnitude eigenvalues will be positive
for all suitably large $n$.

\begin{remark}[Comparing ASE and LSE]
Both the ASE and LSE yield low-dimensional representations                      of the vertices of $G$,
and it is natural to ask which embedding is preferable.
The answer, in general, is dependent on the precise model under consideration
and the intended downstream task.
For example, one can show that neither the ASE nor the Laplacian embedding
strictly dominates in a vertex classification task.
Section 4 of \cite{TanPri2018} demonstrates that ASE performs
better than the Laplacian embedding when applied to graphs with a
core-periphery structure.
Such structures are ubiquitous in real networks;
see, for example, \cite{LesLanDasMah2009} and \cite{JeuETAL2015}.
We refer the interested reader to \cite{CapTanPri2018}
for a more thorough theoretical treatment of this point.
\end{remark}

The two embeddings just discussed
are especially well-suited to the random dot product
graph \citep[RDPG;][]{YouSch2007,RDPGsurvey}, a model
in which graph structure is driven by
the geometry of latent positions associated to the vertices.
\begin{definition} {\em (Inner product distribution)}
	A distribution $F$ on $\bbR^d$ is a
	{\em $d$-dimensional inner product distribution }
	if $0 \le x^T y \le 1$ whenever $x,y \in \supp F$.
\end{definition}
\begin{definition}\label{def:rdpg}
        \emph{(Random Dot Product Graph)}
        Let $F$ be a $d$-dimensional inner product distribution,
	and let $X_1,X_2,\dots,X_n \iid F$ be collected in the rows of
	$X \in \bbR^{n \times d}$.
	Let $G$ be a random graph with adjacency matrix
	$A \in \{0,1\}^{n \times n}$.
	We say that $G$ is a {\em random dot product graph} (RDPG)
	with {\em latent positions} $X_1,X_2,\dots,X_n \in \bbR^d$,
	if the edges of $G$ are independent conditioned on
	$\{X_1,X_2,\dots,X_n\}$, with
        \begin{equation} \label{eq:rdpg}
                \Pr[ A | X ]=
                \prod_{1 \le i<j \le n}
		(X_i^TX_j)^{A_{i,j}}(1-X_i^TX_j)^{1-A_{i,j}}.
        \end{equation}
	We say that $X_i$ is the latent position associated to the $i$-th
	vertex in $G$, and write
	$(A,X) \sim \RDPG(F,n)$ to mean that the rows of
	$X \in \bbR^{n \times d}$ are drawn i.i.d.\ from $F$ and
	that $A \in \{0,1\}^{n \times n}$ is generated according to
	Equation~\eqref{eq:rdpg} conditional on $X$.
\end{definition}
Note that the RDPG has an inherent nonidentifiability, owing to the fact that
the distribution of $A$
is unchanged by an orthogonal rotation of the latent positions:
for latent position matrix $X \in \bbR^{n \times d}$ and
orthogonal matrix $W \in \bbR^{d \times d}$,
both $X \in \bbR^{n \times d}$ and $X W \in \bbR^{n \times d}$ give rise
to the same distribution over adjacency matrices, in that
$\E[ A \mid X] = X X^T = X W (XW)^T$.
Thus, we can only ever hope to recover the latent positions of the RDPG
up to some orthogonal transformation.
Throughout this work, we denote by
$\Delta = \E X_1 X_1^T \in \bbR^{d \times d}$
the second moment matix of the latent position distribution $F$.
Our results require that $\Delta$ be of full rank, an assumption that we
make without loss of generality owing to the fact that if $\Delta$ is of,
say, rank $d' < d$, then we may equivalently think of $F$ as a
$d'$-dimensional inner product distribution by restricting our attention to
an appropriate $d'$-dimensional subspace of $\bbR^d$.

\begin{remark} {\em (Extension to other graph models)}
As alluded to above, the RDPG as defined here only captures graphs with
positive semi-definite expected adjacency matrices.
This limitation can be avoided
by considering the {\em generalized} RDPG \citep{RubPriTanCap2017}.
The results stated in the present work can for the most part be extended
to this model, at the expense of additional notational complexity,
which we prefer to avoid here.
Similarly, using standard concentration inequalities, most of the results
presented here can be extended beyond binary edges
to consider independent edges that are unbiased ($\E A_{i,j} = X_i^T X_j$)
with sub-Gaussian or sub-gamma tails \citep{BLM,Tropp2015}.
\end{remark}

Throughout this paper, we will assume that $(A,X) \sim \RDPG(F,n)$ for
some $d$-dimensional inner product distribution $F$,
and write $P = \E[ A \mid X] = X X^T$.
Under this setting, it is clear that $\Xhat = \ASE(A,d)$ is a natural
estimate of the matrix of true latent positions $X$.
Further, $\Xcheck = \LSE(A,d)$ is a natural estimate of
$\Xtilde = T^{-1/2} X$, where $T \in \bbR^{n \times n}$ is a diagonal matrix
with entries $T_{i,i} = \sum_j X_j^T X_i$.
The rows of $\Xtilde$ can be thought of as the Laplacian spectral embeddings
of the matrix $P = X X^T$, in the sense that
$\Xtilde \Xtilde^T = \calL( P )$.
Indeed, it has been shown previously that the ASE consistently estimates the
latent positions in the RDPG \citep{SusTanFisPri2012,TanSusPri2013},
and successfully recovers community structure in the (positive semi-definite)
stochastic block model \citep{LyzSusTanAthPri2014},
which can be recovered as a special case of the RDPG by taking
the distribution $F$ to be a mixture of point masses.
Similar results can be shown for the LSE \citep{TanPri2018}.
\begin{lemma} \label{lem:2toinfty}
Let $(A,X) \sim \RDPG(F,n)$ for some $d$-dimensional inner product distribution
$F$ and let $\Xhat, \Xcheck, \Xtilde \in \bbR^{n \times d}$ be as above.
Then there exists a sequence of orthogonal matrices
$Q \in \bbR^{d \times d}$ such that
\begin{equation} \label{eq:2toinfty:ase}
\| \Xhat - X Q \|_{2,\infty} = O\left( \frac{ \log n }{ \sqrt{n} } \right).
\end{equation}
Further, if there exists a constant $\eta > 0$ such that
$\eta \le x^T y \le 1-\eta$ whenever $x,y \in \supp F$, then
there exists a sequence of orthogonal matrices
$\Qtilde \in \bbR^{d \times d}$ such that
\begin{equation} \label{eq:2toinfty:lse}
\| \Xcheck - \Xtilde \Qtilde \|_{2,\infty}
	= O\left( \frac{ \log^{1/2} n }{ n } \right).
\end{equation}
\end{lemma}
\begin{proof}
The bound in Equation~\eqref{eq:2toinfty:ase} is
Lemma 5 in \cite{LyzSusTanAthPri2014}.
A proof of Equation~\eqref{eq:2toinfty:lse}
can be found in Appendix~\ref{apx:general}.
\end{proof}

Suppose that graph $G = ([n],E)$
with adjacency matrix $A \in \bbR^{n \times n}$
is a random dot product graph, so that $(A,X) \sim \RDPG(F,n)$,
and we compute
\begin{equation*} 
\Xhat = \ASE(A,d) = [\Xhat_1 \Xhat_2 \cdots \Xhat_n ]^T \in \bbR^{n \times d}
\text{ and }
\Xcheck = \LSE(A,d)
	= [\Xcheck_1 \Xcheck_2 \cdots \Xcheck_n ]^T \in \bbR^{n \times d},
\end{equation*}
where $\Xhat_i, \Xcheck_i \in \bbR^d$ are embeddings of the $i$-th vertex
under ASE and LSE, respectively.
Suppose now that a vertex $v$ having latent position $\wtrue \in \supp F$
is added to the graph $G$ to form $\Gtilde = ([n] \cup \{v\}, E \cup E_v)$,
where $E_v \subseteq \{ \{i,v\} : i=1,2,\dots,n\}$.
The edges between the out-of-sample vertex $v$ and the in-sample vertices
$\{1,2,\dots,n\}$ are specified by
a vector $\avec \in \{0,1\}^n$ such that $a_i = 1$ if $\{i,v\} \in E_v$
and $a_i = 0$ otherwise.
Thus, $\Gtilde$ has adjacency matrix 
$\Atilde$ as in Equation~\eqref{eq:def:Atilde} above.
Having computed an embedding $\Xhat$ or $\Xcheck$, we would like to embed
the vertex $v$ to obtain an estimate
of the true latent position $\wtrue$ (in the case of ASE) or,
in the case of LSE,
its Laplacian spectral embedding
$\wtilde = \wtrue/\sqrt{n\mu^T \wtrue} \in \bbR^d$,
where $\mu = \E X_1$ is the mean of $F$.
In the case of ASE,
the out-of-sample extension problem concerns how to compute an estimate
of $\wtrue$ based only on $\Xhat$ and $\avec$.
Similarly, in the case of LSE,
the out-of-sample extension problem requires computing an estimate of $\wtilde$
based only on the information in $\Xhat$, $\avec$
and, for reasons that will become clear below,
the vector of in-sample vertex degrees, $\dvec \in \bbR^n$.

\subsection{Out-of-sample extension for ASE}

Two natural approaches to the out-of-sample extension of ASE suggest
themselves. The first, following \cite{BengioETAL2003},
involves embedding the out-of-sample vertex $v$ as
\begin{equation} \label{eq:def:ase:llshat}
  \whatls
  = \arg \min_{w \in \bbR^d} \sum_{i=1}^n \left( a_i - \Xhat_i^T w \right)^2,
\end{equation}
where $a_i$ is the $i$-th component of the vector $\avec \in \bbR^n$ of
edges between the out-of-sample vertex and the in-sample vertices.
We refer to $\whatls$ as the {\em linear least squares out-of-sample}
(LLS OOS) extension of adjacency spectral embedding.

An alternative approach to the OOS extension problem, perhaps more
appealing from a statistical perspective, but more computationally expensive,
is to cast the OOS extension as a maximum-likelihood problem.
Letting $X_1,X_2,\dots,X_n \in \bbR^d$ be the true latent positions of the
in-sample vertices and $\wtrue \in \bbR^d$ be the true latent position of
the out-of-sample vertex, the entries of $\avec$
are independent Bernoulli random variables, with
$a_i \sim \Bernoulli( X_i^T \wtrue )$.
Thus, the log likelihood (conditional on the in-sample latent positions) is
\begin{equation*}
\ell(w) = \sum_{i=1}^n a_i \log X_i^T w + (1-a_i) \log(1-X_i^T w).
\end{equation*}
Of course, in practice we observe the latent positions only through
their ASE estimates $\{ \Xhat_i \}_{i=1}^n \subseteq \bbR^d$.
Thus, we define the maximum-likelihood out-of-sample extension for ASE
as the maximizer of the plug-in likelihood, i.e., as the solution to
\begin{equation} \label{eq:def:eigenlikhatunbound}
        \max_{w \in \bbR^d} \sum_{i=1}^n a_i \log \Xhat_i^T w 
        + (1-a_i)\log \left( 1 - \Xhat_i^T w \right).
\end{equation}
Unfortunately, this objective need not achieve its optimum inside the
support of $F$. Indeed, the objective need not even be bounded.
Thus, we will settle for a slight reformulation of this objective,
and define the maximum-likelihood out-of-sample (ML OOS) extension for ASE
to be the solution to a constrained maximum-likelihood problem,
\begin{equation} \label{eq:def:eigenlikhat}
  \whatml = \arg \max_{w \in \calThat_\epsilon}
        \sum_{i=1}^n a_i \log \Xhat_i^T w 
        + (1-a_i)\log \left( 1 - \Xhat_i^T w \right),
\end{equation}
where $\calThat_\epsilon =
\{ w \in \bbR^d : \epsilon \le \Xhat_i^T w \le 1-\epsilon, i \in [n] \}$,
and $\epsilon > 0$ is some small constant.
We note that we call this the maximum-likelihood OOS extension, though
it is, strictly speaking, based on a plug-in approximation to the true
likelihood given in Equation~\eqref{eq:def:eigenlikhatunbound}.

Note that, as required by the out-of-sample problem, both
$\whatls$ and $\whatml$ are functions only of the in-sample embedding
$\Xhat \in \bbR^{n \times d}$ and the edges between the out-of-sample vertex
$v$ and the in-sample vertices $[n]$,
as encoded in the vector $\avec \in \bbR^n$.

\subsection{Out-of-sample extension for LSE}

Recall that given the adjacency matrix $A$ of graph $G = ([n],E)$,
we form the sample graph Laplacian $L = \calL(A) = D^{-1/2} A D^{-1/2}$
and embed in-sample vertex $i \in [n]$ as $\Xcheck_i \in \bbR^d$,
the $i$-th row of
\begin{equation*}
  \Xcheck = \Ucheck \Scheck^{1/2} \in \bbR^{n\times d},
\end{equation*}
where we remind the reader that $\Ucheck \in \bbR^{n \times d}$ denotes the
matrix formed by the top $d$ orthonormal eigenvectors of $L$ with their
corresponding eigenvalues collected in the diagonal matrix
$\Scheck \in \bbR^{d \times d}$.
Conditional on the latent positions $X_1,X_2,\dots, X_n \iid F$,
we have $\E[ A \vert X] = X X^T = P \in \bbR^{n \times n}$, and we view
$L = \calL(A)$ as an estimate of
$\calL(P) = T^{-1/2} P T^{-1/2}$,
where $T \in \bbR^{n \times n}$ is the matrix of
(conditional) expected degrees,
$T_{i,i} = \sum_{j=1}^n P_{i,j} = \sum_{j=1}^n X_i^T X_j$.
Applying the LSE to $\calL(P)$, we may think of the rows of
\begin{equation*}
 \Xtilde = \Utilde \Stilde^{1/2} \in \bbR^{n \times d}
\end{equation*}
as the ``true'' Laplacian spectral embedding, and view $\Xcheck$ as an
estimate of this quantity.

Given out-of-sample vertex $v$ with latent position $\wtrue \in \bbR^d$,
the natural Laplacian embedding of $v$,
in light of the definition of $\Xtilde$,
is given by $\wtilde = \wtrue/\sqrt{n \mu^T \wtrue}$,
where $\mu = \E X_1 \in \bbR^d$ is the mean of $F$.
Of course, in practice we must compute the out-of-sample embedding of $v$
based on $\Xcheck \in \bbR^{n\times d}$
and the vector of edges $\avec \in \bbR^n$
to obtain an estimate of $\wtilde$.
In applying the least-squares approach suggested by
Equation~\eqref{eq:def:ase:llshat} and used in \cite{BengioETAL2003},
it is most natural to consider the minimizer
\begin{equation} \label{eq:def:lse:llshat}
\wcheckls =
        \arg \min_{w \in \bbR^d} \sum_{i=1}^n
        \left( \frac{a_i}{ \sqrt{ d_v d_i } }
                - \Xcheck_i^T w \right)^2,
\end{equation}
where $d_i = \sum_{j=1}^n A_{i,j}$ is the degree of the $i$-th in-sample
vertex, and $d_v = \sum_i a_i$ is the degree of the out-of-sample vertex $v$.
We refer to $\wcheckls$ as the
LLS OOS extension of the Laplacian spectral eembedding.
We note that Equation~\eqref{eq:def:lse:llshat} requires that we keep
in-sample vertex degree information for use in the out-of-sample
extension, which violtates the typical requirement that
we compute the out-of-sample extension using only $\Xcheck$ and $\avec$.
Nonetheless,
it is reasonable to allow the use of the vector $\dvec$,
since typically the embedding dimension $d$ is of a smaller order than $n$
and thus the space required to store node
degrees is of the same or smaller order as that required to
store $\Xcheck \in \bbR^{n \times d}$.
We note that one could avoid this additional storage by replacing $d_i$ with
$\sum_{j=1}^n \Xcheck_j^T \Xcheck_i$
and all our results below would go through (see Lemma~\ref{lem:degreegrowth}),
but this would come at the
expense of notational inconvenience and longer proofs below.
The motivation for the least-squares objective in
Equation~\eqref{eq:def:lse:llshat} becomes clear if we think of
$d_v^{-1/2} d_i^{-1/2} a_i$ as an estimate of the normalized kernel
\begin{equation*}
\Kbar( i, v ) = \frac{ X_i^T \wtrue }
        { n \sqrt{ X_i^T\mu \wtrue^T\mu } },
\end{equation*}
where $\mu \in \bbR^d$ is again the mean of $F$.

\section{Theoretical Results}
\label{sec:theory}

The main results of this paper concern concentration inequalities and
central limit theorems for the OOS extensions introduced in
Section~\ref{sec:oos}. We first present the concentration inequalities,
which allow us to control the rate of convergence of the OOS extension
to the parameter of interest,
given by the true OOS latent position $\wtrue$ in the case of ASE,
and by the transformed latent position $\wtilde = \wtrue/\sqrt{n\mu^T\wtrue}$
in the case of LSE.

\subsection{Rates of convergence for OOS extensions}

A first question surrounding the OOS extensions presented in the preceding
section concerns their quality as estimators of their respective true
parameters.
Interestingly, all of the OOS extensions presented above recover their
respective target parameters at asymptotic rates that match that of
the full-graph embedding.

We begin by considering the ASE OOS extensions defined in
Equations~\eqref{eq:def:ase:llshat} and~\eqref{eq:def:eigenlikhat}.
Both of these estimates recover the true out-of-sample latent position $\wtrue$
at the same asymptotic rate
(see Theorems~\ref{thm:ase:lsrate} and~\ref{thm:ase:mlrate} below),
and this rate matches the one we would
obtain if we were to compute the ASE of the augmented graph
$\Gtilde$ with adjacency matrix $\Atilde$,
given in Lemma~\ref{lem:2toinfty}.
We find that the estimation error between the least squares OOS extension for
ASE $\whatls$ and the true latent position $\wtrue$ follows the same rate.
\begin{theorem} \label{thm:ase:lsrate}
	Let $F$ be a $d$-dimensional inner-product distribution and
	suppose $(A,X) \sim \RDPG(F,n)$. Let $v$ denote the out-of-sample
	vertex, and denote its latent position by $\wtrue \in \supp F$.
	Let $\whatls$ denote the LS-based OOS extension for ASE
	based on $\Xhat = \ASE(A,d)$ and the vector of edges $\avec \in \bbR^n$
	between $v$ and the in-sample vertices,
	as defined in Equation~\eqref{eq:def:ase:llshat}.
	There exists a sequence of orthogonal matrices $Q \in \bbR^{d\times d}$
	such that
        \begin{equation*}
	\| Q \whatls - \wtrue \| = O( n^{-1/2} \log n),
	\end{equation*}
	and this matrix $Q$ is the same one guaranteed
	by Lemma~\ref{lem:2toinfty}.
\end{theorem}
\begin{proof}
	A standard result for solutions of perturbed linear systems allows us
	to show that with high probability,
        $\| Q \whatls - \wls \| \le Cn^{-1/2} \log n$,
	where $Q \in \bbR^{d \times d}$
        is the orthogonal matrix guaranteed by Lemma~\ref{lem:2toinfty} above
	and $\wls$ is the least-squares minimizer obtained if one uses the
	true latent positions $\{ X_i \}$ rather than the ASE estimates
	$\{ \Xhat_i \}$ in Equation~\eqref{eq:def:ase:llshat}.
	Hoeffding's inequality implies that
	$\| \wls - \wtrue \| = O(n^{-1/2} \log n)$.
	The result then follows by a triangle inequality applied to
	$\| Q\whatls - \wtrue\|$.
	A detailed proof can be found in Appendix~\ref{apx:ase:lsconc}.
\end{proof}

In a similar vein, the ML-based OOS extension also recovers the true
out-of-sample latent position at a rate that matches that
of the in-sample embedding,
given by Equation~\eqref{eq:2toinfty:ase} in Lemma~\ref{lem:2toinfty}.
\begin{theorem} \label{thm:ase:mlrate}
	Let $F$ be a $d$-dimensional inner-product distribution
	for which there exists a constant $\eta > 0$
	such that $\eta < x^T y < 1-\eta$ for all $x,y \in \supp F$.
	Suppose that $(A,X) \sim \RDPG(F,n)$ and let $v$ be an
	out-of-sample vertex with latent position $\wtrue \in \supp F$.
        Let $\whatml$ be the out-of-sample embedding defined in
        Equation~\eqref{eq:def:eigenlikhat}, with $\epsilon >0$
	chosen so that $\epsilon < \eta$.
	Then there exists a sequence of orthogonal matrices
	$Q \in \bbR^{d \times d}$ such that
        \begin{equation*}
	\| Q \whatml - \wtrue \| = O( n^{-1/2} \log n ),
	\end{equation*}
        and this matrix $Q$ is the same one guaranteed by
	Lemma~\ref{lem:2toinfty}.
\end{theorem}
\begin{proof}
	Using the definition of $\calThat_\epsilon$ and a standard argument
	from convex optimization, one can show that with probability $1$,
	it holds for all suitably large $n$ that
	\begin{equation*}
	\| Q \whatml - \wtrue \| \le
	\frac{ C\| \nabla \ellhat( Q^T \wtrue ) \| }{ n }.
	\end{equation*}
	An application of the triangle inequality and standard concentration
	inequalities yields
	\begin{equation*}
	\| \nabla \ellhat( Q^T \wtrue ) \| = O( \sqrt{n} \log n ).
	\end{equation*}
	A detailed proof can be found in Appendix~\ref{apx:ase:mlconc}.
\end{proof}

In keeping with the above two results,
the least-squares LSE OOS extension given in
Equation~\eqref{eq:def:lse:llshat} recovers the true out-of-sample Laplacian
embedding $\wtilde$ at a rate that matches that of the
Laplacian spectral embedding $\wtilde$ of the augmented graph $\Gtilde$,
given by Equation~\eqref{eq:2toinfty:lse} in Lemma~\ref{lem:2toinfty}.
\begin{theorem} \label{thm:lse:lsrate}
Let $F$ be a $d$-dimensional inner-product distribution
with mean $\mu = \E X_1$,
and suppose that there exists a constant $\eta > 0$
such that $\eta < x^T y < 1-\eta$ for all $x,y \in \supp F$.
Let $(A,X) \sim \RDPG(F,n)$, let $v$ be an
out-of-sample vertex with latent position $\wtrue \in \supp F$,
and let $\wtilde = \wtrue/\sqrt{n \mu^T \wtrue}$ be the Laplacian spectral
embedding of this latent position.
Then there exists a sequence of orthogonal matrices
$\Qtilde \in \bbR^{d \times d}$ such that
\begin{equation*}
\| \Qtilde \wcheckls - \wtilde \| \le Cn^{-1} \log^{1/2} n ,
\end{equation*}
and this matrix $\Qtilde$ is the same one guaranteed by
Lemma~\ref{lem:2toinfty}.
\end{theorem}
\begin{proof}
Letting $\wtildels$ denote the LLS OOS solution if we had access to the
true latent positions,
the triangle inequality and unitary invariance of Euclidean norm bound
\begin{equation*}
\| \Qtilde \wcheckls - \wtilde \|
\le \| \Qtilde \wcheckls - \wtildels \| + \| \wtildels - \wtilde \|.
\end{equation*}
Both of these terms can be bounded using standard concentration inequalities
and properties of linear least-squares solutions.
A detailed proof is given in Appendix~\ref{apx:lseconc}.
\end{proof}

\subsection{Central limit theorems for the OOS extensions}

We now turn our attention to the question of the asymptotic distribution
of the OOS extensions introduced in Section~\ref{sec:oos}.
Once again, we state the results for the case of Bernoulli edges, but
similar results can be shown for a broader class of edge noise models,
provided that noise model and the latent position distribution $F$ obey
suitable moment conditions.

\begin{theorem} \label{thm:ase:lsclt}
Let $F$ be a $d$-dimensional inner-product distribution and suppose
that $(A,X) \sim \RDPG(F,n)$ and let $v$ be the out-of-sample vertex
with latent position $\wtrue \in \supp F$.
Let $\whatls$ be the least-squares OOS extension as defined in
Equation~\eqref{eq:def:ase:llshat}. Then there exists a sequence of orthogonal
$d$-by-$d$ matrices $Q$ such that
\begin{equation*}
\sqrt{n}( Q \whatls - \wtrue ) \inlaw \calN(0, \Sigma_{F,\wtrue}),
\end{equation*}
where for any $w \in \supp F$, we define
\begin{equation} \label{eq:def:Sigma}
\Sigma_{F,w}
        = \Delta^{-1} \E\left[X_1^T w(1-X_1^T w)X_1 X_1^T \right]
        \Delta^{-1}, \end{equation}
and $\Delta = \E X_1 X_1^T$ is the second moment matrix of $F$.
\end{theorem}
\begin{proof}
This theorem follows by writing the ASE least-squares OOS extension
as a sum of two vectors, one of which converges in probability to $0$
using arguments similar to Theorem~\ref{thm:ase:lsrate},
and the other of which converges in distribution to a normal,
and applying Slutsky's lemma.
A detailed proof can be found in Appendix~\ref{apx:ase:lsclt}.
\end{proof}
If the latent position $\wtrue$ of the OOS vertex $v$
is itself distributed according to $F$,
integrating $\wtrue$ above with respect to $F$ yields the following corollary.
\begin{corollary} \label{cor:ase:lsclt:mixture}
Assume the same setup as Theorem~\ref{thm:ase:lsclt}, but suppose that
the true latent position of the out-of-sample vertex $v$ is given by
$\wtrue \sim F$, independent of $(A,X)$.
Then there exists a sequence of orthogonal matrices $Q \in \bbR^{d \times d}$
such that
\begin{equation*}
\sqrt{n} Q \whatls \inlaw
	\int \calN(w, \Sigma_{F,w}) dF(w),
\end{equation*}
where $\Sigma_{F,w}$ is as defined in Equation~\eqref{eq:def:Sigma}.
That is, $\sqrt{n} Q \whatls$ converges in distribution to a
mixture of normals with mixing distribution $F$.
\end{corollary}

Turning our attention to the LSE, we can obtain a similar CLT result for
the LSE OOS extension, once we adjust for the fact that the LSE does not
estimate the latent position $\wtrue$ but instead estimates the
vector $\wtilde = \wtrue/\sqrt{n \mu^T \wtrue}$, where $\mu \in \bbR^d$ is
the mean of the inner-product distribution $F$.
We note that the scaling of $\wtilde$ by the square root of the expected
degree means that we must scale by $n$ instead of the $\sqrt{n}$
scaling in the ASE CLTs above.
\begin{theorem} \label{thm:lse:lsclt}
Let $F$ be a $d$-dimensional inner-product distribution for which there
exists a constant $\eta > 0$ such that
$\eta \le x^T y \le 1-\eta$ whenever $x,y \in \supp F$.
Let $(A,X) \sim \RDPG(F,n)$ and let $v$ be the out-of-sample vertex
with latent position $\wtrue \in \supp F$.
Let $\wcheckls \in \bbR^d$ denote the least-squares OOS extension
of LSE as defined in Equation~\eqref{eq:def:lse:llshat}.
Then there exists a sequence of orthogonal matrices
$\Qtilde \in \bbR^{d \times d}$ such that
\begin{equation*}
n(\Qtilde \wcheckls - \wtilde) \inlaw \calN(0, \Sigmatilde_{F,\wtrue} ),
\end{equation*}
where for any $w \in \supp F$ we define
\begin{equation} \label{eq:def:Sigmatilde}
\Sigmatilde_{F,\wtrue} 
= \E \left[ \frac{ X_j^T\wtrue(1-X_j^T\wtrue) }{ \mu^T\wtrue }
        \left( \frac{\Deltatilde^{-1} X_j}{X_j^T \mu}                                                         - \frac{ \wtrue}{2\mu^T \wtrue } \right)
        \left( \frac{\Deltatilde^{-1} X_j}{X_j^T \mu}
			- \frac{ \wtrue}{2\mu^T \wtrue } \right)^T \right],
\end{equation}
with $\Deltatilde = \E X_1 X_1^T / \mu^T X_1$.
\end{theorem}
\begin{proof}
The proof follows similarly to that of Theorem~\ref{thm:ase:lsclt},
though it requires a more careful analysis to control convergence of the
degrees. Details are given in Appendix~\ref{apx:lseclt}.
\end{proof}

\section{Experiments}
\label{sec:expts}
In this section, we briefly explore our results through simulations.
We leave a more thorough experimental examination of our results,
particularly as they apply to real-world data, for future work.
We first give a brief exploration of how quickly the asymptotic
distribution in Theorem~\ref{thm:ase:lsclt} becomes a good approximation.
Toward this end, let us consider a simple mixture of point masses,
$ F = F_{\lambda,x_1,x_2} = \lambda \delta_{x_1} + (1-\lambda) \delta_{x_2}$,
where $x_1,x_2 \in \bbR^2$ and $\lambda \in (0,1)$.
This corresponds to a two-block stochastic block model
\citep{Holland1983},
in which the block probability matrix is given by
$$ \begin{bmatrix} x_1^T x_1 & x_1^T x_2 \\
                x_1^Tx_2 & x_2^Tx_2 \end{bmatrix}. $$
Corollary~\ref{cor:ase:lsclt:mixture} implies that if all latent positions
(including the OOS vertex) are drawn according to $F$,
then the OOS estimate should be distributed as a mixture
of normals centered at $x_1$ and $x_2$, with respective mixing coefficients
$\lambda$ and $1-\lambda$.

\begin{figure*}[ht!]
  \centering
  \includegraphics[width=\textwidth]{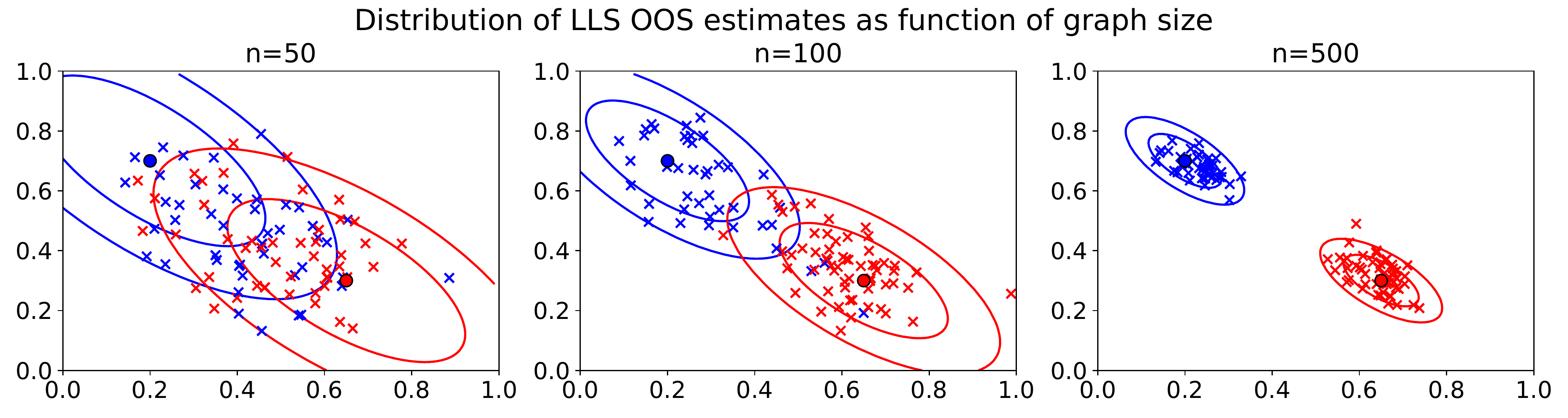}
  \vspace{-5mm}
  \caption{Observed distribution of the LLS OOS estimate for 100
        independent trials for number of vertices
        $n=50$ (left), $n=100$ (middle) and $n=500$ (right).
        Each plot shows the positions of 100 independent OOS
        embeddings, indicated by crosses,
        and colored according to cluster membership.
        Contours indicate two generalized standard deviations of the
        multivariate normal (i.e., 68\% and 95\% of the probability mass)
        about the true latent positions, which are indicated by solid circles.
        We note that even with merely 100 vertices, the normal
        approximation is already quite reasonable.}
  \label{fig:lls_nplot}
\end{figure*}

To assess how well the asymptotic distribution predicted by
Theorem~\ref{thm:ase:lsclt} and Corollary~\ref{cor:ase:lsclt:mixture} holds,
we generate RDPGs with latent positions
drawn i.i.d.\ from distribution $F = F_{\lambda,x_1,x_2}$ defined above, with
$$ \lambda = 0.4,~x_1 = (0.2,0.7)^T, \text{ and } x_2 = (0.65, 0.3)^T. $$
For each trial, we draw $n+1$ independent latent positions from
$F$, and generate a binary adjacency matrix from these latent positions.
We let the $(n+1)$-th vertex be the OOS vertex.
Retaining the subgraph induced by the first $n$ vertices, we obtain
an estimate $\Xhat \in \bbR^{n \times 2}$ via ASE,
from which we obtain an estimate for the OOS vertex
via the LS OOS extension as defined in~\eqref{eq:def:ase:llshat}.
We remind the reader that for each RDPG draw,
we initially recover the latent positions only up to a rotation.
Thus, for each trial, we compute a Procrustes alignment \citep{GowDij2004}
of the in-sample estimates $\Xhat$ to their true latent positions.
This yields a rotation matrix $R$, which we apply to the OOS estimate.
Thus, the OOS estimates are sensibly comparable across trials.
Figure~\ref{fig:lls_nplot}
shows the empirical distribution of the OOS embeddings
of 100 independent RDPG draws, for $n=50$ (left),
$n=100$ (center) and $n=500$ (right) in-sample vertices.
Each cross is the location of the OOS estimate for a single
draw from the RDPG with latent position distribution $F$,
colored according to true latent position.
OOS estimates with true latent position
$x_1$ are plotted as blue crosses, while OOS estimates
with true latent position $x_2$ are plotted as red crosses.
The true latent positions $x_1$ and $x_2$ are plotted as solid circles,
colored accordingly.
The plot includes contours for the two normals centered at $x_1$ and $x_2$
predicted by Theorem~\ref{thm:ase:lsclt} and Corollary~\ref{cor:ase:lsclt:mixture},
with the ellipses indicating the isoclines corresponding
to one and two (generalized) standard deviations.

Examining Figure~\ref{fig:lls_nplot}, we see that even with only 100 vertices,
the mixture of normal distributions predicted by Theorem~\ref{thm:ase:lsclt}
holds quite well, with the exception of a few gross outliers
from the blue cluster. With $n=500$ vertices, the approximation is
particularly good. Indeed, the $n=500$ case appears to be slightly
under-dispersed, possibly due to the Procrustes alignment.
It is natural to wonder whether a similarly good fit is exhibited by the
ML-based OOS extension.
We conjectured at the end of Section~\ref{sec:theory} that a CLT
similar to that in Theorem~\ref{thm:ase:lsclt} would also hold
for the ML-based OOS extension as defined in
Equation~\eqref{eq:def:eigenlikhat}.
Figure~\ref{fig:ml_nplot} shows the empirical distribution of 100 independent
OOS estimates, under the same experimental setup as Figure~\ref{fig:lls_nplot},
but using the ML OOS extension
rather than the linear least-squares extension.
The plot supports our conjecture that the ML-based OOS estimates are
also approximately normally distributed about the true latent positions.
Broadly similar patterns hold for the same experiment applied to the
least-squares LSE OOS extension, as predicted by
Theorem~\ref{thm:lse:lsclt}.

\begin{figure*}[ht!]
  \centering
  \includegraphics[width=\textwidth]{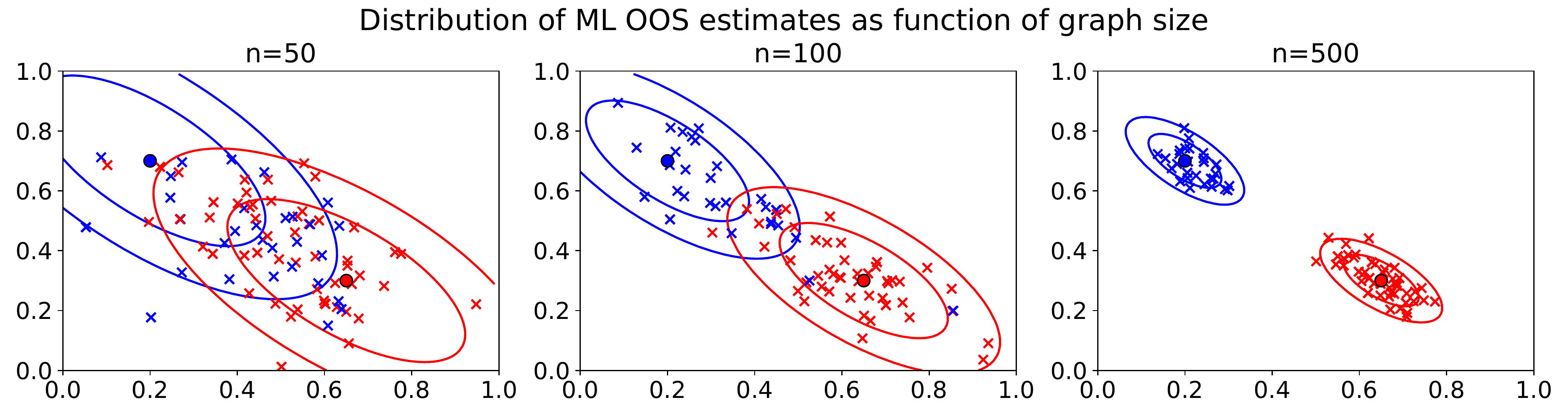}
  \vspace{-5mm}
  \caption{Observed distribution of the ML OOS estimate for 100
        independent trials for number of vertices
        $n=50$ (left), $n=100$ (middle) and $n=500$ (right).
        Each plot shows the positions of 100 independent OOS
        embeddings, indicated by crosses,
        and colored according to cluster membership.
        Contours indicate two generalized standard deviations of the
        multivariate normal
        about the true latent positions, which are indicated by solid circles.
        Once again, even with merely 100 vertices, the normal
        approximation is already quite reasonable, supporting our conjecture
        that the ML OOS estimates also distributed as a mixture of normals
        according to the latent position distribution $F$.}
  \label{fig:ml_nplot}
\end{figure*}

Figure~\ref{fig:lse_nplot} plots the same experiment as that performed
in Figures~\ref{fig:lls_nplot} and~\ref{fig:ml_nplot}, this time for
the linear least squares OOS extension of the Laplacian spectral embedding.
Recall that Theorem~\ref{thm:lse:lsclt} predicts that the out-of-sample
extension should be asymptotically normally distributed about the
true (rescaled) latent position $\wtilde = \wtrue/\sqrt{n \wtilde^T \mu}$.
Compared to the previous two experiments, it is evident that the
asymptotics are slightly slower to kick in, but modulo the same
Procrustes-induced underdispersion observed previously,
the theorem appears to hold quite well with $n=500$ vertices.

\begin{figure*}[ht!]
  \centering
  \includegraphics[width=\textwidth]{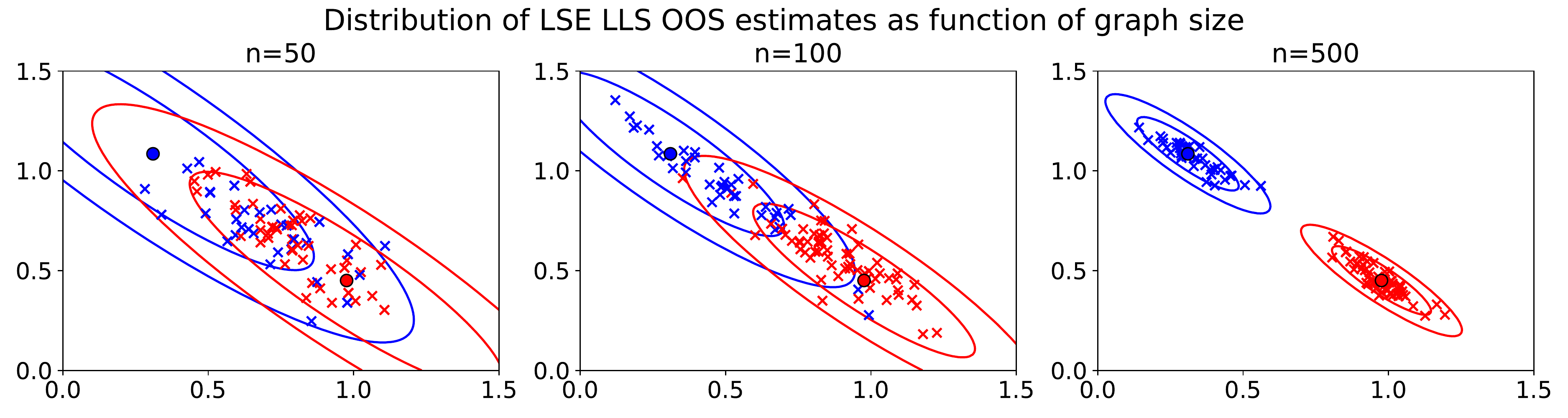}
  \vspace{-5mm}
  \caption{Observed distribution of the LSE OOS estimate for 100
        independent trials for number of vertices
        $n=50$ (left), $n=100$ (middle) and $n=500$ (right).
        Each plot shows the positions of 100 independent OOS
        embeddings, indicated by crosses,
        and colored according to cluster membership.
        Contours indicate two generalized standard deviations of the
        multivariate normal
        about the true latent positions, which are indicated by solid circles.}
  \label{fig:lse_nplot}
\end{figure*}

Figure~\ref{fig:lls_nplot} suggests that we may be confident in applying the
large-sample approximation suggested by Theorem~\ref{thm:ase:lsclt}
and Corollary~\ref{cor:ase:lsclt:mixture}.
Applying this approximation allows us to investigate the trade-offs
between computational cost and classification accuracy,
to which we now turn our attention.
The mixture distribution $F_{\lambda,x_1,x_2}$ above suggests a
task in which, given an adjacency matrix $A$, we wish to classify
the vertices according to which of two clusters or communities they belong.
That is, we will view two vertices as belonging to the same community if
their latent positions are the same
\citep[i.e., the latent positions specify an SBM,]{Holland1983}.
More generally, one may view the task of recovering vertex block memberships
in a stochastic block model as a clustering problem.
\citet{LyzSusTanAthPri2014} showed that applying ASE to such a graph,
followed by $k$-means clustering of the estimated latent positions,
correctly recovers community memberships of all the vertices
(i.e., correctly assigns all vertices to their true latent positions)
with high probability.

For concreteness, let us consider a still simpler mixture model,
$F = F_{\lambda,p,q} = \lambda \delta_p + (1-\lambda) \delta_q$,
where $0 < p < q < 1$,
and draw an RDPG $(\Atilde,X) \sim \RDPG(F,n+m)$,
taking the first $n$ vertices to be in-sample,
with induced adjacency matrix $A \in \bbR^{n \times n}$.
That is, we draw the full matrix
$$ \Atilde = \begin{bmatrix} A & B \\
                B^T & C \end{bmatrix}, $$
where $C \in \bbR^{m \times m}$
is the adjacency matrix of the subgraph induced by the $m$ OOS vertices
and $B \in \bbR^{n \times m}$ encodes the edges between the in-sample
vertices and the OOS vertices.
The latent positions $p$ and $q$ encode a community structure in
the graph $\Atilde$, and, as alluded to above,
a common task in network statistics
is to recover this community structure.
Let $\wtrue^{(1)}, \wtrue^{(2)}, \dots, \wtrue^{(m)} \in \{p,q\}$ denote the
true latent positions of the $m$ OOS vertices,
with respective least-squares OOS estimates
$\whatls^{(1)}, \whatls^{(2)}, \dots, \whatls^{(m)}$,
each obtained from the in-sample ASE $\Xhat \in \bbR^n$ of $A$.
We note that one could devise a different OOS embedding
procedure that makes use of the subgraph $C$ induced by these $m$
OOS vertices, but we leave the development of such a method
to future work.
Corollary~\ref{cor:ase:lsclt:mixture}
implies that each $\whatls^{(t)}$ for $t \in [m]$
is marginally (approximately) distributed as
$$ \whatls^{(t)} \sim \lambda \calN(p,(n+1)^{-1}\sigma^2_p)
+ (1-\lambda)\calN(q,(n+1)^{-1}\sigma^2_q), $$
where
\begin{equation*} \begin{aligned}
\sigma^2_p &= \Delta^{-2} \left(\lambda p^2(1-p^2)p^2
        + (1-\lambda)pq(1-pq)q^2 \right), \\
\sigma^2_q &= \Delta^{-2} \left( \lambda pq(1-pq)p^2
+ (1-\lambda)q^2(1-q^2)q^2 \right), \\
\text{ and } \Delta &= \lambda p^2 + (1-\lambda)q^2.
\end{aligned} \end{equation*}
Classifying the $t$-th OOS vertex based on $\whatls^{(t)}$
via likelihood ratio thus has (approximate) probability of error
\begin{equation*}
\eta_{n,p,q}
= \lambda(1 - \Phi\left( \frac{\sqrt{n+1}(x_{n+1,p,q} - p) }{ \sigma_p } \right)
+
(1-\lambda)\Phi\left( \frac{\sqrt{n+1}(x_{n+1,p,q} - q)}{ \sigma_q  } \right),
\end{equation*}
where $\Phi$ denotes the cdf of the standard normal and
$x_{n,p,q}$ is the value of $x$ solving
\begin{equation*}
\lambda \sigma_p^{-1} \exp\{ n(x-p)^2/(2\sigma^2_p) \}
= (1-\lambda) \sigma_q^{-1} \exp\{ n(x-q)^2/(2\sigma^2_q) \} ,
\end{equation*}
and hence our overall error rate
when classifying the $m$ OOS vertices will grow as $m \eta_{n+1,p,q}$.

As discussed previously,
the OOS extension allows us to avoid
the expense of computing the ASE of the full matrix
$$ \Atilde = \begin{bmatrix} A & B \\
                B^T & C \end{bmatrix}. $$
The LLS OOS extension is computationally inexpensive,
requiring only the computation of the
matrix-vector product $\SA^{-1/2} \UA^T \avec$,
with a time complexity $O( d^2 n )$ (assuming one does not precompute
the product $\SA^{-1/2} \UA^T$).
The eigenvalue computation required for embedding
$\Atilde$ is far more expensive
than the LLS OOS extension.
Nonetheless, if one were intent on reducing the OOS classification error
$\eta_{n+1,p,q}$, one might consider paying the computational
expense of embedding $\Atilde$ to obtain estimates
$\wtilde^{(1)}, \wtilde^{(2)}, \dots, \wtilde^{(m)}$
of the $m$ OOS vertices.
That is, we obtain estimates for the $m$ OOS vertices
by making them in-sample vertices, at the expense of solving
an eigenproblem on the $(m+n)$-by-$(m+n)$ adjacency matrix.
Of course, the entire motivation of our approach is that the in-sample
matrix $A$ may not be available.
Nonetheless, a comparison against this baseline,
in which all data is used to compute our embeddings, is instructive.

Theorem 1 in \citet{AthLyzMarPriSusTan2016} implies that the
$\wtilde^{(t)}$ estimates based on embedding the full matrix $\Atilde$ are
(approximately) marginally distributed as
$$ \wtilde^{(t)} \sim \lambda \calN(p,(n+m)^{-1}\sigma^2_p)
+ (1-\lambda)\calN(q,(n+m)^{-1}\sigma^2_q), $$
with classification error
\begin{equation*}
\eta_{n+m,p,q}
= \lambda \Phi\left( \frac{ p - x_{n+m,p,q} }{ \sigma_p } \right)
+ (1-\lambda)\Phi\left( \frac{ x_{n+m,p,q} - q }{ \sigma_q  } \right),
\end{equation*}
where $x_{n+m,p,q}$ is the value of $x$ solving
\begin{equation*}
\lambda \sigma_p^{-1} \exp\{ (m+n)(x-p)^2/(2\sigma^2_p) \}
=
(1-\lambda) \sigma_q^{-1} \exp\{ (m+n)(x-q)^2/(2\sigma^2_q) \} ,
\end{equation*}
and it can be checked that
$\eta_{n+m,q,p} < \eta_{n,q,p}$ when $m > 1$.
Thus, at the cost of computing the ASE of $\Atilde$,
we may obtain a better estimate.
How much does this additional computation improve
classification the OOS vertices?
Figure~\ref{fig:ratio} explores this question.

\begin{figure}
  \centering
  \includegraphics[width=0.6\columnwidth]{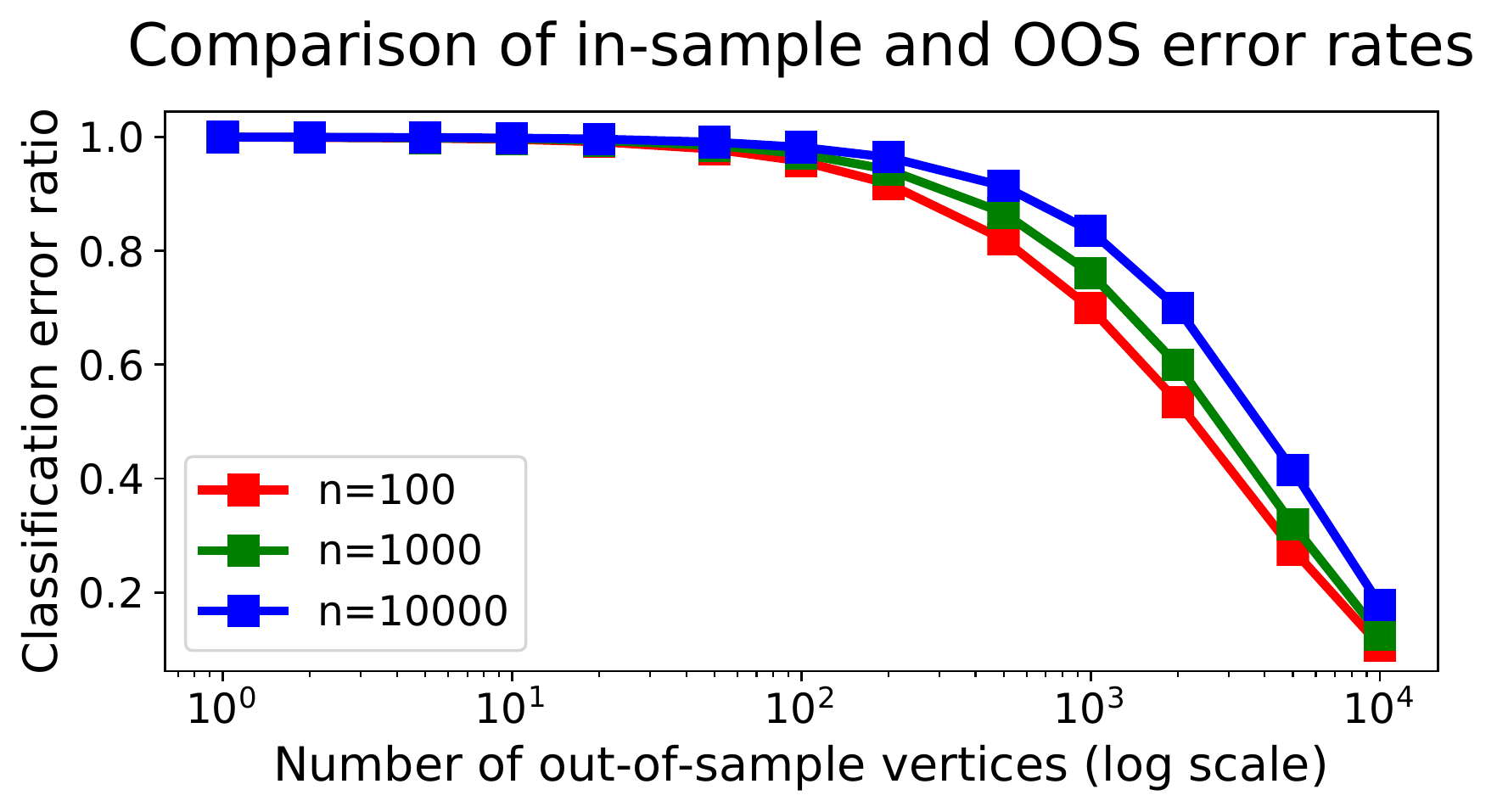}
  \vspace{-5mm}
  \caption{Ratio of the OOS classification error
        to the in-sample classification error as a function of the
        number of OOS vertices $m$, for $n=100$ vertices,
        $n=1000$ vertices and $n=10000$ vertices. We see that for $m \le 100$,
        the expensive in-sample embedding does not improve appreciably
        on the OOS classification error.
        However, when many hundreds or thousands of OOS vertices
        are available simultaneously (i.e., $m \ge 100$), we see that
        the in-sample embedding may improve upon the OOS
        estimate by a significant multiplicative factor.}
  \label{fig:ratio}
  \vspace{-5mm}
\end{figure}

Figure~\ref{fig:ratio} compares the error rates of the in-sample
and OOS estimates as a function of $m$ and $n$ in the model just described,
with $\lambda = 0.4, p=0.6$ and $q=0.61$.
The plot depicts the ratio of the (approximate) in-sample classification
error $\eta_{(n+m),p,q}$ to the (approximate) OOS classification
error $\eta_{(n+1),p,q}$, as a function of the number of OOS
vertices $m$, for differently-sized in-sample graphs,
$n=100, 1000,$ and $10000$.
We see that over several magnitudes of graph size,
the in-sample embedding does not improve appreciably over the
OOS embedding except when multiple hundreds of OOS
vertices are available.
When hundreds or thousands of OOS vertices are available
simultaneously, we see in the right-hand side of Figure~\ref{fig:ratio}
that the in-sample embedding classification error may improve upon
the OOS classification error by a large multiplicative factor.
Whether or not this improvement is worth the additional computational
expense will, depend upon the available resources and desired accuracy,
but this suggests that the additional
expense associated with performing a second ASE computation
is only worthwhile in the event that hundreds or thousands of OOS
vertices are available simultaneously.
This surfeit of OOS vertices is rather divorced from the typical setting
of OOS extension problems, where one typically wishes to embed at most a few
previously unseen observations.

\section{Discussion and Conclusion}
\label{sec:discussion}
We have presented theoretical results for out-of-sample extensions of
graph embeddings, the adjacency spectral embedding and the Laplacian
spectral embedding. In both cases, we have shown that under the random dot
product graph, a least squares-based OOS extension recovers the true latent
position at the same rate as the more expensive in-sample embedding.
Further, this linear least squares OOS extension obeys a CLT, whereby the
OOS embedding is normally distributed about the true latent position.
We have also presented results for an ASE OOS extension based on
a maximum-likelihood obective function showing that this embedding
recovers the true out-of-sample latent position at the same rate as the
in-sample embedding.
Experiments suggest that convergence to the predicted normal distribution
is fairly fast, being a good approximation with only a few hundred vertices.
Finally, we have briefly investigated how the approximation introduced by
these OOS extensions might be traded off against the computational expense
associated with computing the more expensive full graph embedding
by investigating how the approximate classification error predicted by our
CLT depends on the size of the size of the in-sample and the number of
out-of-sample vertices.

The results in this work suggest a number of interesting directions for future
work, a few of which we briefly enumerate here.
Firstly, though all of the OOS extensions presented in this paper match
the asymptotic estimation error rates of their respective in-sample embeddings,
our results say little about the constants associated with those rates
or about finite-sample behavior of those OOS extensions (aside from their
obvious restatements as finite-sample results alluded to briefly in
Section~\ref{subsec:notation}). A more thorough investigation of how these
different OOS extensions behave for different sizes of the in-sample
graph and for different latent position distributions $F$                       would be of particular interest to practitioners faced with choosing between
these different embeddings and OOS extensions as they apply to real data.
Our discussion surrounding Figure~\ref{fig:ratio}
makes an initial step in this direction, but only suggests rules of thumb
for when the speed/accuracy trade-off associated with out-of-sample extension
is likely to be favorable.

A related line of questioning concerns how one should, when possible,
select the in-sample vertices so as to yield optimal (as measured by, e.g.,
vertex classification or estimation accuracy of the latent positions)
out-of-sample embeddings.
Consider the setting where one has a graph $\Gtilde$ of size $\ntilde = n+m$    that is far too large to be embedded via ASE or LSE.
If $n$ is the largest number of vertices that can be feasibly embedded as a
full in-sample graph, it is natural to choose $n$ vertices from $\Gtilde$
to serve as the in-sample vertices, and embed the remaining $m$ vertices
via one of the out-of-sample extensions discussed in this paper.
In this setting, how should one choose these $n$ vertices from $\Gtilde$?
Problems of a similar nature have been considered elsewhere in the literature
under the heading of {\em anchor graphs} or choosing {\em anchor points}
\citep[see, e.g.,][]{LiuHeCha2010},
but we are not aware of any work in this area as it pertains to the ASE and LSE.
This also suggests the problem of how best to embed $m$ out-of-sample vertices
jointly, rather than applying an OOS extension to each of them in isolation,
particularly in the setting where we have access to the subgraph induced by
these $m$ out-of-sample vertices.
Of most import here is the question, also explored by Figure~\ref{fig:ratio}
of how large the out-of-sample size $m$ must be before one should prefer
the expense of the full-graph embedding, and whether an embedding that
makes use of this out-of-sample induced graph might bridge the gap between
these two extremes by providing an embedding which, while more expensive
than performing $m$ OOS extensions in isolation, is still far less
computationally intensive than embedding a graph of size $m+n$.
A more thorough exploration of this trade-off from both a theoretical and
empirical standpoint is the subject of on-going work.

%
%


%

\appendix

\section{Technical Results for the Random Dot Product Graph}
\label{apx:general}
Here we collect a number of basic results that will be useful in our
subsequent proofs of the main theorems.
Most of the results in this section are adapted from existing results in
\cite{LevAthTanLyzPri2017}, \cite{LyzSusTanAthPri2014} and \cite{TanPri2018}.
We refer the interested reader to \cite{RDPGsurvey} for a more thorough
overview of the RDPG and the statistical problems that arise in relation to it.

\begin{lemma}[\cite{LevAthTanLyzPri2017}, Observation 2] \label{lem:Pspecgrowth}
Let $(A,X) \sim \RDPG(F,n)$
for some $d$-dimensional inner product distribution $F$.
There exists constants $0 < C_1 < C_2$, depending only on $F$,
such that with probability $1$ is holds for all suitably large $n$ that
\begin{equation*} \begin{aligned}
  C_1 n &\le \lambda_d(P) \le \lambda_1(P) \le C_2 n \text{ and } \\
  C_1 \sqrt{n} &\le \lambda_d(X) \le  \lambda_1(X) \le C_2 \sqrt{n}.
\end{aligned} \end{equation*}
\end{lemma}

\begin{lemma}[\cite{LevAthTanLyzPri2017}, Lemma 3]
\label{lem:OMNI:Qdef}
With notation as above, let $V_1 \Lambda V_2^T$
be the SVD of $\UP^T \UA \in \bbR^{d \times d}$,
and define $Q =  V_1V_2^T$.
Then
\begin{equation*}
\| \UP^T \UA - Q \|_F = O( n^{-1} \log n ).
\end{equation*}
\end{lemma}

\begin{lemma}[\cite{TanPri2018}, Proposition B.2]
\label{lem:LSE:Qdef}
With notation as above, let $\Vtilde_1 \Lambdatilde \Vtilde_2^T$
be the SVD of $\Utilde^T \Ucheck \in \bbR^{d \times d}$
and define $\Qtilde = \Vtilde_1 \Vtilde_2^T$.
Then
\begin{equation*}
\| \Utilde^T \Ucheck - \Qtilde \|_F = O( n^{-1} ).
\end{equation*}
\end{lemma}

\begin{lemma}[\cite{lyzinski15_HSBM} Lemma 15; \cite{TanPri2018} Lemma B.3]
\label{lem:switching} \label{lem:approxcommute}
With notation as above,
\begin{equation*} \begin{aligned}
\left\| \Ucheck^T \Utilde \Stilde^{1/2} - \Scheck^{1/2} \Ucheck^T \Utilde
	\right\|
&= O(n^{-1}), \\
\left\| \Uhat^T U \SP^{-1/2} - \SA^{-1/2} \Uhat^T U
	\right\|_F
&= O( n^{-3/2} \log n ) \text{ and } \\
\left\| \Uhat^T U \SP^{1/2} - \SA^{1/2} \Uhat^T U
	\right\|_F
&=  O( n^{-1/2} \log n ).
\end{aligned} \end{equation*}
\end{lemma}

\begin{lemma} \label{lem:degreegrowth}
Let $F$ be a $d$-dimensional inner-product distribution and
let $(A,X) \sim \RDPG(F,n)$, and let $v$ be the out-of-sample vertex
with latent position $\wtrue \in \supp F$.
For $i \in [n]$, let $d_i = \sum_j A_{i,j}$ denote the degree of vertex $i$
and $t_i = \sum_j X_j^T X_i = \E[d_i \vert X]$ denote its expectation
conditional on the latent positions.
Analogously,
let $d_v = \sum_j a_j$ denote the degree of the out-of-sample vertex
and $t_v =  \sum_j X_j^T \wtrue$ denote its expectation.
Then
\begin{equation} \label{eq:degree:hoeff}
\max \left\{ |d_i - t_i| : i \in [n] \cup \{ v \} \right\}
	= O( \sqrt{n} \log^{1/2} n ).
\end{equation}
Similarly, letting $\mu = \E X_1 \in \bbR^d$ denote the mean
of latent position distribution $F$ and taking $X_v = \wtrue$,
\begin{equation} \label{eq:degree:hoeff2}
\max \left\{ |t_i - n \mu^T X_i | : i \in [n] \cup \{v\} \right\}
= O( \sqrt{n} \log^{1/2} n ).
\end{equation}
Further, uniformly over all $i \in [n]$,
\begin{align}
| d_i^{-1/2} - t_i^{-1/2} | &= O( n^{-1} \log^{1/2} n ),
	\label{eq:degree:recsqrt} \\ 
| d_i^{-1} - t_i^{-1} | &= O( n^{-3/2} \log^{1/2} n )
	\label{eq:degree:inverses}, \\
t_i &= \Theta( n ) \label{eq:degree:linear}
\end{align} 
\end{lemma}
\begin{proof}
Fix some $i \in [n] \cup \{v\}$. By definition, we have
\begin{equation*}
d_i - t_i = \begin{cases}
	 \sum_{j \neq i} (A_{i,j} - P_{i,j})
		&\mbox{ if } i \in [n] \\
	\sum_{j=1}^n a_j - X_j^T\wtrue
		&\mbox{ if } i =v,
	\end{cases}
\end{equation*}
a sum of independent random variables, each contained in $[-1,1]$
and thus Hoeffding's inequality immediately yields
\begin{equation*}
\Pr[ |d_i - t_i| \ge s ]
\le 2\exp\left\{ \frac{ -2s^2 }{ n } \right\}
\end{equation*}
for any $s \ge 0$.
Taking $s=C \sqrt{n} \log^{1/2} n$ for suitably large constant $C > 0$,
we have
\begin{equation*}
\Pr\left[ |d_i - t_i | \ge C \sqrt{n} \log^{1/2} n \right] \le C' n^{-3}.
\end{equation*}
Taking a union bound over all $i \in [n] \cup \{v\}$, we conclude that
\begin{equation*}
\Pr\left[ \exists i : |d_i - t_i| \ge C\sqrt{n} \log^{1/2} n \right]
	\le C n^{-2},
\end{equation*}
and an application of the Borel-Cantelli Lemma \citep{Billingsley1995}
yields Equation~\eqref{eq:degree:hoeff}.

Again by definition, we have for any $i \in [n] \cup \{v\}$,
\begin{equation*}
t_i - n X_i^T \mu 
= X_i^T( X_i - \mu) + \sum_{j \neq i} X_i^T (X_j - \mu).
\end{equation*}
The first term on the right-hand side is $O(1)$, since $X_i \sim F$
and $\mu$ is constant.
The sum over $j \neq i$ is, conditioned on $X_i$,
a sum of independent unbiased random variables,
which are bounded by the assumption that
$0 \le x^T y \le 1$ whenever $x,y \in \supp F$.
Thus, an application of Hoeffding's inequality similar to that above
yields that, conditioned on $X_i=x_i \in \supp F$,
\begin{equation*}
\sum_{j \neq i} x_i^T (X_j - \mu) \le C \sqrt{n} \log^{1/2} n,
\end{equation*}
where the constant $C$ can be chosen independent of $x_i$
again because $\supp F$ is bounded.
Unconditioning establishes Equation~\eqref{eq:degree:hoeff2},
since $X_i^T(X_i - \mu) = O(1)$.
\eqref{eq:degree:linear} follows,
since $t_i = n X_i^T \mu + O(\sqrt{n} \log^{1/2} n)$.
Writing
\begin{equation*}
\left| \frac{1}{\sqrt{d_i}} - \frac{1}{\sqrt{t_i}} \right|
= \frac{ \left| d_i - t_i \right| }
	{ \sqrt{d_i} \sqrt{t_i} (\sqrt{d_i} + \sqrt{t_i}) }
\end{equation*}
and applying Equations~\eqref{eq:degree:hoeff2} and~\eqref{eq:degree:linear}
implies~\eqref{eq:degree:recsqrt}.
A similar argument establishes~\eqref{eq:degree:inverses}.
\end{proof}

\begin{lemma} \label{lem:lapspecgrowth}
Let $P = X X^T \in \bbR^{n \times n}$ with rows of $X$ drawn i.i.d.\ from $F$
as above. Then
\begin{equation} 
\lambda_d( \calL(P) ) = \Theta(1),
\lambda_1( \calL(P) ) = \Theta(1)
\text{ and }
\lambda_d( \Xtilde ) = \Theta(1).
\end{equation}
\end{lemma}
\begin{proof}
By definition,
$ \calL(P) = T^{-1/2} \UP \SP \UP^T T^{-1/2}$, so that
\begin{equation*}
\lambda_d( \calL(P) ) \le \lambda_1(\calL(P))
\le \| \calL(P) \| \le \| T^{-1/2} \| \| \SP \| \| T^{-1/2} \|
        \le \frac{ \| \SP \| }{ \min_i t_i } \le C,
\end{equation*}
where the last inequality follows from
Lemmas~\ref{lem:Pspecgrowth} and~\ref{lem:degreegrowth}.

To show the corresponding lower-bound, we adapt an argument from
the proof of Theorem 8.1.17 in \cite{GolVan2012} to write
\begin{equation*}
\lambda^2_1(T^{1/2})\lambda_d( \calL(P) )
        \ge \lambda_d( P ) \ge Cn,
\end{equation*}
where the second lower-bound follows from Lemma~\ref{lem:Pspecgrowth}.
We conclude that
\begin{equation*}
\lambda_1(\calL(P))
	\ge \lambda_d(\calL(P)) \ge \frac{ Cn }{ \lambda_1(T) } \ge C,
\end{equation*}
since $\lambda_1^2(T^{1/2}) = \lambda_1(T) \le n$.

By definition of $\Xtilde$,
$\lambda_k( \Xtilde ) = \sqrt{ \lambda_k( \calL(P) )}$ for all $k \in [d]$,
whence $\lambda_d( \Xtilde ) = \Theta(1)$
\end{proof}

\begin{lemma} \label{lem:Scheck}
Let $F$ be a $d$-dimensional inner-product distribution with
mean $\mu$ and suppose that there exists a constant $\eta > 0$
such that $\eta \le x^Ty \le 1-\eta$ for all $x,y \in \supp F$.
Define $\Deltatilde = \E X_1 X_1^T / X_1^T \mu $ where $X_1 \sim F$
and let $\Scheck = \Xcheck^T \Xcheck$ and $\Stilde = \Xtilde^T \Xtilde$.
Then
\begin{equation*}
\| \Qtilde \Scheck \Qtilde^T - \Stilde\|
	= O\left( \frac{1}{n} \right)
\text{ and }
\| \Stilde - \Deltatilde \|
	= O\left( \frac{ \log^{1/2} n }{ \sqrt{n} } \right).
\end{equation*}
\end{lemma}
\begin{proof}
Adding and subtracting appropriate quantities
and applying a triangle inequality followed by submultiplicativity,
we have
\begin{equation*} \begin{aligned}
\| \Qtilde \Scheck \Qtilde^T - \Stilde \|
&= \left\| \Qtilde \Scheck^{1/2}
	\left( \Scheck^{1/2} \Qtilde^T - \Qtilde^T \Stilde^{1/2} \right)
	+ \left( \Qtilde  \Scheck^{1/2} - \Stilde^{1/2} \Qtilde \right)
		\Qtilde^T \Stilde^{1/2} \right\| \\
&\le \left( \| \Qtilde \Scheck^{1/2} \| + \| \Qtilde^T \Stilde^{1/2} \| \right)
	\| \Qtilde  \Scheck^{1/2} - \Stilde^{1/2} \Qtilde \|,
\end{aligned} \end{equation*}
where we have used the unitary invariance of the spectral norm
to write
\begin{equation*}
\|\Qtilde  \Scheck^{1/2} - \Stilde^{1/2} \Qtilde \|
= \| \Scheck^{1/2} \Qtilde^T - \Qtilde^T \Stilde^{1/2} \|.
\end{equation*}
An additional application of the unitary invariance of the spectral norm yields
\begin{equation} \label{eq:QScheckQT-Stilde:1}
\| \Qtilde \Scheck \Qtilde^T - \Stilde \|
\le \left( \|  \Scheck^{1/2} \| + \| \Stilde^{1/2} \| \right)
	\| \Qtilde  \Scheck^{1/2} - \Stilde^{1/2} \Qtilde \|.
\end{equation}
By definition of $\Scheck$ and $\Stilde$ as the top $d$ eigenvalues of
$\calL(A)$ and $\calL(P)$, respectively, we have
\begin{equation*}
\| \Scheck - \Stilde \| \le \| \calL(A) - \calL(P) \|.
\end{equation*}
Theorem 3.1 in \cite{Oliveira2010} implies that
\begin{equation*}
\| \calL(A) - \calL(P) \|
\le C \left( \min_i t_i \right)^{-1/2} \log^{1/2} n,
\end{equation*}
and Lemma~\ref{lem:degreegrowth} implies that $\min_i t_i = \Omega(n)$,
so that
\begin{equation*}
\| \calL(A) - \calL(P) \| = O( n^{-1/2} \log^{1/2} n ),
\end{equation*}
and it follows that
\begin{equation*}
\| \Scheck^{1/2} \| \le = \| \Stilde^{1/2} \|\left(1 + o(1) \right).
\end{equation*}
Lemma~\ref{lem:lapspecgrowth} bounds the growth of $\| \Stilde \|$ as $O(1)$,
whence $\| \Scheck^{1/2} \| = O(1)$
and we conclude that
\begin{equation} \label{eq:QScheckQT-Stilde:specgrowth}
\|  \Scheck^{1/2} \| + \| \Stilde^{1/2} \| = O(1).
\end{equation}
Once again adding and subtracting appropriate quantities,
applying the triangle inequality folowed by submultiplicativity,
\begin{equation*} \begin{aligned}
\| \Qtilde  \Scheck^{1/2} - \Stilde^{1/2} \Qtilde \|
&\le \| (\Qtilde - \Utilde^T \Ucheck) \Scheck^{1/2} \|
	+ \| \Utilde^T \Ucheck  \Scheck^{1/2}
		- \Stilde^{1/2} \Utilde^T \Ucheck \|
	+ \| \Stilde^{1/2}( \Utilde^T \Ucheck - \Qtilde) \| \\
&\le \left( \| \Scheck^{1/2} \| + \| \Stilde^{1/2} \| \right)
	\| \Qtilde - \Utilde^T \Ucheck \|
	+ \| \Utilde^T \Ucheck  \Scheck^{1/2}
                - \Stilde^{1/2} \Utilde^T \Ucheck \|.
\end{aligned} \end{equation*}
Equation~\eqref{eq:QScheckQT-Stilde:specgrowth} and Lemma~\ref{lem:LSE:Qdef}
imply that
\begin{equation*}
 \left( \| \Scheck^{1/2} \| + \| \Stilde^{1/2} \| \right)
        \| \Qtilde - \Utilde^T \Ucheck \|
= O( n^{-1} ),
\end{equation*}
and Lemma~\ref{lem:approxcommute} implies that
\begin{equation*}
  \| \Utilde^T \Ucheck  \Scheck^{1/2}
                - \Stilde^{1/2} \Utilde^T \Ucheck \|
= O( n^{-1} ).
\end{equation*}
Combining the above two displays, we conclude that
\begin{equation*}
\| \Qtilde  \Scheck^{1/2} - \Stilde^{1/2} \Qtilde \|
= O( n^{-1} ).
\end{equation*}
Applying this and Equation~\eqref{eq:QScheckQT-Stilde:specgrowth} to
Equation~\eqref{eq:QScheckQT-Stilde:1}, we conclude that
$\| \Qtilde \Scheck \Qtilde^T - \Stilde \| = O(n^{-1}).$

To bound $\| \Stilde - \Deltatilde \|$, note that
\begin{equation*}
\Stilde = \sum_{i=1}^n \Xtilde_i \Xtilde_i^T
= \sum_{i=1}^n \frac{ X_i X_i^T }{ t_i }.
\end{equation*}
Applying Lemma~\ref{lem:degreegrowth},
$\max_i |t_i^{-1} - (nX_i^T \mu)^{-1}| = O( n^{-3/2} \log^{1/2} n)$,
and thus
\begin{equation*}
\Stilde = \frac{1}{n} \sum_{i=1}^n \frac{ X_i X_i^T }{ X_i^T \mu }
		+ O( n^{-1/2} \log^{1/2} n ).
\end{equation*}
Hoeffding's inequality applied to the sum implies
$\Stilde = \Deltatilde + O( n^{-1/2} \log^{1/2} n ),$
completing the proof.
\end{proof}

\begin{lemma} \label{lem:fullrank}
Suppose that $F$ is a $d$-dimensional inner-product distribution
with $X_1 \sim F$ for which
$\Delta = \E_F X_1 X_1^T \in \bbR^{d \times d}$ is full rank.
If $X_1,X_2,\dots,X_n \iid F$,
then with probability $1$ there exists an $n_0$ such that
$X \in \bbR^{n \times d}$ has full column rank for all $n \ge n_0$.
\end{lemma}
\begin{proof}
Since the top $d$ eigenvalues of $P = X X^T$ are precisely the $d$
eigenvalues of $X^TX$, Lemma~\ref{lem:Pspecgrowth} implies that
$\lambda_d(X^T X) = \Omega(n)$.
It follows that $X^T X \in \bbR^{d \times d}$ is
invertible for all suitably large $n$.
\end{proof}

We now give a proof of the bound in Equation~\ref{eq:2toinfty:lse}
in Lemma~\ref{lem:2toinfty}.
\begin{proof}[Proof of Lemma~\ref{lem:2toinfty}]
Let $\zeta_i \in \bbR^d$ denote the (transposed) $i$-th row
of $\Xcheck - \Xtilde \Qtilde$, where $\Qtilde = \Vtilde_1 \Vtilde_2^T$
as in Lemma~\ref{lem:LSE:Qdef} above.
Define the event
\begin{equation*}
E_n = \left\{
	\forall i \in [n] : \| \zeta_i \| \le \frac{ C \log^{1/2} n }{ n }
	\right\}
\end{equation*}
where $C > 0$ is a constant that we will specify below,
depending on the latent position distribution $F$ but not on $n$.
It will suffice for us to show that $E_n$ holds eventually.

Fix some $i \in [n]$ and define
$\mu = \E X_1 \in \bbR^d$ to be the mean of $F$.
Following the argument in Appendix B.1 of \cite{TanPri2018}, we have
\begin{equation} \label{eq:LSE:B6B7}
\zeta_i = \frac{ (\Xtilde^T\Xtilde)^{-1} }{ n }
	\frac{ \sqrt{n} }{ \sqrt{ t_i } }
	\sum_{j \neq i} \frac{ A_{i,j} - P_{i,j} }{ \sqrt{n} }
		\left( \frac{X_j}{X_j^T\mu}
		- \frac{ \Deltatilde X_i}{2 X_i^T\mu} \right)
	+ o(n^{-1}).
\end{equation}
For all $j \in [n] \setminus \{ i \}$, define
\begin{equation*}
Z^{(i)}_j = \frac{ A_{i,j} - P_{i,j} }{ \sqrt{n} }
                \left( \frac{X_j}{X_j^T\mu}
                - \frac{ \Deltatilde X_i}{2 X_i^T\mu} \right).
\end{equation*}
Condition on $X_i = x_i \in \supp F$ and fix $k \in [d]$.
Thanks to the assumption that $0 < \eta \le x^Ty \le 1-\eta$
whenever $x,y \in \supp F$, we have that
$ \sum_{j \neq i} Z^{(i)}_{j,k} $ is a sum of independent $0$-mean bounded
random variables.
Hoeffding's inequality implies that 
\begin{equation} \label{eq:condhoeff1}
\Pr\left[ \left| \sum_{j \neq i} Z^{(i)}_{j,k} \right| \ge s
	\Big\vert X_i=x_i\right]
\le 2\exp\left\{ \frac{ -s^2 }{ 2n^{-1} \sum_{j\neq i} V_{j,k}^2 } \right\},
\end{equation}
where
\begin{equation*}
V_{j,k}
= \frac{ X_{j,k} }{ X_j^T\mu }
	- \frac{ (\Deltatilde x_i)_k }{ 2x_i^T \mu }.
\end{equation*}
Using the fact that $X_j, x_i \in \supp F$
and that $X_j$ is independent of $X_i$ for $j \neq i$, we have
\begin{equation} \label{eq:Vltconstant}
\E\left[ V_{j,k}^2 \Big \vert X_i=x_i \right] \le
C\left( \frac{ \| \Deltatilde x_i \|^2 }{ 4( x_i^T \mu)^2 }
+ \E\left[ \frac{ |X_{j,k}|^2 }{ (X_j^T \mu)^2 } \right] \right)^2
\le C_F,
\end{equation}
where $C_F$ depends on $F$ but can be chosen independent of $k$ and $x_i$.
By the law of large numbers (conditional on $X_i=x_i$),
\begin{equation*}
n^{-1} \sum_{j\neq i} V_{j,k}^2
	\rightarrow \E[ V_{j,k}^2 \mid X_i = x_i ] \text{ almost surely.}
\end{equation*}
Thus, applying Equation~\eqref{eq:Vltconstant} and
integrating out by $X_i$,
\begin{equation*}
n^{-1} \sum_{j \neq i} V_{j,k}^2 \le 2 C_F \text{ eventually.}
\end{equation*}
Integrating~\eqref{eq:condhoeff1} with respect to $F$ and using the above fact,
we conclude that
\begin{equation*}
\Pr\left[ \left| \sum_{j \neq i} Z^{(i)}_{j,k} \right|
		\ge C \log^{1/2} n \right]
\le 2n^{-3},
\end{equation*}
for suitably large constant $C>0$.
A union bound over all $k \in [d]$ yields
\begin{equation*}
\Pr\left[ \left\| \sum_{j \neq i} Z^{(i)}_j \right\|
		\ge C \log^{1/2} n \right]
	\le 2d n^{-3},
\end{equation*}
and a further union bound over $i \in [n]$ implies
\begin{equation} \label{eq:2toinf:jkbound}
\max_{i \in [n]} \|\sum_{j \neq i} Z^{(i)}_j\| = O( \log^{1/2} n ).
\end{equation}
Applying this result to Equation~\eqref{eq:LSE:B6B7} and using
the fact that
$ \Xtilde^T \Xtilde \rightarrow \Deltatilde$ almost surely
and $\sqrt{n} t_i^{-1/2} = O(1)$
by Lemmas~\ref{lem:degreegrowth} and~\ref{lem:Scheck} respectively,
we have
\begin{equation} \label{eq:2toinf:payoff}
\max_{i \in [n]} \| \zeta_i \| \le 
\frac{ 1 }{ n \| \Xtilde^T \Xtilde \| }
	\frac{ \sqrt{n} }{ \min_{i \in [n]} \sqrt{ t_i } }
\max_{i \in [n]} \left\| \sum_{j \neq i} Z^{(i)}_j \right\|
= O\left( \frac{ \log^{1/2} n }{ n } \right),
\end{equation}
which completes the proof.
\end{proof}

The following spectral norm bound will be useful at several points in
our proofs.
\begin{theorem}\emph{\citep[Matrix Bernstein inequality,][]{Tropp2015} }
\label{thm:asymbern}
Let $\{ Z_k \}$ be a finite collection
of random matrices in $\bbR^{d_1 \times d_2}$
with $\E Z_k = 0$ and $\| Z_k \| \le R$ for all $k$, then
\begin{equation*}
\Pr\left[ \left\| \sum_k Z_k \right\| \ge t \right]
\le (d_1 + d_2)\exp\left\{ \frac{ -t^2 }{ \nu^2 + Rt/3 } \right\},
\end{equation*}
where
\begin{equation*}
\nu^2 = \max\left\{ \left\| \sum_k \E Z_k Z_k^T \right\|,
                        \left\| \sum_k \E Z_k^T Z_k \right\| \right\}.
\end{equation*}
\end{theorem}

\section{Proof of ASE LS-OOS Concentration Inequality}
\label{apx:ase:lsconc}
To prove Theorem~\ref{thm:ase:lsrate}, we must relate the least squares
solution $\whatls$ of \eqref{eq:def:ase:llshat}
to the true latent position $\wtrue$.
We will proceed in two steps.
First, we will show that
$\whatls$ is close to a least squares solution
based on the true latent positions $\{ X_i \}_{i=1}^n$ rather than on the
estimates $\{ \Xhat_i \}_{i=1}^n$. That is, letting $\wls$ be the solution
\begin{equation} \label{eq:def:wls}
  \wls = \arg \min_{w \in \bbR^d} \| X w - \avec \|_F,
\end{equation}
we will bound the error introduced by the ASE,
$\| Q \whatls - \wls \|$,
taking $Q \in \bbR^{d \times d}$ to be as defined in Lemma~\ref{lem:2toinfty}.
This is the content of Lemma~\ref{lem:ls:hat}.
Second, we will show that $\wls$
is close to the true latent position $\wtrue$.
That is, we will control the error introduced by the $n$ random
in-sample latent positions and the network $A$.
This is done in Lemma~\ref{lem:ls:true}.
The triangle inequality will then yield Theorem~\ref{thm:ase:lsrate}.

We first establish a bound on $\| Q \whatls - \wls \|$,
where $\whatls$ is the solution to Equation~\eqref{eq:def:ase:llshat},
$\wls$ is as defined by Equation~\eqref{eq:def:wls},
and $Q \in \bbR^{d \times d}$ is the orthogonal matrix guaranteed to exist
by Lemma~\ref{lem:2toinfty}.
Our bound will depend upon
a basic result for solutions of perturbed linear systems,
which we adapt from \cite{GolVan2012}.
In essence, we wish to compare
\begin{equation*}
  \whatls = \arg \min_{w \in \bbR^d} \| \Xhat w - \avec \|_F
\end{equation*}
against
\begin{equation*}
  \wls = \arg \min_{w \in \bbR^d} \| X w - \avec \|_F.
\end{equation*}
Recall that for a matrix $B \in \bbR^{n \times d}$ of full column rank,
we define the condition number
\begin{equation*}
  \kappa_2(B) = \frac{ \sigma_1(B) }{ \sigma_d(B) }.
\end{equation*}
\begin{theorem}[\cite{GolVan2012}, Theorem 5.3.1]
\label{thm:gvl}
Suppose that the quantities $\wls, \whatls \in \bbR^d$
and $\rls, \rhatls \in \bbR^n$ satisfy
\begin{equation*} \begin{aligned}
  \| X \wls - \avec \| = \min_w \| X w - \avec \|,
        &\enspace\enspace
  \rls = \avec - X \wls, \\
\| \Xhat \whatls - \avec \| = \min_w \| \Xhat w - \avec \|,
        &\enspace\enspace
  \rhatls = \avec - \Xhat \whatls,
\end{aligned} \end{equation*}
and that
\begin{equation} \label{eq:gvl:errorgrowth}
 \| \Xhat - XQ \| < \lambda_d( X ).
\end{equation}
Assume $\avec, \rls$ and $\wls$ are all non-zero and define
$\thetals \in (0, \pi/2)$ by
$\sin \thetals = \| \rls \|/\|\avec\|$.
Letting
\begin{equation*}
\nuls = \frac{ \| X \wls \| }{ \sigma_d(XQ) \| Q^T \wls \| },
\end{equation*}
we have
\begin{equation} \label{eq:GVLbound}
\begin{aligned}
&\frac{ \| \whatls - Q^T \wls \| }{ \| Q^T \wls \| } \\
&~~~
\le \frac{ \|\Xhat - XQ \| }{\| XQ \|} \left( \frac{ \nuls }{ \cos \thetals }
        + (1 + \nuls \tan \thetals) \kappa_2(XQ) \right)
        + O\left( \frac{ \|\Xhat - XQ \|^2 }{ \| XQ \|^2 } \right).
\end{aligned} \end{equation}
\end{theorem}

To apply Theorem~\ref{thm:gvl}, we will first need to show
that the condition in Equation~\eqref{eq:gvl:errorgrowth}
and the non-zero conditions on $\avec, \rls$ and $\wls$
all hold with high probability.
This is done in Lemma~\ref{lem:gvlerrcond}.
We will then show, using Lemma~\ref{lem:gvlerrcond}
and Lemma~\ref{lem:gvlresidual}, that
the right-hand side of Equation~\eqref{eq:GVLbound} is
$O(n^{-1/2} \log n)$.

\begin{lemma} \label{lem:gvlerrcond}
With notation as above,
$\avec$, $\rls$ and $\wls$ are all nonzero eventually.
and \eqref{eq:gvl:errorgrowth} holds eventually,
That is, with probability $1$, there exists a sequence of orthogonal matrices
$Q \in \bbR^{d \times d}$ such that
\begin{equation} \label{eq:gvlcond:specbound}
  \| \Xhat - XQ \| < \lambda_d(X) \eventually
\end{equation}
Further,
\begin{equation} \label{eq:gvlcond:epsilon}
  \frac{ \| \Xhat - XQ \| }{ \| XQ \| } 
  = O\left( \frac{ \log n }{ \sqrt{n} } \right).
\end{equation}
\end{lemma}
\begin{proof}
That $\avec$ is non-zero eventually is an immediate consequence of the model,
and it follows that $\wls$ is non-zero eventually, from which
it follows that the residual $\rls = \avec - X\wls$ is also nonzero eventually.
Let $Q \in \bbR^{d \times d}$ be the orthogonal matrix guaranteed by
Lemma~\ref{lem:2toinfty}. We begin by observing that
\begin{equation*}
\| \Xhat - XQ \|^2 \le \| \Xhat - XQ \|_F^2
= \sum_{i=1}^n \| \Xhat_i - QX_i \|^2
= O( \log^2 n )
\end{equation*}
where the last equality follows from Lemma~\ref{lem:2toinfty}.
By the definition of the RDPG, we can write
$XQ = \UP \SP^{1/2}Q$, from which
$\sigma_d(XQ) = \sigma_d^{1/2}(P) = \Omega( \sqrt{n})$
by Lemma~\ref{lem:Pspecgrowth}.
This establishes~\eqref{eq:gvlcond:specbound} immediately,
and \eqref{eq:gvlcond:epsilon} follows from the above display.
\end{proof}

\begin{lemma} \label{lem:gvlresidual}
With notation as in Theorem~\ref{thm:gvl},
there exists a constant $0 \le \gamma < 1$,
not depending on $n$, such that with probability $1$,
$\cos \thetals \ge \gamma$ for all suitably large $n$.
That is, there exists a constant $0 < \gamma'$ such that
\begin{equation*} 
  \frac{ \| XQ \wls - \avec \| }{ \| \avec \| }
  \le \gamma' \text{ eventually.}
\end{equation*}
\end{lemma}
\begin{proof}
By definition of $\wls$, we have
$\| X Q \wls - \avec \| \le \| X \wtrue - \avec \|$.
For ease of notation, set $\rvec = \avec - X \wtrue$.
It will suffice for us to show that for some
constant $\rho > 0$, we have
\begin{equation} \label{eq:sintheta:term2goal}
(1-\rho)\| \avec \|^2 - \| \rvec \|^2 \ge 0 \eventually,
\end{equation}
since then, after rearranging terms, $\sin^2 \thetals \le 1-\rho$.
To show~\eqref{eq:sintheta:term2goal}, note that
\begin{equation*} \begin{aligned}
(1-\rho)\| \avec \|^2 - \| \rvec \|^2
&= 2\sum_{i=1}^n a_i X_i^T\wtrue - \sum_{i=1}^n (X_i^T\wtrue)^2
        -\rho \sum_{i=1}^n a_i^2 \\
&\ge \E\left[ (1-\rho)\| \avec \|^2 - \| \rvec \|^2 \right]
        + C \sqrt{n} \log^{1/2} n \eventually,
\end{aligned} \end{equation*}
where the inequality follows from an application of Hoeffding's inequality
to show that the sum concentrates about its expectation.
We will have established~\eqref{eq:sintheta:term2goal} if we can show that
$\E \left[ (1-\rho)\| \avec \|^2 - \| \rvec \|^2 \right]$ grows faster
than $C \sqrt{n} \log^{1/2} n$.
To establish this, let $i \in [n]$ be arbitrary and write
\begin{equation*} \begin{aligned}
\E\left[ (1-\rho) a_i^2 - r_i^2 \right]
&= \E\left[ (1-\rho) a_i^2 - (a_i - X_i^T\wtrue)^2 \right]
	= -\rho \E a_i^2 +2\E a_i X_i^T\wtrue - \E(X_i^T\wtrue)^2 \\
&= -\rho \E a_i^2 - \E(a_i-X_i^T\wtrue)X_i^T\wtrue + \E a_i X_i^T\wtrue
	= \E a_i X_i^T\wtrue - \rho \E a_i^2 .
\end{aligned} \end{equation*}
By our boundedness assumption on $\supp F$,
$\E a_i X_i^T\wtrue = \E(X_i^T \wtrue)^2$ is bounded away from zero uniformly
in $i \in [n]$, and thus choosing $\rho > 0$ suitably small ensures that
there exists a small constant $\eta' > 0$ such that
$\E\left[ (1-\rho) a_i^2 - r_i^2 \right] \ge \eta' > 0.$
Summing over $n$,
\begin{equation*}
\E\left[ (1-\rho)\| \avec \|^2 - \| \rvec \|^2 \right]
= \sum_{i=1}^n \E \left[ (1-\rho)a_i^2 r_i^2 \right]
\ge n \eta' = \Omega(n),
\end{equation*}
which proves
the bound in~\eqref{eq:sintheta:term2goal}, completing the proof.
\end{proof}

\begin{lemma} \label{lem:ls:hat}
With notation as in Theorem~\ref{thm:gvl},
there exists a sequence of orthogonal matrices $Q \in \bbR^{d \times d}$
such that
\begin{equation*}
\| Q \whatls - \wls \| = O( n^{-1/2} \log n ).
\end{equation*}
\end{lemma}
\begin{proof}
This is a direct result of Theorem~\ref{thm:gvl} and the preceding Lemmas,
once we establish bounds on $\kappa_2(XQ)$ and
\begin{equation*}
 \nuls = \frac{ \| XQ \wls \| }{ \lambda_d(XQ) \| \wls \| }.
\end{equation*}
By Lemma~\ref{lem:Pspecgrowth}, we have
$C_1\sqrt{n} \ge \lambda_1(XQ) \ge \lambda_d(XQ) \ge C_2\sqrt{n}$,
and it follows immediately that $\kappa_2(XQ) \le C$ eventually.
Since
$\| XQ \wls \|/\|\wls\| \le \| XQ \| \le \sqrt{n}$,
we also have $\nuls \le C$ eventually.

By Lemma~\ref{lem:gvlerrcond}, we are assured that Theorem~\ref{thm:gvl}
applies eventually.
Lemmas~\ref{lem:gvlerrcond} and \ref{lem:gvlresidual}
ensure that the each of $(\cos \thetals)^{-1}$ and $\tan \thetals$
are bounded by constants eventually.
Thus, using Lemma~\ref{lem:gvlerrcond} to bound $\| \Xhat - XQ\|/\|XQ\|$,
it follows that the right-hand side of
Equation~\ref{eq:GVLbound} is $O( n^{-1/2} \log n)$ and the result follows.
\end{proof}

We now turn to showing that $\wls$
is close to the true latent position $\wtrue$.
A combination of this result with Lemma~\ref{lem:ls:hat}
will yield Theorem~\ref{thm:ase:lsrate}.
\begin{lemma} \label{lem:ls:true}
Let notation be as above and let $\wtrue \in \supp F$ be the
(fixed) latent position of the out-of-sample vertex.
Then for all but finitely many $n$,
\begin{equation*}
  \| \wls - \wtrue \| \le \frac{ C \log n }{ \sqrt{n} }.
\end{equation*}
\end{lemma}
\begin{proof}
Define $\rvec = \avec - X \wtrue$.
As noted previously, by definition of $\wls$, we have
\begin{equation*}
  \| X \wls - \avec \|^2 \le \| X \wtrue - \avec \|^2 = \| \rvec \|^2,
\end{equation*}
whence plugging in $\avec = X\wtrue + \rvec$ yields
$ \| X \wls - X\wtrue - \rvec \|^2 \le \| \rvec \|^2 $.
Thus,
\begin{equation} \label{eq:subtractrvec}
 \| X \wls - X \wtrue \|^2 \le 2\rvec^T X(\wls - \wtrue).
\end{equation}
By Lemma~\ref{lem:fullrank},
$X$ has full column rank eventually, and thus also
$\| X(\wls - \wtrue) \| \ge \sigma_d(X) \| \wls - \wtrue \|$ eventually.
Combining this fact with \eqref{eq:subtractrvec} and using the fact
that $\sigma_d^2(X) = \sigma_d(P)$, we have
\begin{equation*}
\| \wls - \wtrue \|^2
\le \frac{ \| X(\wls - \wtrue) \|^2 }{ \sigma_d^2(X) }
\le \frac{ 2\rvec^T X(\wls - \wtrue) }{ \sigma_d(P) }.
\end{equation*}
Applying the Cauchy-Schwartz inequality and dividing by $\|\wls - \wtrue\|$,
\begin{equation*}
 \| \wls - \wtrue \| \le \frac{ 2 \| X^T \rvec \| }{ \sigma_d(P) }.
\end{equation*}
Thus, it remains for us to show that $\| X^T \rvec \|$
grows as at most $O( \sqrt{n} \log^2 n )$, from which
Lemma~\ref{lem:Pspecgrowth} will yield our desired growth rate.
Expanding, we have
\begin{equation} \label{eq:etasum}
\| X^T \rvec \|_2^2
   = \sum_{k=1}^d \left( \sum_{i=1}^n (a_i - X_i^T\wtrue) X_{i,k} \right)^2.
\end{equation}
Fixing some $k \in [d]$, Hoeffding's inequality
implies that with probability at least $1-O(n^{-2})$,
$| \sum_{i=1}^n (a_i - X_i^T\wtrue) X_{i,k} | \le 2 \sqrt{n} \log n$.
Since $d$ is assumed to be constant in $n$,
a union bound over all $k \in [d]$ implies
$\| X^T \rvec \|_2^2 \le 4dn \log^2 n$
with probability at least $1 - O(n^{-2})$.
Applying the Borel-Cantelli Theorem
and taking square roots completes the proof.
\end{proof}

\section{Proof of ASE ML-OOS Concentration Inequality}
\label{apx:ase:mlconc}
To prove Theorem~\ref{thm:ase:mlrate}, we will
apply a standard argument
from convex optimization and use the properties of the set
$\calThat_\epsilon$ to show that
\begin{equation*}
\| Q \whatml - \wtrue \|
\le \frac{ \| \nabla \ellhat( Q^T \wtrue ) \| }{ C n },
\end{equation*}
where $Q \in \bbR^{d \times d}$ is the orthogonal matrix guaranteed
by Lemma~\ref{lem:2toinfty}. This is proven in Lemma~\ref{lem:cvxopt}.
We then show in Lemma~\ref{lem:nabla} that
\begin{equation*}
\| \nabla \ellhat( Q^T \wtrue ) \| = O( \sqrt{n} \log n ),
\end{equation*}
which establishes Theorem~\ref{thm:ase:mlrate} by the triangle inequality.

Recall the log-likelihood functions
\begin{equation} \label{eq:def:loglik} \begin{aligned}
\ell(w) &= \sum_{i=1}^n a_i \log X_i^T w
                        +(1-a_i) \log (1-X_i^T w) \\
\ellhat(w) &= \sum_{i=1}^n a_i \log \Xhat_i^T w
                        +(1-a_i) \log (1-\Xhat_i^T w)
\end{aligned} \end{equation}
and observe that both are convex in their arguments.
\begin{lemma} \label{lem:cvxopt}
With notation as above, under the assumptions of Theorem~\ref{thm:ase:mlrate},
it holds almost surely that for all suitably large $n$,
there exists an orthogonal matrix $Q \in \bbR^{d \times d}$ satisfying
\begin{equation*}
\| Q \whatml - \wtrue \|
\le \frac{ \| \nabla \ellhat( Q^T \wtrue ) \| }{ C n }.
\end{equation*}
\end{lemma}
\begin{proof}
By a standard argument, we have
\begin{equation*} \begin{aligned}
&\left( \nabla \ellhat(Q^T \wtrue) \right)^T (Q^T \wtrue - \whatml ) \\
&~~~~~~= \left( \nabla \ellhat(\whatml) \right)^T (Q^T \wtrue - \whatml) \\
&~~~~~~~~~~~~+ \int_0^1 (Q^T \wtrue - \whatml)^T
        \nabla^2 \ellhat\left( Q^T \wtrue + t(Q^T \wtrue - \whatml) \right)
        (Q^T\wtrue-\whatml) dt \\
&~~~~~~\ge \| \wtrue - Q\whatml \|^2
        \min_{w \in \calThat_\epsilon}
        \lambdamin \left( \nabla^2 \ellhat(w) \right).
\end{aligned} \end{equation*}
Rearranging and applying the Cauchy-Schwarz inequality implies
\begin{equation*}
\| \wtrue - Q\whatml \|
        \le \frac{ \| \nabla \ellhat( Q^T \wtrue ) \| }
		{ |\lambdamin \left( \nabla^2 \ellhat(w) \right)| }.
\end{equation*}
The constraint that $w \in \calThat_\epsilon$ implies that
for suitably large $n$,
\begin{equation*}
\min_{w \in \calThat_\epsilon}
        \lambdamin\left( \nabla^2 \ellhat(w) \right)
  \ge C n, \end{equation*}
with $C > 0$ depending on $\epsilon$ and $F$ but not on $n$,
where we have used Lemma~\ref{lem:2toinfty} to ensure that
$\{ \Xhat_i \}_{i=1}^n$ are uniformly close to $\supp F$.
We conclude that eventually,
\begin{equation*}
\| \wtrue - Q\whatml \|
        \le \frac{ \| \nabla \ellhat( Q^T \wtrue ) \| }{ Cn },
\end{equation*}
completing the proof.
\end{proof}

\begin{lemma} \label{lem:nabla}
With notation as above,
under the assumptions of Theorem~\ref{thm:ase:mlrate},
\begin{equation*}
 \| \nabla \ellhat( Q^T \wtrue ) \| = O( \sqrt{n} \log n ) .
\end{equation*}
\end{lemma}
\begin{proof}
By the triangle inequality,
\begin{equation} \label{eq:gradtriangle}
\| \nabla \ellhat( Q^T \wtrue ) \|
\le \| \nabla \ell( \wtrue ) \| +
 \| \nabla \ellhat( Q^T \wtrue ) - \nabla \ell( \wtrue ) \|.
\end{equation}
We will show that
both terms on the right hand side of \eqref{eq:gradtriangle}
are $O( \sqrt{n} \log^{1/2} n )$.

Fix $k \in [d]$.
By our boundedness assumption on $\supp F$ and the fact that
$\wtrue,X_1,X_2,\dots,X_n \in \supp F$,
\begin{equation*}
\left(\nabla \ell( \wtrue ) \right)_k = 
\sum_{i=1}^n \left( \frac{ a_i }{ X_i^T \wtrue }
        - \frac{ 1-a_i }{ 1 - X_i^T \wtrue } \right) X_{i,k}
= \sum_{i=1}^n \frac{ (a_i - X_i^T \wtrue) X_{i,k} }
                        { X_i^T \wtrue (1-X_i^T \wtrue) }
\end{equation*}
is a sum of bounded zero-mean random variables.
Applying Hoeffding's inequality,
\begin{equation*}
\Pr\left[ \left| \left(\nabla \ell( \wtrue ) \right)_k \right| \ge t \right] 
        \le 2\exp\left\{ \frac{ -2t^2 }{ C n } \right\}
\end{equation*}
for some constant $C > 0$ depending on $F$ but not $n$.
Choosing $t = \sqrt{ Cn } \log^{1/2} n$,
we have $(\nabla \ell( \wtrue ) )_k \ge \sqrt{ Cn } \log^{1/2} n$
with probability at most $O(n^{-2})$.
A union bound over all $k \in [d]$, implies that
with probability at least $1 - Cdn^{-2}$,
\begin{equation*}
\sum_{k=1}^d \left( \nabla \ell( \wtrue ) \right)^2_k
  \le d C n \log n,
\end{equation*}
and the Borel-Cantelli Lemma implies
$\| \nabla \ell( \wtrue ) \| = O(\sqrt{n} \log^{1/2} n)$
after taking square roots.

Turning to the second term on the right hand side of
\eqref{eq:gradtriangle}, fixing $k \in [d]$, we have
\begin{equation*}
\left( \nabla \ellhat( Q^T \wtrue ) - \nabla \ell( \wtrue ) \right)_k
= \sum_{i=1}^n \frac{ (a_i - \Xhat_i^T Q^T \wtrue) \Xhat_{i,k} }
                        { \Xhat_i^T Q^T \wtrue (1-\Xhat_i^T Q^T \wtrue) }
- \sum_{i=1}^n \frac{ (a_i - X_i^T \wtrue) X_{i,k} }
                        { X_i^T \wtrue (1-X_i^T \wtrue) }.
\end{equation*}
Taking expectation conditional on $A$ and $X$, 
the second sum has expectation $0$, and
\begin{equation*}
\E\left[
	\left( \nabla \ellhat( Q^T \wtrue ) - \nabla \ell( \wtrue ) \right)_k
	\Big \vert A,X \right]
= \sum_{i=1}^n \frac{ \left( (Q \Xhat_i) - X_i \right)^T \wtrue}
		{ (Q \Xhat_i)^T\wtrue(1-(Q\Xhat_i)^T \wtrue) } \Xhat_{i,k}.
\end{equation*}
By Lemma~\ref{lem:2toinfty} and our boundedness assumptions on $\supp F$,
the denominators of this sum are uniformly bounded away from zero
over almost all sequences of $(A,X)$.
Lemma~\ref{lem:2toinfty} also bounds the numerators in this sum
uniformly by $O(n^{-1/2} \log n)$,
and it follows that
\begin{equation} \label{eq:nabla:condexprate}
\E\left[
        \left( \nabla \ellhat( Q^T \wtrue ) - \nabla \ell( \wtrue ) \right)_k
        \Big \vert A,X \right]
= O(\sqrt{n} \log n).
\end{equation}
Our proof will be complete if we can show that
\begin{equation*}
\left(  \nabla \ellhat( Q^T \wtrue ) - \nabla \ell( \wtrue ) \right)_k
- \E\left[
        \left( \nabla \ellhat( Q^T \wtrue ) - \nabla \ell( \wtrue ) \right)_k
        \Big \vert A,X \right]
\end{equation*}
concentrates at the same rate.
Toward this end, for ease of notation, for each $i \in [n]$ define
$p_i = X_i^T \wtrue$ and $\phat_i = \Xhat_i^T \wtrue$.
Then
\begin{equation*} \begin{aligned}
&\left(  \nabla \ellhat( Q^T \wtrue ) - \nabla \ell( \wtrue ) \right)_k
- \E\left[
        \left( \nabla \ellhat( Q^T \wtrue ) - \nabla \ell( \wtrue ) \right)_k
        \Big \vert A,X \right] \\
&~~~~~~~~~~~~~~~= \sum_{i=1}^n\left[
	\frac{ (a_i - \phat_i) \Xhat_{i,k} }{ \phat_i (1-\phat_i) }
  - \frac{ (a_i - p_i) X_{i,k} }{ p_i(1-p_i) }
  - \frac{ (p_i - \phat_i) \Xhat_{i,k} }{ \phat_i(1-\phat_i) }
	\right] \\
&~~~~~~~~~~~~~~~= \sum_{i=1}^n (a_i - p_i)
		\left( \frac{ \Xhat_{i,k} }{\phat_i(1-\phat_i) }
			- \frac{ X_{i,k} }{ p_i(1-p_i) } \right).
\end{aligned} \end{equation*}
Conditional on $(A,X)$, this is a sum of $n$ independent zero-mean random
vectors, with the $i$-th summand bounded by
\begin{equation*} 
\left| (a_i-p_i) \left( \frac{ \Xhat_{i,k} }{\phat_i(1-\phat_i) }
                - \frac{ X_{i,k} }{ p_i(1-p_i) } \right) \right|
\le \left| \frac{ \Xhat_{i,k} }{\phat_i(1-\phat_i) }
                - \frac{ X_{i,k} }{ p_i(1-p_i) } \right|
\end{equation*}
since $|a_i-p_i| \le 1$.
Let $M_i$ denote this bound for each $i \in [n]$.
Let $s > 0$ be a value which we will specify below,
and let $B_n$ denote the event that
\begin{equation*}
\left| \left(  \nabla \ellhat( Q^T \wtrue ) - \nabla \ell( \wtrue ) \right)_k
- \E\left[
        \left( \nabla \ellhat( Q^T \wtrue ) - \nabla \ell( \wtrue ) \right)_k
        \Big \vert A,X \right] \right| > s.
\end{equation*}
Hoeffding's inequality conditional on $A,X$ implies that
\begin{equation*}
\Pr\left[ B_n \mid A,X \right]
\le 2\exp\left\{ \frac{ -s^2 }{ 2\sum_{i=1}^n M_i^2 } \right\}.
\end{equation*}
By definition of $M_i$, we have
\begin{equation*} \begin{aligned}
M_i
&= \left| \frac{ \Xhat_{i,k} }{\phat_i(1-\phat_i) }
                - \frac{ X_{i,k} }{ p_i(1-p_i) } \right| \\
&\le \frac{ | \Xhat_{i,k} - X_{i,k} | }{ p_i(1-p_i) } 
 + \left| \frac{1}{p_i(1-p_i)} - \frac{1}{\phat_i(1-\phat_i)} \right||X_{i,k}|\\
&\le \frac{ O(n^{-1/2} \log n) }{ p_i(1-p_i) }
 + \frac{|p_i-\phat_i|(1-p_i) + p_i|p_i-\phat_i| }
	{ p_i(1-p_i)\phat_i(1-\phat_i)},
\end{aligned} \end{equation*}
where the first inequality follows from the triangle inequality,
and the second inequality follows from Lemma~\ref{lem:2toinfty}
and the fact that $\|X_i\| \le 1$ by definition of $F$ being an
inner product distribution.
Lemma~\ref{lem:2toinfty} implies that $|\phat_i-p_i| = O(n^{-1/2} \log n)$,
since $\| \wtrue \| \le 1$.
Our boundedness assumptions on the support of $F$, along
with yet another application of Lemma~\ref{lem:2toinfty}, imply that
both denominators are bounded away from $0$ eventually.
Thus, uniformly over all $i \in [n]$, $M_i = O( n^{-1/2} \log n)$,
so that $\sum_{i=1}^n M_i^2 = O(\log^2 n)$,
and integrating with respect to $(A,X)$ implies that
\begin{equation*}
\Pr\left[ B_n \mid A,X \right]
\le 2\exp\left\{ \frac{-Cs^2}{ \log^2 n } \right\}.
\end{equation*}
Taking $s = C\log^{3/2} n$ for suitably large constant $C$
and applying the Borel-Cantelli Lemma
ensures that $B_n$ occurs eventually,
and we have that
\begin{equation*}
\left(  \nabla \ellhat( Q^T \wtrue ) - \nabla \ell( \wtrue ) \right)_k
- \E\left[
        \left( \nabla \ellhat( Q^T \wtrue ) - \nabla \ell( \wtrue ) \right)_k
        \Big \vert A,X \right]
= O( \log^{3/2} n ).
\end{equation*}
Combining this with Equation~\eqref{eq:nabla:condexprate}, we conclude that
\begin{equation*}
\left(  \nabla \ellhat( Q^T \wtrue ) - \nabla \ell( \wtrue ) \right)_k
= O( \sqrt{n} \log n).
\end{equation*}
Since $d$ is assumed constant,
this rate holds uniformly over all $k \in [d]$,
and we conclude that
\begin{equation*}
\|  \nabla \ellhat( Q^T \wtrue ) - \nabla \ell( \wtrue ) \|
= O( \sqrt{n} \log n ),
\end{equation*}
completing the proof.
\end{proof}

\section{Proof of LSE LS-OOS Concentration Inequality}
\label{apx:lseconc}
Here we provide a proof of Theorem~\ref{thm:lse:lsrate}.
The argument proceeds similarly to the proof of
Theorem~\ref{thm:ase:lsrate} in Appendix~\ref{apx:ase:lsconc} above.
Recall that $\wcheckls \in \bbR^d$ denotes the least-squares OOS extension,
given by the solution to
\begin{equation*}
        \min_{w \in \bbR^d} \sum_{i=1}^n
        \left( \frac{a_i}{ d_v^{1/2} d_i^{1/2} }
                - \Xcheck_i^T w \right)^2,
\end{equation*}
where $\Xcheck_i \in \bbR^d$ is the LSE estimate of the Laplacian spectral
embedding of the true latent position of the $i$-th vertex
and $d_i$ denotes the degree of vertex $i$ for $i \in [n] \cup \{ v \}$.
We define $\wtildels \in \bbR^d$ to be the least-squares OOS extension
if we had access to the true latent positions.
That is, $\wtildels$ is the solution to the least-squares problem
\begin{equation*}
\min_{w \in \bbR^d} \sum_{i=1}^n
        \left( \frac{a_i}{ d_v^{1/2} d_i^{1/2} }
                - \Xtilde_i^T w \right)^2.
\end{equation*}
Letting $\Qtilde \in \bbR^{d \times d}$ denote the orthogonal matrix guaranteed
by Lemma~\ref{lem:2toinfty},
our proof of Theorem~\ref{thm:lse:lsrate} will proceed by showing that
both $\| \wtildels - \wtilde \|$ and $\| \wcheckls - \Qtilde^T \wtildels \|$
are $O( n^{-1} \log^{1/2} n )$,
after which the triangle inequality will yield our desired result.

\begin{lemma} \label{lem:lse:wtildels2wtilde}
With notation as above,
\begin{equation*}
 \| \wtildels - \wtilde \| = O( n^{-1} \log^{1/2} n ).
\end{equation*}
\end{lemma}
\begin{proof}
Recall that $D \in \bbR^{n \times n}$ is the diagonal matrix of in-sample
vertex degrees and $d_v = \sum_{i=1}^n a_i$ denotes the degree of the
out-of-sample vertex $v$.
Define $\bvec = d_v^{-1/2} D^{-1/2} \avec$, and let
$\zvec = \bvec - \Xtilde \wtilde$.
By definition of $\wtildels$ as a least squares solution, we have
\begin{equation*}
\| \Xtilde \wtildels - \bvec \| \le \| \zvec \|.
\end{equation*}
Substituting $\bvec = \zvec + \Xtilde \wtilde$,
expanding the squares of both sizes and rearranging,
\begin{equation} \label{eq:wtildels:innerprod}
\| \Xtilde(\wtildels - \wtilde) \|^2
\le 2\zvec^T \Xtilde (\wtildels - \wtilde)
\end{equation}
By Lemma~\ref{lem:fullrank}, $\Xtilde$ is full rank eventually, and therefore
\begin{equation*}
\| \Xtilde(\wtildels - \wtilde) \|
        \ge \sigma_d(\Xtilde) \| \wtildels - \wtilde \| \eventually.
\end{equation*}
Combining this with~\eqref{eq:wtildels:innerprod} and making use of
the Cauchy-Schwarz inequality,
\begin{equation*}
\| \wtildels - \wtilde \|
\le \frac{ 2 \| \Xtilde^T \zvec \| }{ \sigma_d^2(\Xtilde ) } \eventually.
\end{equation*}
Lemma~\ref{lem:lapspecgrowth} implies that
$\sigma_d^2(\Xtilde) = \Theta(1)$, so our proof will be complete if we
can bound the growth of $\| \Xtilde^T \zvec \|$.
We have
\begin{equation*}
\| \Xtilde^T \zvec \|^2
= \sum_{k=1}^d \left( \sum_{i=1}^n z_i \Xtilde_{i,k} \right)^2
= \sum_{k=1}^d Y_k^2,
\end{equation*}
where $Y_k = \sum_{i=1}^n z_i \Xtilde_{i,k}$.
Fixing some $k \in [d]$,
\begin{equation*}
Y_k
= \sum_{i=1}^n \left( \frac{ X_i^T \wtrue }
                { \sqrt{ t_i } \sqrt{ n \mu^T \wtrue } }
                -
                \frac{ a_i }{ \sqrt{ d_i } \sqrt{ d_v } } \right)
                \frac{ X_{i,k} }{ \sqrt{t_i} }.
\end{equation*}
Adding and subtracting appropriate quantities,
\begin{equation} \label{eq:wtildels:Ybreak}
Y_k
= \sum_{i=1}^n \frac{ (X_i^T \wtrue - a_i) }
                        { t_i \sqrt{n\mu^T\wtrue} } X_{i,k}
+ \sum_{i=1}^n \frac{ a_i X_{i,k} }{ \sqrt{t_i} }
                \left( \frac{ 1}{ \sqrt{t_i} \sqrt{n\mu^T\wtrue} }
                        - \frac{1}{\sqrt{ d_i }\sqrt{ d_v } } \right).
\end{equation}
Conditional on $X$, the first term is a sum of independent mean-$0$ random
variables, with
\begin{equation*}
\frac{ (X_i^T \wtrue - a_i)X_{i,k} }{ t_i \sqrt{n\mu^T\wtrue} }
\in \left[ \frac{-1}{ t_i \sqrt{n\mu^T\wtrue} },
	  \frac{1}{ t_i \sqrt{n\mu^T\wtrue} } \right] \text{ almost surely }
\end{equation*}
for each $i \in [n]$.
Let $G_n$ denote the event that
\begin{equation*}
  \left| \sum_{i=1}^n \frac{ (X_i^T \wtrue - a_i) }
                        { t_i \sqrt{n\mu^T\wtrue} } X_{i,k} \right|
                > s,
\end{equation*}
where $s = s_n > 0$ will be specified below.
Conditional Hoeffding's inequality yields
\begin{equation*}
  \Pr\left[ B_n \mid X \right]
\le 2\exp\left\{  \frac{-n \mu^T \wtrue s^2 }{ \sum_{i=1}^n t_i^{-2} } \right\}
\end{equation*}
Let $B_n$ denote the event that
$\min_i t_i \ge Cn$ for some suitably-chosen constant $C>0$.
Lemma~\ref{lem:degreegrowth} ensures that $\Pr[B_n^c] = O(n^{-2})$, and
integrating with respect to $X \in \bbR^{n \times d}$ yields
\begin{equation*}
\Pr[G_n] \le \Pr[ G_n \mid B_n ] + \Pr[ B_n^c ]
\le 2\exp\left\{ -C n^2 \mu^T \wtrue s^2 \right\} + O(n^{-2}).
\end{equation*}
Taking $s = Cn^{-1} \log^{1/2} n$ for $C>0$ suitably large
ensures that both terms on the right-hand side are $O(n^{-2})$,
and we have
\begin{equation} \label{eq:wtildels:Yterm1}
\left| \sum_{i=1}^n \frac{ (X_i^T \wtrue - a_i) X_{i,k} }
                        { \sqrt{t_i} \sqrt{ t_v } \sqrt{n\mu^T\wtrue} }
        \right|
= O( n^{-1} \log^{1/2} n).
\end{equation}
Lemma~\ref{lem:degreegrowth} similarly bounds
the second sum in~\eqref{eq:wtildels:Ybreak}:
\begin{equation} \label{eq:wtildels:Yterm2:pre}
\sum_{i=1}^n \frac{ a_i X_{i,k} }{ \sqrt{t_i} }
                \left( \frac{ 1}{ \sqrt{t_i} \sqrt{n\mu^T\wtrue} }
                        - \frac{1}{\sqrt{ d_i }\sqrt{ d_v } } \right)
\le \frac{C}{ \sqrt{n} }
        \sum_{i=1}^n \left( \frac{ 1}{ \sqrt{t_i} \sqrt{n\mu^T\wtrue} }
                        - \frac{1}{\sqrt{ d_i }\sqrt{ d_v } } \right)
                        a_i X_{i,k}.
\end{equation}
Adding and subtracting appropriate quantities, the sum becomes
\begin{equation*} \begin{aligned}
\sum_{i=1}^n & \left( \frac{ 1}{ \sqrt{t_i} \sqrt{n\mu^T\wtrue} }
                        - \frac{1}{\sqrt{ d_i }\sqrt{ d_v } } \right)
                        a_i X_{i,k} \\
&~~~~~~= \sum_{i=1}^n \frac{ a_i X_{i,k} }{ \sqrt{ t_i } }
        \left( \frac{1}{\sqrt{n \mu^T\wtrue}}-\frac{1}{\sqrt{d_v}} \right)
+ \sum_{i=1}^n \frac{ a_i X_{i,k} }{ \sqrt{ d_v } }
        \left( \frac{1}{\sqrt{t_i}} - \frac{1}{\sqrt{d_i}} \right),
\end{aligned} \end{equation*}
and several applications of Lemma~\ref{lem:degreegrowth} yields that
\begin{equation*}
\sum_{i=1}^n \left( \frac{ 1}{ \sqrt{t_i} \sqrt{n\mu^T\wtrue} }
                        - \frac{1}{\sqrt{ d_i }\sqrt{ d_v } } \right)
                        a_i X_{i,k}
= O( n^{-1/2} \log^{1/2} n ),
\end{equation*}
whence, applying this to Equation~\eqref{eq:wtildels:Yterm2:pre},
we have
\begin{equation*}
\sum_{i=1}^n \frac{ a_i X_{i,k} }{ \sqrt{t_i} }
                \left( \frac{ 1}{ \sqrt{t_i} \sqrt{n\mu^T\wtrue} }
                        - \frac{1}{\sqrt{ d_i }\sqrt{ d_v } } \right)
= O( n^{-1} \log^{1/2} n ).
\end{equation*}
Applying this and~\eqref{eq:wtildels:Yterm1} to the right-hand side
of~\eqref{eq:wtildels:Ybreak},
$|Y_k| = O( n^{-1} \log^{1/2} n )$
and a union bound over $k \in [d]$ completes the proof.
\end{proof}

\begin{lemma} \label{lem:lse:wtildels2wcheck}
With notation as above, there exists a sequence of orthogonal matrices
$\Qtilde \in \bbR^{d \times d}$ such that
\begin{equation*}
\| \Qtilde \wcheckls - \wtildels \| = O(n^{-1} \log^{1/2} n).
\end{equation*}
\end{lemma}
\begin{proof}
Recall from above our definition $\bvec = d_v^{-1/2} D^{-1/2} \avec$, where
$d_v$ is the degree of the out-of-sample vertex and $D \in \bbR^{n \times n}$
is the diagonal matrix of in-sample vertex degrees,
and note that $\wcheckls = (\Xcheck^T\Xcheck)^{-1} \Xcheck^T \bvec$.
Our main tool, as in Section~\ref{apx:ase:lsconc},
is Theorem 5.3.1 from \cite{GolVan2012}, quoted above as Theorem~\ref{thm:gvl}.
Applying that theorem, we have that so long as
$\bvec, \bvec - \Xtilde \wtildels$ and $\wtildels$ are all non-zero,
\begin{equation*}
\frac{ \| \wcheckls - \Qtilde^T \wtildels \| }{ \| \Qtilde^T \wtildels \| }
\le
\frac{ \| \Xcheck - \Xtilde \Qtilde \| }{ \| \Xtilde \Qtilde \| }
\left( \frac{ \nuls }{\cos \thetals} +
        (1 + \nuls \tan \thetals) \kappa_2(\Xtilde \Qtilde) \right)
+ C\frac{ \| \Xcheck - \Xtilde \Qtilde \|^2 }{ \| \Xtilde \Qtilde \|^2 },
\end{equation*}
where $\thetals \in (0,\pi/2)$ with
\begin{equation*}
\sin \thetals = \frac{ \| \rtildels \| }{ \| \bvec \|},
~~~\text{and}~~~
\nuls = \frac{ \| \Xtilde \wtildels \| }
        { \sigma_d( \Xtilde \Qtilde ) \| \Qtilde^T \wtildels \| }.
\end{equation*}
In order to apply Theorem~\ref{thm:gvl}, we must first show that eventually
\begin{enumerate}
\item $\| \Xcheck - \Xtilde \Qtilde \| < \sigma_d( \Xtilde )$
and \label{item:cond1}
\item the quantities
	$\bvec, \bvec-\Xtilde \wtildels,$ and $\wtildels$
	are all non-zero. \label{item:cond2}
\end{enumerate}
The first condition holds eventually by Lemma~\ref{lem:lapspecgrowth}
and the fact that, using the relations between the spectral,
Frobenius and $(2,\infty)$-norms,
\begin{equation} \label{eq:checktilde:spectral}
\| \Xcheck - \Xtilde \Qtilde \|^2
\le \| \Xcheck - \Xtilde \Qtilde \|_F^2
\le n \| \Xcheck_i - \Qtilde \Xtilde_i \|_{\tti}^2
\le \frac{ C \log n }{ n }, 
\end{equation}
where the last inequality holds eventually by Lemma~\ref{lem:2toinfty}.
As in the proof of Lemma~\ref{lem:gvlerrcond}, it is
immediate from the model that condition~\ref{item:cond2} holds eventually.

Equation~\eqref{eq:checktilde:spectral},
along with another application of Lemma~\ref{lem:lapspecgrowth}
to control $\lambda_d( \calL(P) )$ implies that 
\begin{equation} \label{eq:checktilde:specbound}
\frac{ \| \Xcheck - \Xtilde \Qtilde \| }{ \| \Xtilde \Qtilde \| }
\le \frac{ C \log^{1/2} n }{ \sqrt{n \sigma_d(\calL(P)) } }
\le \frac{ C \log^{1/2} n }{ \sqrt{n} } \eventually
\end{equation}
Thus, applying Theorem~\ref{thm:gvl}, we have
\begin{equation} \label{eq:main1:inter}
\| \wcheckls - \Qtilde^T \wtildels \|
\le \frac{ C \| \wtildels \| \log^{1/2} n }{ \sqrt{n} }
        \left( \frac{ \nuls }{ \cos \thetals }
        + (1 + \nuls \tan \thetals ) \kappa_2( \Xtilde \Qtilde ) \right)
  + \frac{ C \log^2 n }{ n^2 }.
\end{equation}
Lemma~\ref{lem:lapspecgrowth} bounds the condition number
$\kappa_2( \Xtilde \Qtilde ) = \kappa_2( \Xtilde ) \le C$,
whence
\begin{equation*}
\nuls
= \frac{ \| \Xtilde \wtildels \| }{ \sigma_d(\Xtilde) \| \Qtilde^T \wtildels \| }
= \frac{ \| \Xtilde \wtildels \| }{ \sigma_d(\Xtilde) \| \wtildels \| } \\
\le \frac{ \| \Xtilde \| }{ \sigma_d(\Xtilde) } = \kappa_2( \Xtilde ) \le C
\eventually.
\end{equation*}
By the triangle inequality,
the definition of $\wtilde$ and
using Lemma~\ref{lem:lse:wtildels2wtilde}
to bound $\| \wtildels - \wtilde \|$,
\begin{equation*}
 \| \wtildels \|
= \left\| \frac{\wtrue}{\sqrt{n \mu^T\wtrue}} \right\|
		+ O( n^{-1} \log^{1/2} n )
 = O( n^{-1/2} ) + O(n^{-1} \log^{1/2} n),
\end{equation*}
whence Equation~\eqref{eq:main1:inter} becomes
\begin{equation*}
\| \Qtilde \wcheckls - \wtildels \|
\le \frac{ C \log^{1/2} n }{ n }
	\left( 1 + \frac{1 + \sin \thetals }{ \cos \thetals} \right)
	+ \frac{ C \log^2 n }{ n^2 } \eventually.
\end{equation*}
Thus, to complete the proof,
it will suffice to bound $\cos \thetals$ away from $0$.
To do this, we will show by an argument similar to
that in Lemma~\ref{lem:gvlresidual} that
there exists a constant $\rho \in (0,1)$ such that
$\sin \thetals \le 1-\rho$ eventually.

Toward this end,
define $\bbar = t_v^{-1/2} T^{-1/2} \avec$,
where we remind the reader that
$t_v = \sum_{i=1}^n X_i^T \wtrue$ is the expected degree of the out-of-sample
vertex conditioned on the latent positions,
and $T \in \bbR^{n \times n}$ is the diagonal matrix of in-sample vertex
expected degrees, i.e., $T_{i,i} = \sum_{j=1}^n X_j^T X_i$.
Letting $\Xtilde^\dagger = (X^T T^{-1} X)^{-1} X^T T^{-1/2}$
denote the pseudoinverse of $\Xtilde$,
(with the inverse existing eventually by Lemma~\ref{lem:fullrank}),
we have
\begin{equation} \label{eq:sintheta:firstbound}
\begin{aligned}
\sin \thetals &= \frac{ \| \bvec - \Xtilde \wtildels \| }{ \| \bvec \| }
        = \frac{ \| (I - \Xtilde \Xtilde^\dagger) \bvec \| }{ \| \bvec \| }
= \frac{ \| \bbar \| }{ \| \bvec \| }
        \frac{ \| (I - \Xtilde \Xtilde^\dagger) \bvec \| }{ \| \bbar \| } \\
&\le \frac{ \| \bbar \| }{ \| \bvec \| }
        \left(
        \frac{ \| I - \Xtilde \Xtilde^\dagger \| \| \bvec - \bbar \| }
                { \| \bbar \| }
        + \frac{ \| (I - \Xtilde \Xtilde^\dagger) \bbar \| }{ \| \bbar \| }
        \right),
\end{aligned}
\end{equation}
where the inequality follows from
the triangle inequality and submultiplicativity.
By definition of $\bvec$ and $\bbar$, we have
\begin{equation*}
\frac{ \| \bvec - \bbar \| }{ \| \bbar \| }
= \frac{ \left\| (d_v^{-1/2} D^{-1/2} - t_v^{-1/2} T^{-1/2}) \avec \right\| }
	{ \| t_v^{-1/2} T^{-1/2} \avec \| }
\le \frac{ \| d_v^{-1/2} D^{-1/2} - t_v^{-1/2} T^{-1/2} \| }
	{ t_v^{-1/2}/\max_i \sqrt{t_i} },
\end{equation*}
where we have used submultiplicativity to upper bound the numerator,
$\| T^{-1/2} \avec \| \ge \| \avec \|/\max_i \sqrt{t_i}$
to lower-bound the denominator, and cancelled the resulting factor
of $\| \avec \|$.
Cancelling factors of $t_v^{-1/2}$, we have
\begin{equation*}
\frac{ \| \bvec - \bbar \| }{ \| \bbar \| }
\le \| t_v^{1/2} d_v^{-1/2} D^{-1/2} - T^{-1/2} \| \max_i \sqrt{ t_i }.
\end{equation*}
Lemma~\ref{lem:degreegrowth} implies $\max_i \sqrt{t_i} = O(\sqrt{n})$,
and a second application of Lemma~\ref{lem:degreegrowth} implies that
$\| t_v^{1/2} d_v^{-1/2} D^{-1/2} - T^{-1/2} \| = O( n^{-1} \log^{1/2} n)$,
from which
\begin{equation} \label{eq:bbratio1}
\frac{ \| \bvec - \bbar \| }{ \| \bbar \| } = O( n^{-1/2} \log^{1/2} n),
\end{equation}
and it follows from the triangle inequality that
\begin{equation} \label{eq:bbratio2}
\frac{ \| \bbar \| }{ \| \bvec \| }
\le \frac{\| \bvec \| + \| \bbar - \bvec \|}{\|\bvec\|}
= 1 + O(n^{-1/2} \log^{1/2} n) = O(1).
\end{equation}
Applying Equations~\eqref{eq:bbratio1} and~\eqref{eq:bbratio2}
to Equation~\eqref{eq:sintheta:firstbound}
and using the bound $\| I - \Xtilde \Xtilde^\dagger \| \le 1$,
\begin{equation} \label{eq:sintheta:twoterms}
\sin \thetals
\le O\left( \frac{ \log^{1/2} n }{\sqrt{n} } \right)
        +
  \frac{ C\|  (I - \Xtilde \Xtilde^\dagger) \bbar \| }{ \| \bbar \| }.
\end{equation}

Letting $\projcompXtilde = (I - \Xtilde \Xtilde^\dagger)$ denote the
orthogonal projection onto the orthogonal complement of the column space of
$\Xtilde = T^{-1/2} X$, we have, canceling factors of $t_v^{-1/2}$ in the
numerator and denominator,
\begin{equation*} \begin{aligned}
\frac{ \|  (I - \Xtilde \Xtilde^\dagger) \bbar \| }{ \| \bbar \| }
&= \frac{ \|  (I - \Xtilde \Xtilde^\dagger) T^{-1/2} \avec \| }
	{ \| T^{-1/2} \avec \| }
= \frac{ \| \projcompXtilde T^{-1/2} \avec \| }
	{ \| T^{-1/2} \avec \| }
= \frac{ \| \projcompXtilde T^{-1/2}(\avec - X \wtrue) \| }
	{ \| T^{-1/2} \avec \| },
\end{aligned} \end{equation*}
where we have used the fact that $\projcompXtilde T^{-1/2} X \wtrue = 0$,
since $T^{-1/2} X \wtrue = \Xtilde \wtrue$ is in the column space of $\Xtilde$.
Thus, defining $\rvec = \avec - X \wtrue$, we have
\begin{equation*} \begin{aligned}
\frac{ \|  (I - \Xtilde \Xtilde^\dagger) \bbar \| }{ \| \bbar \| }
&= \frac{ \| \projcompXtilde T^{-1/2} \rvec \| }
        { \| T^{-1/2} \avec \| }
\le \frac{ \| T^{-1/2} \| \| \rvec \| }{ \| \avec \| / \max_i \sqrt{t_i} }
\le C \frac{ \| \rvec \| }{ \| \avec \| },
\end{aligned} \end{equation*}
where the last inequality follows from the fact that the expected degrees
$\{ t_i \}_{i=1}^n$ are all of the same order by Lemma~\ref{lem:degreegrowth}.
The same argument as that given in the proof of Lemma~\ref{lem:gvlresidual}
lets us bound $\| \rvec \|/\|\avec\|$
by a constant $\rho > 0$ smaller than $1/(2C)$.
Applying this
to~\eqref{eq:sintheta:twoterms}, we obtain
\begin{equation*}
\sin \thetals \le 1-\rho + O( n^{-1/2} \log^{1/2} n )
\end{equation*}
It follows that
\begin{equation*}
\sin \thetals \le 1 - \frac{\rho}{2} \eventually,
\end{equation*}
i.e., $\sin \thetals$ is bounded away from $1$, completing the proof.
\end{proof}

\section{Proof of ASE linear least squares out-of-sample CLT}
\label{apx:ase:lsclt}

In this section, we prove Theorem~\ref{thm:ase:lsclt},
which shows that taking $\{ Q_n \}_{n=1}^\infty$ to be the sequence of
orthogonal $d$-by-$d$ matrices guaranteed to exist by Lemma~\ref{lem:2toinfty},
the quantity $\sqrt{n}(\whatls - Q^T \wtrue)$
is asymptotically multivariate normal.
We begin by recalling that
\begin{equation*}
  \whatls = (\Xhat^T \Xhat)^{-1} \Xhat^T \avec = \SA^{-1/2} \UA^T \avec.
\end{equation*}
Our proof will consist of writing
$\sqrt{n}(\whatls - Q^T \wtrue)$ as a sum of two random vectors,
\begin{equation*}
\sqrt{n}(\whatls - Q^T \wtrue)
  = \sqrt{n} \gvec + \sqrt{n} \hvec,
\end{equation*}
and showing that
$\sqrt{n}\gvec$ converges in law to a normal,
while $\sqrt{n}\hvec$ converges in probability to $0$.
The multivariate version of Slutsky's Theorem
will then yield the desired result.
We begin by showing that 
$\gvec = \sqrt{n} \SP^{-1/2} \UP^T(\avec - X \wtrue)$ will suffice.
We remind the reader that $\Delta = \E X_1 X_1^T \in \bbR^{d \times d}$
is the second moment matrix of the latent position distribution $F$.

\begin{lemma} \label{lem:inlaw}
Let $F$ be a $d$-dimensional inner product distribution, with
$(A,X) \sim \RDPG(F,n)$ and let $\wtrue \in \supp F$ be the
fixed latent position of the out-of-sample vertex. Then
\begin{equation*}
\sqrt{n} \SP^{-1/2} \UP^T(\avec - X \wtrue)
        \inlaw \calN(0, \Sigma_{F,\wtrue} ),
\end{equation*}
where
$\Sigma_{F,\wtrue} =
\Delta^{-1} \E\left[X_1^T \wtrue(1-X_1^T\wtrue) X_1 X_1^T \right] \Delta^{-1}$.
\end{lemma}
\begin{proof}
We begin by observing that since $\wtrue \in \bbR^d$ is fixed,
\begin{equation*}
n^{-1/2} X^T(\avec - X \wtrue)
= n^{-1/2} \sum_{i=1}^n (\avec_i - X_i^T \wtrue) X_i
\end{equation*}
is a scaled sum of of $n$ independent $0$-mean
$d$-dimensional random vectors, each with covariance matrix
\begin{equation*}
 V_{\wtrue} = \E X_1^T \wtrue(1-X_1^T\wtrue) X_1 X_1^T \in \bbR^{d \times d}.
\end{equation*}
The multivariate central limit theorem implies that
\begin{equation*}
 n^{-1/2} X^T(\avec - X \wtrue) X_i \inlaw \calN(0,V_{\wtrue}).
\end{equation*}
We have
$\sqrt{n} \SP^{-1/2} \UP^T( \avec - X \wtrue)
= n \SP^{-1} n^{-1/2} X^T( \avec - X \wtrue )$.
By the WLLN, $\SP/n \inprob \Delta$, and hence
by the continuous mapping theorem, $n \SP^{-1} \inprob \Delta^{-1}$.
Thus, the multivariate version of Slutsky's Theorem implies that
\begin{equation*}
 \sqrt{n} \SP^{-1/2} \UP^T( \avec - X \wtrue) \inlaw
        \calN( 0, \Delta^{-1} V_{\wtrue} \Delta^{-1}),
\end{equation*}
as we set out to show.
\end{proof}

The following technical lemma will be crucial for proving one of the
convergence results required by our main theorem.
Its comparative complexity merits stating it here rather than
including it in the proof of Theorem~\ref{thm:ase:lsclt} below.
We remind the reader that $\SA,\SP \in \bbR^{d \times d}$ are the diagonal
matrices formed by the top $d$ eigenvalues of $A$ and $P$, repsectively,
and $\UA,\UP \in \bbR^{n \times d}$  are the matrices whose columns are the
corresponding unit eigenvectors.
\begin{lemma} \label{lem:exchangeable}
With notation as above,
\begin{equation*}
\sqrt{n}\SA^{-1/2}(\UA^T - \UA^T\UP \UP^T)(\avec - X \wtrue)
  \inprob 0 .
\end{equation*}
\end{lemma}
\begin{proof}
For ease of notation, define the vector
\begin{equation*}
  \zvec = (\UA^T - \UA^T \UP \UP^T)(\avec - X \wtrue ).
\end{equation*}
Let $\epsilon > 0$ be a constant, and note that for suitably large $n$,
\begin{equation*}
\Pr\left[ \sqrt{n} \| \SA^{-1/2} \zvec \| > \epsilon \right]
\le \Pr\left[ \sqrt{n} \| \SA^{-1/2} \zvec \| > C_0 n^{-1/4} \right],
\end{equation*}
where $C_0 > 0$ is a constant that we are free to choose.
Define the events
\begin{equation*} \begin{aligned}
E_{1,n} &= \{ \| \SA^{-1/2} \| \le C_1 n^{-1/2} \}, \\
&\text{and} \\
E_{2,n} &= \{ \sqrt{n} \| \zvec \| \le C_2 n^{1/4} \},
\end{aligned} \end{equation*}
and note that
$\Pr\left[ \sqrt{n} \| \SA^{-1/2} \zvec \| > C_0 n^{-1/4} \right]
\le \Pr\left[ (E_{1,n} \cap E_{2,n})^c \right]$
so long as $C_1 C_2 \le C_0$.
Thus, it will suffice for us to show that
$ \lim_{n\rightarrow \infty} \Pr\left[ (E_{1,n} \cap E_{2,n})^c \right]
	\rightarrow 0$.
The proof of Lemma~\ref{lem:Pspecgrowth} implies that
$\lim_{n\rightarrow \infty} \Pr[ E_{1,n}^c ] = 0$,
so our proof will be complete once we show that
$\lim_{n\rightarrow \infty} \Pr[ E_{2,n}^c ] = 0$.

Toward this end, define the matrix
\begin{equation*}
W = \onevec_n^T \otimes \wtrue
  = \begin{bmatrix} \wtrue & \wtrue & \dots & \wtrue \end{bmatrix}
        \in \bbR^{d \times n}
\end{equation*}
and let $B \in \bbR^{n \times n}$ be a random matrix with
independent binary entries with
$\E B_{i,j} = (X W)_{i,j} = X_i^T \wtrue.$
Define the event
\begin{equation*} 
E_{3,n} = \{ \| (\UA^T - \UA^T\UP\UP^T)(B-XW) \|_F^2
                \le C \log^2 n \}.
\end{equation*}
Since $\Pr[ E_{2,n}^c ] \le \Pr[ E_{2,n}^c \mid E_{3,n} ]
		+ \Pr[ E_{3,n}^c ]$,
it will suffice to show that
\begin{enumerate}
\item $\lim_{n \rightarrow \infty} \Pr\left[ E_{3,n}^c \right] = 0$,
	and
\item $\lim_{n \rightarrow \infty}
	\Pr\left[ E_{2,n}^c \mid E_{3,n} \right] = 0$.
\end{enumerate}
By submultiplicativity, we have
\begin{equation} \label{eq:Zsubmult}
\| (\UA^T - \UA^T\UP\UP^T)(B-XW) \|_F^2
  \le \| \UA^T - \UA^T\UP\UP^T \|_F^2 \| B - XW \|^2.
\end{equation}
Theorem~\ref{thm:asymbern} applied to $B - X W$
implies that with probability $1-O(n^{-2})$,
\begin{equation} \label{eq:asymbern}
\| B - X W \| \le C n^{1/2} \log^{1/2} n.
\end{equation}
Theorem 2 in \cite{YuWanSam2015} guarantees an orthogonal
$R^* \in \bbR^{d \times d}$ such that
\begin{equation} \label{eq:UAUPfrobenius}
\| \UA - \UP R^* \|_F 
\le \frac{ C \| A - P \| }{ \lambda_d(P) }
= O\left( \frac{ \log^{1/2} n }{ \sqrt{n} } \right),
\end{equation}
where we have used Lemma~\ref{lem:Pspecgrowth} to lower-bound $\lambda_d(P)$
and bounded $\|A-P\| = O(n^{1/2} \log^{1/2} n)$
by a result in \cite{Oliveira2010}.
Since $R = \UA^T \UP$ solves the minimization
\begin{equation*}
  \min_{R \in \bbR^{d \times d} } \| \UA^T R - \UA^T \UP \UP^T \|_F,
\end{equation*}
Equation~\eqref{eq:UAUPfrobenius} implies
\begin{equation*} 
\| \UA^T - \UA^T \UP \UP^T \|_F
\le \| \UA^T - R^* \UP^T \|_F = O( n^{-1/2} \log^{1/2} n ).
\end{equation*}
Plugging this and~\eqref{eq:asymbern} back into~\eqref{eq:Zsubmult}, we have
that with probability $1-O(n^{-2})$,
\begin{equation} \label{eq:Zbound}
\| (\UA^T - \UA^T\UP\UP^T)(B-X W) \|_F^2
\le C \log^2 n
\end{equation}
which is to say, $\Pr[E_{3,n}^c] = O(n^{-2})$.

It remains to show that $\Pr[E_{2,n}^c \mid E_{3,n} ] \rightarrow 0$.
By construction, the columns of the matrix
$(\UA^T - \UA^T\UP \UP^T)(B - X W)$
are $n$ independent copies of $\zvec$.
Using this fact and the conditional Markov inequality, we have
\begin{equation*} \begin{aligned}
\Pr[E_{2,n}^c \mid E_{3,n} ]
&= \Pr[ \sqrt{n} \| \zvec \| > C_2 n^{1/4} \mid E_{3,n} ]
\le \frac{ n\E[\| \zvec \|^2 \mid E_{3,n} ] }{ C_2^2 n^{1/2} } \\
&= \frac{ \E[ \| (\UA^T - \UA^T\UP \UP^T)(B - XW) \|_F^2
                \mid E_{3,n} ] }{ C_2^2 n^{1/2} }
\le \frac{ C \log^2 n }{ n^{1/2} },
\end{aligned} \end{equation*}
where the last inequality follows from the definition of event $E_{3,n}$.
This quantity goes to zero in $n$, thus completing the proof.
\end{proof}

The following technical lemma will prove useful in our proof of
Theorem~\ref{thm:ase:lsclt} below.
We state it here rather than proving it in-line for the sake of clarity.
\begin{lemma} \label{lem:unitaryhoeff}
With notation as above,
\begin{equation*}
  \| \UP^T(\avec - X\wtrue) \| = O( n^{1/2} \log^{1/2} n ).
\end{equation*}
\end{lemma}
\begin{proof}
For $k \in [d]$ and $i \in [n]$, observe that
\begin{equation*}
\left( \UP^T(\avec - X\wtrue) \right)_{k,i}
= \sum_{j=1}^n (\UP)_{j,k}(a_j - X_j^T \wtrue)
\end{equation*}
is a sum of independent $0$-mean random variables,
and Hoeffding's inequality yields
\begin{equation*}
\Pr\left[ | \UP^T(\avec - X\wtrue) |_{k,i} \ge t \right]
  \le 2\exp\left\{ \frac{ -t^2 }{ 2\sum_{j=1}^n (\UP)^2_{k,j} } \right\}
        = 2\exp\left\{ \frac{ -t^2 }{ 2 } \right\}.
\end{equation*}
Taking $t = C \log^{1/2} n$ for suitably large constant $C > 0$,
a union bound over all $k\in[d]$ and $i \in [n]$ followed by
the Borel-Cantelli Lemma yields the result.
\end{proof}

We are now ready to present the proof of Theorem~\ref{thm:ase:lsclt}.
\begin{proof}[Proof of Theorem~\ref{thm:ase:lsclt}]
Let $Q = Q_n \in \bbR^{d \times d}$ denote the orthogonal matrix guaranteed
to exist by Lemma~\ref{lem:2toinfty}.
Adding and subtracting appropriate quantities,
\begin{equation} \label{eq:ase:clt:decomp}
\begin{aligned}
\sqrt{n}(Q \whatls - \wtrue)
&= \sqrt{n} Q \left( \SA^{-1/2} \UA^T \avec - Q^T \wtrue \right) \\
&= \sqrt{n}\SP^{-1/2} \UP^T (\avec - X\wtrue ) \\
&~~~~~~ + \sqrt{n} Q \SA^{-1/2}(\UA^T - Q^T \UP^T)(\avec - X \wtrue ) \\
&~~~~~~ + \sqrt{n} Q ( \SA^{-1/2} \UA^T X - Q^T)\wtrue \\
&~~~~~~ + \sqrt{n} Q (\SA^{-1/2}Q^T  - Q^T \SP^{-1/2})\UP^T(\avec-X\wtrue).
\end{aligned}
\end{equation}
By Lemma~\ref{lem:inlaw}, the first of these terms converges in law:
\begin{equation} \label{eq:inlaw:restate}
\sqrt{n} \SP^{-1/2} \UP^T (\avec - X\wtrue )
\inlaw \calN(0, \Sigma_{F,\wtrue}),
\end{equation}
where $\Sigma_{F,\wtrue}$ is as defined in Lemma~\ref{lem:inlaw}.
Thus, by Slutsky's Theorem, our proof will be complete
once we show that the remaining terms in Equation~\eqref{eq:ase:clt:decomp}
go to zero in probability.

Since $Q$ is orthogonal, it suffices to prove that
\begin{equation} \label{eq:inprob:1}
\sqrt{n} \SA^{-1/2}(\UA^T - Q^T\UP^T)(\avec - X \wtrue ) \inprob 0,
\end{equation}
\begin{equation} \label{eq:inprob:2}
\sqrt{n}( \SA^{-1/2} \UA^T X - Q^T)\wtrue \inprob 0,
\end{equation}
and
\begin{equation} \label{eq:inprob:3}
\sqrt{n}(\SA^{-1/2}Q^T - Q^T\SP^{-1/2})\UP^T(\avec-X\wtrue) \inprob 0.
\end{equation}
We will address each of these three convergences in order.

To see the convergence in~\eqref{eq:inprob:1},
adding and subtracting appropriate quantities gives
\begin{equation} \label{eq:inprob1:split}
\begin{aligned}
\sqrt{n} \SA^{-1/2}(\UA^T - Q^T \UP^T)(\avec - X \wtrue )
&= \sqrt{n}\SA^{-1/2}(\UA^T \UP \UP^T - Q^T \UP^T)(\avec - X \wtrue) \\
&~~~~~~+ \sqrt{n}\SA^{-1/2}(\UA^T - \UA^T\UP \UP^T)(\avec - X \wtrue).
\end{aligned}
\end{equation}
To bound the first of these two summands,
Lemmas~\ref{lem:Pspecgrowth},~\ref{lem:unitaryhoeff}
and~\ref{lem:OMNI:Qdef} imply
\begin{equation*} \begin{aligned}
\| \sqrt{n}\SA^{-1/2}(\UA^T \UP \UP^T - Q^T \UP^T)(\avec - X \wtrue) \|
&\le \sqrt{n} \| \SA^{-1/2} \|
	\| \UA^T \UP - Q^T \| \| \UP^T(\avec-X\wtrue) \|_F \\
&= O( n^{-1/2} \log^{3/2} n ).
\end{aligned} \end{equation*}
Lemma~\ref{lem:exchangeable} shows that the second term
in~\eqref{eq:inprob1:split} also goes to zero in probability,
and Equation~\eqref{eq:inprob:1} follows.

To see~\eqref{eq:inprob:2}, note that
\begin{align}
\sqrt{n}& (\SA^{-1/2}\UA^T X - Q^T)\wtrue
= \sqrt{n}\left( \SA^{-1/2}\UA^T \UP \SP^{1/2} - Q^T \right) \wtrue
	\nonumber \\
&= \sqrt{n}\SA^{-1/2}\left( \UA^T\UP - Q^T \right) \SP^{1/2}\wtrue
  + \sqrt{n}\SA^{-1/2}\left( Q^T \SP^{1/2} - \SA^{1/2} Q^T \right)\wtrue.
	\label{eq:inprob:med}
\end{align}
Submultiplicativity of matrix norms combined with Lemmas~\ref{lem:Pspecgrowth}
and~\ref{lem:OMNI:Qdef} and the fact that $\| \wtrue \| \le 1$
imply
\begin{equation} \label{eq:inprob2:tri1}
\begin{aligned}
\| \sqrt{n}\SA^{-1/2}&\left( \UA^T\UP - Q^T \right) \SP^{1/2} \wtrue \|
\le C \sqrt{n} \| \SA^{-1/2} \|
	\| \UA^T \UP - Q^T \|_F \| \SP^{1/2} \| \| \wtrue \| \\
&= O( n^{-1/2} \log n ).
\end{aligned}
\end{equation}
Applying Lemma~\ref{lem:Pspecgrowth} again and taking the Frobenius norm as
a trivial upper bound on the spectral norm,
Lemma~\ref{lem:LSE:Qdef} implies
\begin{equation} \label{eq:inprob2:tri2}
\begin{aligned}
\| \sqrt{n}\SA^{-1/2}&\left( Q^T \SP^{1/2} - \SA^{1/2} Q^T \right) \wtrue \|
\le C \sqrt{n} \| \SA^{-1/2} \|
	\| Q^T \SP^{1/2} - \SA^{1/2} Q^T \| \| \wtrue \| \\
&\le C \| Q \SP^{1/2} - \SA^{1/2} Q \|,
\end{aligned}
\end{equation}
where we have used the fact that the spectral norm is preserved by
matrix transposition.
Adding and subtracting appropriate quantities,
\begin{equation*}
Q \SP^{1/2} - \SA^{1/2} Q
= (Q - \UA^T\UP) \SP^{1/2} + \SA^{1/2}( \UA^T\UP - Q)
	+ \UA^T\UP \SP^{1/2} - \SA^{1/2} \UA^T\UP.
\end{equation*}
By the triangle inequality and submultiplicativity,
\begin{equation} \label{eq:swap}
\| Q \SP^{1/2} - \SA^{1/2} Q \|
\le \left( \| \SP^{1/2} \| + \| \SA^{1/2} \| \right) \| \UA^T\UP - Q \|
	+ \| \UA^T\UP \SP^{1/2} - \SA^{1/2} \UA^T\UP \|.
\end{equation}
Lemmas~\ref{lem:Pspecgrowth} and~\ref{lem:OMNI:Qdef} bound the first term
as $O( n^{-1/2} \log n)$,
and the second term is bounded by Lemma~\ref{lem:approxcommute},
and thus Equation~\eqref{eq:inprob2:tri2} is bounded as
\begin{equation*}
\| \sqrt{n}\SA^{-1/2}\left( Q^T \SP^{1/2} - \SA^{1/2} Q^T \right) \wtrue \|
= O( n^{-1/2} \log n ).
\end{equation*}
Applying this and Equation~\eqref{eq:inprob2:tri1}
to Equation~\eqref{eq:inprob:med}
proves~\eqref{eq:inprob:2} by the triangle inequality.

Finally, to prove~\eqref{eq:inprob:3}, note that
\begin{equation*}
\| \sqrt{n}(\SA^{-1/2}Q^T - Q^T \SP^{-1/2})\UP^T(\avec-X\wtrue) \|
\le \sqrt{n} \| \SA^{-1/2}Q^T - Q^T\SP^{-1/2} \| \| \UP^T(\avec - X\wtrue \|_F,
\end{equation*}
and Lemmas~\ref{lem:approxcommute} and~\ref{lem:unitaryhoeff}
along with an argument similar to the bound in
Equation~\eqref{eq:swap} imply that
\begin{equation*}
\| \sqrt{n}(\SA^{-1/2}Q^T - Q^T\SP^{-1/2})\UP^T(\avec-X\wtrue) \|
  = O( n^{-1/2} \log^{3/2} n),
\end{equation*}
which completes the proof.
\end{proof}


\section{Proof of LSE linear least squares out-of-sample CLT}
\label{apx:lseclt}
In this section, we prove Theorem~\ref{thm:lse:lsclt},
which shows that the least-squares out-of-sample extension for the
Laplacian spectral embedding is, in the large-$n$ limit,
normally distributed about the true embedding
$\wtilde = \wtrue/\sqrt{n\mu^T \wtrue}$, after appropriate rescaling.
We remind the reader that
$\avec \in \bbR^n$ denotes the vector of edges between the out-of-sample
vertex $v$ and the in-sample vertices $V = [n]$
and $D \in \bbR^n$ is the diagonal matrix of in-sample
node degrees, so that $D_{i,i} = d_i = \sum_{j=1}^n A_{i,j}$.
Below, we will also need to define the matrix
\begin{equation*}
T = \diag( t_1, t_2,\dots,t_n ) \in \bbR^{n \times n},~~~
t_i = \sum_{j=1}^n X_j^T X_i,
\end{equation*}
the matrix of in-sample expected degrees
conditioned on the latent positions.
Analogously, we denote the out-of-sample vertex
degree $d_v = \sum_{j=1}^n a_j$,
and its expectation $t_v = \sum_{j=1}^n X_j^T \wtrue$.
Recall that the LSE least-squares out-of-sample extension is given by
\begin{equation*}
 \wcheckls = (\Xcheck^T \Xcheck)^{-1} \Xcheck^T D^{-1/2}
		\frac{ \avec }{ \sqrt{d_v} }.
\end{equation*}

Our aim is to prove that for a suitably-chosen sequence of orthogonal
matrices $\Qtilde \in \bbR^{d \times d}$,
\begin{equation*}
n(\Qtilde \wcheckls - \wtilde) \inlaw \calN(0,\Sigmatilde_{F,\wtrue}),
\end{equation*}
where $\Sigmatilde_{F,\wtrue}$
depends only on the latent position distribution $F$
and the true out-of-sample latent position $\wtrue \in \supp F$,
and is given by
\begin{equation*}
\Sigmatilde_{F,\wtrue} 
= \E \left[ \frac{ X_j^T\wtrue(1-X_j^T\wtrue) }{ \mu^T\wtrue }
        \left( \frac{\Deltatilde X_j}{X_j^T \mu}                                                         - \frac{ \wtrue}{2\mu^T \wtrue } \right)
        \left( \frac{\Deltatilde X_j}{X_j^T \mu}                                                         - \frac{ \wtrue}{2\mu^T \wtrue } \right)^T \right]
\in \bbR^{d \times d},
\end{equation*}
where $\Deltatilde = \E X_1 X_1^T /(X_1^T \mu)$ with
$\mu = \E X_1$ is the mean of $F$.

\begin{proof}[Proof of Theorem~\ref{thm:lse:lsclt}]
Take $\Qtilde \in \bbR^{d \times d}$ to be the matrix guaranteed
by Lemma~\ref{lem:2toinfty}.
Similarly to the proof of Theorem~\ref{thm:ase:lsclt},
our proof will proceed by writing $n(\wcheckls - \Qtilde\wtilde)$ as a sum,
\begin{equation*}
  n(\Qtilde \wcheckls - \wtilde) = n \gvec + n \hvec,
\end{equation*}
where $n \hvec \inprob 0$ and $n \gvec$ converges in law to our desired
normal distribution, whence Slutsky's Theorem will yield the result.
We begin by writing
\begin{equation} \label{eq:path:1}
n(\wcheckls - \Qtilde^T \wtilde)
= n(\Xcheck^T \Xcheck)^{-1} \frac{ \Xcheck^T D^{-1/2} \avec }{ \sqrt{d_v} }
	- n \Ucheck^T \Utilde \wtilde
	- n (\Qtilde^T - \Ucheck^T \Utilde) \wtilde.
\end{equation}
By submultiplicativity of the spectral norm,
Lemma~\ref{lem:LSE:Qdef} and the definition of
$\wtilde = \wtrue/\sqrt{n \mu^T \wtrue}$,
\begin{equation*}
\| (\Qtilde^T - \Ucheck^T \Utilde) \wtilde \|
\le \| \Qtilde^T - \Ucheck^T \Utilde \| \| \wtilde \|
	\le \frac{ C \| \wtrue \| }{ n^{3/2} }.
\end{equation*}
Applying this to Equation~\eqref{eq:path:1} and using the
fact that $\| \wtrue \|$ is bounded, we have
\begin{equation} \label{eq:path:2}
n(\wcheckls - \Qtilde^T \wtilde)
= n(\Xcheck^T \Xcheck)^{-1} \frac{ \Xcheck^T D^{-1/2} \avec }{ \sqrt{d_v} }
  - n \Ucheck^T \Utilde \wtilde + O( n^{-1/2} ).
\end{equation}
Adding and subtracting quantities,
\begin{equation} \label{eq:orphan}
\Ucheck^T \Utilde \wtilde
= \Scheck^{-1/2} \Ucheck^T\Utilde \Stilde^{1/2} \wtilde
 - (\Scheck^{-1/2} \Ucheck^T\Utilde \Stilde^{1/2} - \Ucheck^T \Utilde) \wtilde.
\end{equation}
By Lemma~\ref{lem:switching},
\begin{equation*}
\| \Ucheck^T\Utilde \Stilde^{1/2} - \Scheck^{1/2} \Ucheck^T \Utilde \|
= O( n^{-1} ),
\end{equation*}
so that, applying submultiplicativity followed by
Lemmas~\ref{lem:lapspecgrowth} and~\ref{lem:approxcommute},
\begin{equation*} 
\| (\Scheck^{-1/2} \Ucheck^T\Utilde \Stilde^{1/2}
	- \Ucheck^T \Utilde) \wtilde \|
\le
\| \Scheck^{-1/2} \|
\| \Ucheck^T \Utilde \Stilde^{1/2} - \Scheck^{1/2} \Ucheck^T \Utilde \|
	\| \wtilde \| \\
= O( n^{-3/2} ).
\end{equation*}
Plugging this into Equation~\ref{eq:orphan}, we have shown that
\begin{equation*}
n\Ucheck^T \Utilde \wtilde
= n\Scheck^{-1/2} \Ucheck^T\Utilde \Stilde^{1/2} \wtilde + O(n^{-1/2}),
\end{equation*}
and plugging this, in turn, into Equation~\eqref{eq:path:2}, we have
\begin{equation*} \begin{aligned}
n(\wcheckls - \Qtilde^T \wtilde)
&= n(\Xcheck^T \Xcheck)^{-1} \frac{ \Xcheck^T D^{-1/2} \avec }{ \sqrt{d_v} }
  - n\Scheck^{-1/2} \Ucheck^T\Utilde \Stilde^{1/2} \wtilde + O(n^{-1/2}) \\
&= n(\Xcheck^T \Xcheck)^{-1} \Xcheck^T
        \left( \frac{ D^{-1/2} \avec }{ \sqrt{d_v} }
        - \Xtilde \wtilde \right) + O( n^{-1/2} ),
\end{aligned}\end{equation*}
where the second equality follows from the definitions of $\Xcheck$
and $\Xtilde$ and $\Xcheck^T \Xcheck = \Scheck$.
Again adding and subtracting quantities, we have
\begin{equation} \label{eq:path:3}
\begin{aligned}
n(\wcheckls - \Qtilde^T \wtilde)
&= n(\Xcheck^T \Xcheck)^{-1} \Qtilde^T \Xtilde^T
  \left( \frac{ D^{-1/2} \avec }{ \sqrt{d_v} } - \Xtilde \wtilde \right)\\
&~~~~~~+ n(\Xcheck^T \Xcheck)^{-1}(\Xcheck - \Xtilde \Qtilde)^T
  \left( \frac{ D^{-1/2} \avec }{ \sqrt{d_v} } - \Xtilde \wtilde \right)
+ O( n^{-1/2} ).
\end{aligned}
\end{equation}
Expanding the second term on the right-hand side,
\begin{equation*} \begin{aligned}
(\Xcheck &- \Xtilde \Qtilde)^T
        \left( \frac{ D^{-1/2} \avec }{ \sqrt{d_v} } - \Xtilde \wtilde \right)
= \sum_{j=1}^n \left( \frac{ a_j }{\sqrt{ d_j d_v }}
		- \frac{ X_j^T \wtrue }{ \sqrt{ t_j n\mu^T\wtrue } } \right)
		(\Xcheck_j - \Qtilde^T \Xtilde_j) \\
&= \sum_{j=1}^n \frac{ a_j - X_j^T\wtrue }{ \sqrt{ d_j d_v } }
		(\Xcheck_j - \Qtilde^T \Xtilde_j)
  + \sum_{j=1}^n \left( \frac{1}{\sqrt{t_j n \mu^T\wtrue}}
		- \frac{1}{\sqrt{d_j d_v}} \right)
		X_j^T \wtrue (\Xcheck_j - \Qtilde^T\Xtilde_j).
\end{aligned} \end{equation*}
Recalling that $\avec$ is independent of $A$ conditioned on $X$ and that
$\E[ a_j \mid X_j ] = X_j^T \wtrue$,
the first of these two summations is a sum of independent zero-mean random
variables, and an application of Hoeffding's inequality
along with Lemmas~\ref{lem:2toinfty} and~\ref{lem:degreegrowth} yields
\begin{equation*}
\sum_{j=1}^n \frac{ a_j - X_j^T\wtrue }{ \sqrt{ d_j d_v } }
                (\Xcheck_j - \Qtilde^T \Xtilde_j)
= O(n^{-3/2} \log n).
\end{equation*}
Again applying Lemmas~\ref{lem:2toinfty} and~\ref{lem:degreegrowth},
\begin{equation*} \begin{aligned}
\sum_{j=1}^n & \left( \frac{1}{\sqrt{t_j n \mu^T\wtrue}}
                - \frac{1}{\sqrt{d_j d_v}} \right) X_j^T\wtrue
			(\Xcheck_j - \Qtilde^T\Xtilde_j)\\
&= 
\sum_{j=1}^n \left(\frac{1}{\sqrt{n \mu^T\wtrue}}-\frac{1}{\sqrt{d_v}}\right)
		\frac{ X_j^T \wtrue }{ \sqrt{t_j} }
		\left( \Xcheck_j - \Qtilde^T \Xtilde_j \right)
+
\sum_{j=1}^n \left(\frac{1}{\sqrt{t_j}}-\frac{1}{\sqrt{d_j}}\right)
                \frac{ X_j^T \wtrue }{ \sqrt{d_v} }
                \left( \Xcheck_j - \Qtilde^T \Xtilde_j \right) \\
&= O( n^{-3/2} \log n ) 
\end{aligned} \end{equation*}
Thus, the above two displays imply that
\begin{equation*}
(\Xcheck - \Xtilde \Qtilde)^T
        \left( \frac{ D^{-1/2} \avec }{ \sqrt{d_v} } - \Xtilde \wtilde \right)
= O( n^{-3/2} \log n ).
\end{equation*}
Recalling that $\Scheck = \Xcheck^T \Xcheck$,
Lemmas~\ref{lem:Scheck} and~\ref{lem:fullrank} imply
that $\Scheck$ is invertible eventually,
and $\| (\Xcheck^T \Xcheck)^{-1} \| = \Theta(1)$.
Equation~\eqref{eq:path:3} thus becomes
\begin{equation*}
n(\wcheckls - \Qtilde^T \wtilde)
= 
n\Scheck^{-1} \Qtilde^T \Xtilde^T
        \left( \frac{ D^{-1/2} \avec }{ \sqrt{d_v} } - \Xtilde \wtilde \right)
+ \Otilde( n^{-1/2} ),
\end{equation*}
and multiplying through by $\Qtilde$ yields
\begin{equation*}
n(\Qtilde \wcheckls - \wtilde)
= 
n \Qtilde \Scheck^{-1} \Qtilde^T \Xtilde^T
        \left( \frac{ D^{-1/2} \avec }{ \sqrt{d_v} } - \Xtilde \wtilde \right)
+ \Otilde( n^{-1/2} ).
\end{equation*}

Lemma~\ref{lem:Scheck} and the continuity of the inverse imply that
\begin{equation*} 
\Qtilde \Scheck^{-1} \Qtilde^T \inprob \Deltatilde^{-1}.
\end{equation*}
An application of Slutsky's Theorem will thus yield our result,
provided we can show that
\begin{equation} \label{eq:clt:goal:2}
n \Xtilde^T
	\left( \frac{ D^{-1/2} \avec }{ \sqrt{d_v} } - \Xtilde \wtilde \right)
\inlaw \calN( 0, \Sigma_{F,\wtrue} ),
\end{equation}
where
\begin{equation*}
\Sigma_{F,\wtrue}
= 
\E \left[ \frac{ X_j^T\wtrue(1-X_j^T\wtrue) }{ \mu^T\wtrue }
        \left( \frac{X_j}{X_j^T \mu}                                                         - \frac{ \Deltatilde \wtrue}{2\mu^T \wtrue } \right)
        \left( \frac{X_j}{X_j^T \mu}                                                         - \frac{ \Deltatilde \wtrue}{2\mu^T \wtrue } \right)^T \right].
\end{equation*}

To establish~\eqref{eq:clt:goal:2}, we recall
$t_v = \sum_{j=1}^n X_j^T\wtrue = \E d_v$ and note that
\begin{equation*} \begin{aligned}
n \Xtilde^T\left(
\frac{ D^{-1/2} \avec }{\sqrt{d_v}} - \Xtilde \wtilde \right)
&= \frac{ n\Xtilde^T T^{-1/2} (\avec - Xw) }{ \sqrt{ t_v } }
+ n \Xtilde^T\left( \frac{D^{-1/2}}{\sqrt{d_v}}
	-\frac{T^{-1/2}}{\sqrt{ t_v }} \right) X \wtrue \\
&~~~~~~+ n\Xtilde^T \left( \frac{D^{-1/2}}{\sqrt{d_v}}
        -\frac{T^{-1/2}}{\sqrt{ t_v }} \right)(\avec - Xw).
\end{aligned} \end{equation*}
The last of these terms is $O( n^{-1/2} \log n )$
by a Hoeffding inequality followed by an application of
Lemma~\ref{lem:degreegrowth}, so that
\begin{equation} \label{eq:fork:1}
\begin{aligned}
n \Xtilde^T\left(
\frac{ D^{-1/2} \avec }{\sqrt{d_v}} - \Xtilde \wtilde \right)
&= n \Xtilde^T T^{-1/2} \frac{ (\avec - Xw) }{ \sqrt{ t_v } } \\
&~~~~~~+ n \Xtilde^T \left(
	\frac{D^{-1/2}}{\sqrt{d_v}}
        -\frac{T^{-1/2}}{\sqrt{ t_v }} \right) X \wtrue
+ O( n^{-1/2} \log n ).
\end{aligned}
\end{equation}
Multiplying numerator and denominator and applying Lemma~\ref{lem:degreegrowth},
it holds for all $i \in [n]$
\begin{equation*} \begin{aligned}
\frac{1}{\sqrt{d_i}} - \frac{1}{\sqrt{t_i}}
&= \frac{ t_i - d_i }{ (\sqrt{d_i} + \sqrt{t_i})\sqrt{d_i t_i} }
= \frac{t_i-d_i}{2t_i^{3/2}}
	+ (t_i - d_i)\frac{t_i(\sqrt{t_i}-\sqrt{d_i})
		+ (t_i-d_i)\sqrt{t_i}}
	{ 2t_i^{3/2}(d_i\sqrt{t_i} + t_i\sqrt{d_i}) } \\
&= \frac{t_i-d_i}{2t_i^{3/2}} + O(n^{-3/2} \log n),
\end{aligned} \end{equation*}
and a similar result holds for the out-of-sample vertex, in that
\begin{equation*}
\frac{1}{\sqrt{d_v}} - \frac{1}{\sqrt{t_v}}
= \frac{t_v-d_v}{2t_v^{3/2}} + O(n^{-3/2} \log n).
\end{equation*}
Thus,
\begin{equation*} \begin{aligned}
\Xtilde^T & \left( \frac{D^{-1/2}}{\sqrt{d_v}}
        -\frac{T^{-1/2}}{\sqrt{t_v}} \right) X\wtrue \\
&=
\Xtilde^T T^{-1/2}\left(\frac{1}{\sqrt{d_v}}
		- \frac{1}{\sqrt{t_v}}\right) X \wtrue
+ \Xtilde^T \frac{ (D^{-1/2}-T^{-1/2}) X \wtrue }{ \sqrt{d_v} } \\
&=
\Xtilde^T T^{-1/2} \frac{ t_v - d_v }{ 2 t_v^{3/2} }X\wtrue
+ \frac{ \Xtilde^T T^{-3/2}(T-D) X\wtrue }{ 2\sqrt{d_v} }
+ \sum_{j=1}^n \xi_j X_j^T \wtrue
	\left( \frac{1}{\sqrt{t_j}} + \frac{1}{\sqrt{d_j}} \right)
	\frac{ X_j }{\sqrt{t_j} }
\end{aligned} \end{equation*}
where $\xi_j \in \bbR, j=1,2,\dots,n$ satisfy $\xi_j = O(n^{-3/2} \log n)$.
Using Lemma~\ref{lem:degreegrowth}, this last sum is itself
$O(n^{-3/2} \log n)$, so that
\begin{equation*} \begin{aligned}
n\Xtilde^T \left( \frac{D^{-1/2}}{\sqrt{d_v}}
        -\frac{T^{-1/2}}{\sqrt{t_v}} \right) X\wtrue
&= n\Xtilde^T T^{-1/2}
	\frac{ t_v - d_v }{ 2t_v^{3/2} }X\wtrue \\
&~~~~~~+ n \Xtilde^T \frac{ T^{-3/2}(T-D) X\wtrue }{ 2\sqrt{d_v} }
+ O(n^{-1/2} \log n).
\end{aligned} \end{equation*}
Plugging this into Equation~\eqref{eq:fork:1},
\begin{equation*} \begin{aligned}
n \Xtilde^T\left(
\frac{ D^{-1/2} \avec }{\sqrt{d_v}} - \Xtilde \wtilde \right)
&= n \Xtilde^T T^{-1/2} \frac{ (\avec - Xw) }{ \sqrt{t_v} }
+ n\Xtilde^T T^{-1/2} \frac{ t_v - d_v }{ 2t_v^{3/2} }X\wtrue \\
&~~~~~~+ n \Xtilde^T \frac{ T^{-3/2}(T-D) X\wtrue }{ 2\sqrt{d_v} }
+ O(n^{-1/2} \log n).
\end{aligned} \end{equation*}
To complete our proof, it will suffice to show the following two facts:
\begin{equation} \label{eq:finisher:1}
n \Xtilde^T T^{-1/2} 
	\left( \frac{ (\avec - X\wtrue) }{ \sqrt{t_v} }
		+ \frac{ t_v - d_v }{ 2t_v^{3/2} }X\wtrue
	\right) \inlaw \calN(0, \Sigma_{F,\wtrue} )
\end{equation}
\begin{equation} \label{eq:finisher:2}
n \Xtilde^T \frac{ T^{-3/2}(T-D) X\wtrue }{ 2\sqrt{d_v} }
	\inprob 0
\end{equation}
To see the latter of these two points, observe that
by our definitions of $d_i = \sum_{j=1}^n A_{i,j}$
and $t_i = \sum_{j=1}^n X_j^TX_i$,
\begin{equation*} \begin{aligned}
n \Xtilde^T &\frac{ T^{-3/2}(T-D) X\wtrue }{ 2\sqrt{d_v} }
= \frac{n}{2\sqrt{d_v} }
	\sum_{i=1}^n \frac{(t_i-d_i)}{t_i^{2}} X_i^T \wtrue X_i \\
&= \frac{n}{2\sqrt{d_v}}
	\sum_{i=1}^n \frac{ X_i^T X_i }{t_i^2} X_i^T \wtrue X_i
 + \frac{n}{2\sqrt{d_v}} \sum_{1 \le i < j \le n}
	(X_j^TX_i-A_{i,j})
	\left( \frac{ X_i^T \wtrue }{ t_i^2 }X_i
		+ \frac{ X_j^T \wtrue }{ t_j^2 }X_j \right).
\end{aligned} \end{equation*}
The former of these two sums is $O(n^{-1/2})$
by an application of Lemma~\ref{lem:degreegrowth}
and using the fact that $X_i \in \supp F$ are bounded.
The latter of these two sums is,
conditioned on $\{X_i\}_{i=1}^n$,
a sum of independent $0$-mean random variables, with
$\| t_j^{-2} (X_j^T X_i - A_{i,j}) X_j^T \wtrue X_j \|
	\in [-Ct_j^{-2}, Ct_j^{-2}]$ for all $j \in [n]$.
Thus,
\begin{equation*}
\Pr\left[ \left| \sum_{1 \le i < j \le n}
        t_j^{-2}( X_j^T X_i - A_{i,j} ) X_j^T \wtrue X_j \right|
	\ge s \right]
\le 2\exp\left\{ \frac{ -Cs^2 }{ \sum_{i<j} t_j^{-4} } \right\}.
\end{equation*}
Let $E_n = \{ t_j \ge Cn : j=1,2,\dots,n \}$
denote the high-probability event of Lemma~\ref{lem:degreegrowth},
for which we have
$\Pr[ E_n^c ] \le Cn^{-2}$ for all suitably large $n$.
Taking $s=C n^{-1} \log^{1/2} n$ for suitably large $C > 0$,
letting $\bbP_{E_n}$ denote conditional probability
$\Pr[\cdot \mid E_n]$,
\begin{equation*}
\bbP_{E_n} \left[ \left| \sum_{1 \le i < j \le n}
        t_j^{-2} ( X_j^T X_i - A_{i,j} ) X_j^T \wtrue X_j \right|
        \ge Cn^{-1} \log^{1/2} n \right]
\le C n^{-2}.
\end{equation*}
Thus,
\begin{equation*} \begin{aligned}
\Pr & \left[ \left| \sum_{1 \le i < j \le n}
        t_j^{-2} ( X_j^T X_i - A_{i,j} ) X_j^T \wtrue X_j \right|
        \ge Cn^{-1} \log^{1/2} n \right] \\
&\le \bbP_{E_n}\left[ \left| \sum_{1 \le i < j \le n}
        t_j^{-2}( X_j^T X_i - A_{i,j} ) X_j^T \wtrue X_j \right|
        \ge Cn^{-1} \log^{1/2} n \right] + \Pr[ E_n^c ] \\
&\le Cn^{-2},
\end{aligned} \end{equation*}
and we conclude that, bounding $d_v^{-1/2} = O(n^{-1/2} \log^{1/2} n)$
by Lemma~\ref{lem:degreegrowth},
\begin{equation*}
n \Xtilde^T \frac{ T^{-3/2}(T-D) X\wtrue }{ 2\sqrt{d_v} }
= O( n^{-1/2} \log^{1/2} n ),
\end{equation*}
which establishes~\eqref{eq:finisher:2}.

It remains only to prove Equation~\eqref{eq:finisher:1}.
Let $m_i = nX_i^T \mu$ for $i \in [n]$ and define the diagonal matrix
\begin{equation*}
M = \diag(m_1,m_2,\dots,m_n) \in \bbR^{n \times n}.
\end{equation*}
The argument in Lemma~\ref{lem:degreegrowth}
allows us to bound $|t_v^{-1/2}-(n\mu^T\wtrue)^{-1/2}|$, so an argument
similar to that above wherein we apply Hoeffding's inequality followed
by Lemma~\ref{lem:degreegrowth} implies
\begin{equation*}
n \left(\frac{1}{\sqrt{t_v}} - \frac{1}{\sqrt{n \mu^T \wtrue}} \right)
	\Xtilde^T T^{-1/2} (\avec - X\wtrue)
= O( n^{-1/2} \log n ).
\end{equation*}
Lemma~\ref{lem:degreegrowth} also bounds $\max_i |t_i^{-1/2} - m_i^{-1/2}|$,
whence
\begin{equation*}
\frac{n \Xtilde^T(T^{-1/2} - M^{-1/2})(\avec - X\wtrue)}
	{ \sqrt{n \mu^T \wtrue} }
= O(n^{-1/2} \log n ).
\end{equation*}
The same Hoeffding-style argument once again yields,
recalling that $\Xtilde = T^{-1/2} X$,
\begin{equation*}
\frac{ n X^T (T^{-1/2} - M^{-1/2})M^{-1/2}(\avec -X \wtrue)}
	{ \sqrt{ n \mu^T \wtrue } }
= O(n^{-1/2} \log n).
\end{equation*}
Combining the above three displays,
the first term in the quantity of interest in Equation~\eqref{eq:finisher:1} is
\begin{equation} \label{eq:cltpayoff:firstterm}
\frac{ n \Xtilde^T T^{-1/2} (\avec - X \wtrue) }{ \sqrt{t_v} }
= \frac{ n X^T M^{-1} (\avec - X \wtrue) }{ \sqrt{n \mu^T\wtrue} }
	+ \Otilde(n^{-1/2}).
\end{equation}
Turning to the second term on the left-hand side of
Equation~\eqref{eq:finisher:1},
rearranging terms and recalling the definition of
$\Deltatilde = \E X_1 X_1^T /(X_1^T \mu)$,
\begin{equation*}
\frac{ n \Xtilde^T T^{-1/2}(t_v - d_v) X\wtrue }{ 2t_v^{3/2} }
= \frac{ n (t_v - d_v) \Xtilde^T \Xtilde \wtrue }{ 2(n \mu^T \wtrue)^{3/2} }
	+ \Otilde( n^{-1/2} )
= \frac{ n (t_v - d_v) \Deltatilde \wtrue }{ 2(n \mu^T \wtrue)^{3/2} }
	+ \Otilde( n^{-1/2} ),
\end{equation*}
where the first equality follows from Lemma~\ref{lem:degreegrowth}
and the second equality follows from using (multivariate)
Hoeffding's inequality to bound
\begin{equation*}
\| \Xtilde^T \Xtilde - \Deltatilde \|
= \left\| \sum_{i=1}^n \frac{ X_i X_i^T }{ X_i^T \mu } - \Deltatilde \right\|
= O( n^{-1/2} \log^{1/2} n ).
\end{equation*}
Thus, combining with Equation~\eqref{eq:cltpayoff:firstterm},
the quantity on the left-hand side of Equation~\eqref{eq:finisher:1} is
\begin{equation*}
n\Xtilde^T T^{-1/2} 
        \left( \frac{ (\avec - X\wtrue) }{ \sqrt{t_v} }
                + \frac{ t_v - d_v }{ 2t_v^{3/2} }X\wtrue
        \right)
=
\frac{ n X^T M^{-1} (\avec - X\wtrue) }{ \sqrt{n \mu^T \wtrue} }
	+ \frac{ n(t_v - d_v)\Deltatilde\wtrue }{ 2(n\mu^T \wtrue)^{3/2} }
+ O(n^{-1/2} \log^{1/2} n).
\end{equation*}
Rearranging, and recalling $m_i = n X_i^T \mu$,
$t_v = \sum_{j=1}^n X_j^T \wtrue$ and $d_v = \sum_{j=1}^n a_j$,
\begin{equation*} \begin{aligned}
n\Xtilde^T T^{-1/2} &
        \left( \frac{ (\avec - X\wtrue) }{ \sqrt{t_v} }
                + \frac{ t_v - d_v }{ 2t_v^{3/2} }X\wtrue
        \right)
= n\sum_{j=1}^n \frac{ a_j - X_j^T\wtrue }{ \sqrt{n \mu^T \wtrue} }
	\left( \frac{X_j}{n X_j^T \mu}
		- \frac{ \Deltatilde \wtrue}{2 n \mu^T \wtrue } \right)
	+ O( n^{-1/2} \log^{1/2} n ) \\
&= \frac{1}{\sqrt{n}}
	\sum_{j=1}^n
	\frac{ (a_j - X_j^T\wtrue) }{\sqrt{ \mu^T \wtrue} }
	\left( \frac{X_j}{X_j^T \mu}
			- \frac{ \Deltatilde \wtrue}{2\mu^T \wtrue } \right)
	+ O(n^{-1/2} \log^{1/2} n).
\end{aligned}
\end{equation*}
Observe that this is a sum of $n$ independent mean-zero random
variables, so that by the multivariate CLT and Slutsky's Theorem,
\begin{equation*}
n\Xtilde^T T^{-1/2} 
        \left( \frac{ (\avec - X\wtrue) }{ \sqrt{t_v} }
                + \frac{ t_v - d_v }{ 2t_v^{3/2} }X\wtrue
        \right)
\inlaw \calN(0, \Sigma_{F,\wtrue} ),
\end{equation*}
where
\begin{equation*} \begin{aligned}
\Sigma_{F,\wtrue}
&= \E \left[ \frac{ X_j^T\wtrue(1-X_j^T\wtrue) }{ \mu^T\wtrue }
	\left( \frac{X_j}{X_j^T \mu}                                                         - \frac{ \Deltatilde \wtrue}{2\mu^T \wtrue } \right)
        \left( \frac{X_j}{X_j^T \mu}                                                         - \frac{ \Deltatilde \wtrue}{2\mu^T \wtrue } \right)^T \right],
\end{aligned} \end{equation*}
completing the proof.
\end{proof}

\newpage

\bibliographystyle{plainnat}
\bibliography{biblio}

\end{document}